\RequirePackage{fix-cm}
\documentclass[smallextended]{svjour3}       % onecolumn (second format)

\smartqed  % flush right qed marks, e.g. at end of 
\usepackage{graphicx}
\usepackage{amsmath}
\usepackage{amssymb}
\usepackage{amsfonts}
\usepackage[numbers]{natbib} 
\usepackage[acronym,nowarn,section,nogroupskip,nonumberlist]{glossaries}
\usepackage{nccmath}
\usepackage[Symbol]{upgreek}

\usepackage{booktabs}
\usepackage{tabularx}
\usepackage[makeroom]{cancel}
\usepackage[normalem]{ulem}
\usepackage[hidelinks]{hyperref}
\usepackage{xcolor}
\hypersetup{
    colorlinks,
    linkcolor={blue!50!black},
    citecolor={blue!50!black},
    urlcolor={blue!80!black}
}

\usepackage{url}
\usepackage{color}
\usepackage{multirow, makecell}
\usepackage[export]{adjustbox}
\usepackage[ruled,vlined]{algorithm2e}
\usepackage{enumerate}
\usepackage{bm}
\usepackage{mathtools}
\usepackage{thmtools}
\usepackage{bigints}
\usepackage[breakable]{tcolorbox}

\usepackage{thmtools, thm-restate}

\makeatletter
\def\cl@chapter{\@elt {theorem}}
\makeatother

\usepackage[nameinlink,capitalise, poorman]{cleveref}

\spnewtheorem{thm}{Theorem}{\bfseries}{\itshape}
\crefname{thm}{Thm.}{thm}
\spnewtheorem{example}{Example}{\bfseries}{}
\spnewtheorem{ntn}{Notations}[section]{\bfseries}{\itshape}
\spnewtheorem{pro}{Proposition}[section]{\bfseries}{\itshape}
\spnewtheorem{dfn}{Definition}[section]{\bfseries}{\itshape}% maybe \upshape?
\spnewtheorem{as}{Assumption}[section]{\bfseries}{\itshape}
\spnewtheorem{rem}{Remark}[section]{\bfseries}{\itshape}
\spnewtheorem{ob}{Observation}[section]{\bfseries}{\itshape}

\newcommand{\vx}{{x}}
\newcommand{\vz}{{x}}

\newcommand{\h}{\rho}

\newcommand{\of}[1]{\big( #1 \big)}
\newcommand{\ofb}[1]{\big[ #1 \big]}
\newcommand{\ofc}[1]{\left(#1\right)}
\newcommand{\bracket}[1]{\left\{ #1 \right\} }

\newcommand{\expqof}[1]{\left[1+(1-q) #1 \right]_{+}^{\frac{1}{1-q}} }

\newcommand{\vupi}{\bm{\tpi}}
\newcommand{\vwbeta}{\bm{\beta}}

\newcommand{\cX}{{\mathcal{X}}}
\newcommand{\cY}{{\mathcal{Y}}}

\newcommand{\tfwd}{{\mathcal{T}_t}}
\newcommand{\trev}{{\tilde{\mathcal{T}}_t}}
\newcommand{\bbR}{{\mathbb{R}}}

\newcommand{\tpi}{{\tilde{\pi}}}

\tcbset{mybox/.style={colback=red!1, width =\textwidth, boxrule=0.0pt, top=5pt,bottom=5pt,left=2pt, right=2pt},
  before upper=\setlength{\parindent}
  {0em}\everypar{{\setbox0\lastbox}\everypar{}}
}
\newtcolorbox{mybox}[1][]{breakable, mybox,#1}

\newcommand{\DKL}{D_{\text{KL}}}

\newcommand{\DJSa}{D^{(\alpha)}_{\text{JS}}}
\newcommand{\DJSb}{D^{(\beta)}_{\text{JS}}}

\newcommand{\DIS}{D_{\text{IS}}}
\newcommand{\DBeta}[1]{D_{\text{B}}^{(#1)}}

\newcommand{\KL}{{\text{KL}}}
\newcommand{\kl}{{\textsc{kl} }}
\newcommand{\Breg}[1]{D_{#1}}
\newcommand{\BregInfo}[1]{{\mathcal{I}_{#1}}}%^{B}

\newcommand{\rhobreg}{D_{f}} %{D_{f,\rho}}
\newcommand{\taubreg}{D_{f^*}} %{D_{f^*,\tau}}
\newcommand{\canonical}{D_{f,f^*}}
\newcommand{\canonicaldual}{D_{f^*,f}}

\newcommand{\qpath}{\tpi^{(q)}_\beta}
\newcommand{\tpigeo}{\tpi_{\beta}^{(\text{geo})}}

\newcommand{\hl}[1]{{#1}}

\newcommand{\tpimix}{\pi^{(\text{arith})}_{\beta}}

\newcommand{\firstarg}{{{\tpi_a}}}
\newcommand{\secondarg}{{{\tpi_b}}}
\newcommand{\firstargn}{{\pi_a}}
\newcommand{\secondargn}{{\pi_b}}
\newcommand{\firstargparam}{{\pi_{\btheta_a}}}
\newcommand{\secondargparam}{{\pi_{\btheta_b}}}
\newcommand{\tpizero}{\tpi_0}

\newcommand{\GammaRT}[1]{\Gamma^{(#1)}}%_{f,\rho}}}

\newcommand{\btheta}{\bm{\theta}}
\newcommand{\boldtheta}{\btheta}
\newcommand{\bthetaprime}{\btheta^{\prime}}
\newcommand{\boldeta}{\bm{\eta}}
\newcommand{\bt}{\bm{T}(\vz)}

\newcommand{\base}{g}
\newcommand{\expbase}{\base(\vz)}
\newcommand{\expfam}{\pi_{\btheta}}
\newcommand{\qexpfam}{\pi^{(q)}_{\btheta}}
\newcommand{\texpfam}{\tilde{\pi}_{\btheta}}
\newcommand{\tqexpfam}{\tilde{\pi}^{(q)}_{\btheta}}

\newcommand{\Z}{\mathcal{Z}}

\newcommand{\norm}{\psi}
\newcommand{\potential}{{\Psi}}
\newcommand{\potentialf}{{\Psi_{f}}}%,\rho
\newcommand{\potentialdual}{{\Psi^*_{f^*}}}%, \tau

\DeclareMathOperator*{\argmin}{arg\,min}

\newacronym{AIS}{\textsc{ais}}{annealed importance sampling}
\newacronym{BA}{\textsc{ba}}{Barber-Agakov}
\newacronym{BQ}{\textbf{bq}}{Bayesian Quadrature}
\newacronym{AUC}{\textbf{auc}}{area under the curve}
\newacronym{BAR}{\textsc{bar}}{Bennett's Acceptance Ratio}
\newacronym{GAN}{\textsc{gan}}{generative adversarial network}
\newacronym{GANs}{\textsc{gan}}{generative adversarial networks}
\newacronym{BDMC}{\textsc{bdmc}}{Bidirectional Monte Carlo}
\newacronym{LREF}{\textsc{lref}}{likelihood ratio exponential family}
\newacronym{JS}{\textsc{js}}{Jensen-Shannon}
\newacronym{CFT}{\textsc{cft}}{Crooks's Fluctuation Theorem}
\newacronym{ELBO}{\textsc{elbo}}{evidence lower bound}
\newacronym{EUBO}{\textsc{eubo}}{Evidence Upper Bound}
\newacronym{HMC}{\textsc{hmc}}{Hamiltonian Monte Carlo}
\newacronym{IB}{\textsc{ib}}{Information Bottleneck}
\newacronym{MI}{\textsc{mi}}{Mutual information}
\newacronym{MINE}{\textsc{mine}}{Mutual Information Neural Estimation}
\newacronym{JE}{\textsc{je}}{Jarzynksi equality}
\newacronym{JSD}{\textsc{jsd}}{Jensen-Shannon divergence}
\newacronym{IS}{\textsc{is}}{importance sampling}
\newacronym{IWAE}{\textsc{iwae}}{importance-weighted autoencoder}
\newacronym{GIWAE}{\textsc{giwae}}{\textit{Generalized} \textsc{iwae}}
\newacronym{MCMC}{\textsc{mcmc}}{Markov Chain Monte Carlo}
\newacronym{RD}{\textsc{rd}}{rate-distortion}
\newacronym{RBM}{\textsc{rbm}}{Restricted Boltzmann Machines}
\newacronym{RWS}{\textsc{rws}}{reweighted wake-sleep}
\newacronym{SGD}{\textsc{sgd}}{stochastic gradient descent}
\newacronym{SMC}{\textsc{smc}}{Sequential Monte Carlo}
\newacronym{PT}{\textsc{pt}}{Parallel Tempering}
\newacronym{SNIS}{\textsc{snis}}{self-normalized importance sampling}
\newacronym{TI}{\textsc{ti}}{thermodynamic integration}
\newacronym{TVI}{\textsc{tvi}}{thermodynamic variational inference}
\newacronym{TVO}{\textsc{tvo}}{thermodynamic variational objective}
\newacronym{VAE}{\textsc{vae}}{variational autoencoders}
\newacronym{VAEc}{\textsc{vae}}{Variational Autoencoders}
\newacronym{VI}{\textsc{vi}}{variational inference}
\newcommand{\mycor}[1]{\hyperref[cor:#1]{Cor.~\ref*{cor:#1}}}
\newcommand{\myeq}[1]{\hyperref[eq:#1]{Eq.~(\ref*{eq:#1})}}
\newcommand{\mysecondeq}[1]{\hyperref[eq:#1]{(\ref*{eq:#1})}}
\newcommand{\mysec}[1]{\hyperref[sec:#1]{Sec.~\ref*{sec:#1}}}
\newcommand{\myapp}[1]{\hyperref[app:#1]{App.~\ref*{app:#1}}}
\newcommand{\myex}[1]{\hyperref[example:#1]{Ex.~\ref*{ex:#1}}}
\newcommand{\myalg}[1]{\hyperref[alg:#1]{Alg. ~\ref*{alg:#1}}}
\newcommand{\mytable}[1]{\hyperref[tab:#1]{Table~\ref*{tab:#1}}}
\newcommand{\mytab}[1]{\hyperref[tab:#1]{Table~\ref*{tab:#1}}}
\newcommand{\myfig}[1]{\hyperref[fig:#1]{Fig.~\ref*{fig:#1}}}
\newcommand{\mylemma}[1]{\hyperref[lemma:#1]{Lemma~\ref*{lemma:#1}}}
\newcommand{\myprop}[1]{\hyperref[prop:#1]{Prop.~\ref*{prop:#1}}}
\newcommand{\mythm}[1]{\hyperref[thm:#1]{Thm.~\ref*{thm:#1}}}

\newcommand{\tablelabel}{\alpha(\beta,\rho,\tau) \text{ for } \Gamma}%{\nabla}

\newcommand{\wi}{\beta}

\newcommand{\vectorinp}{u}
\newcommand{\bvector}{\bm{\vectorinp}}

\newcommand{\robpara}[1]{\textbf{#1}}
\newcommand{\robsubsection}[1]{\subsection{\textbf{#1}}}
\newcommand{\robsubsubsection}[1]{\subsubsection{\textbf{#1}}}
\newcommand{\vheader}{{\vspace*{-.24cm}}}

\begin{document}
\sloppy
% \doparttoc % Tell to minitoc to generate a toc for the parts
% \faketableofcontents

\title{Variational Representations of Annealing Paths: \\
Bregman Information under Monotonic Embedding
}
\titlerunning{Variational Representations of Annealing Paths}

\author{Rob Brekelmans         \and
        Frank Nielsen %etc.
}

\institute{Rob Brekelmans \at
            Vector Institute \\
            USC Information Sciences Institute \\
              \email{rob.brekelmans@vectorinstitute.ai}   \\
           \and
           Frank Nielsen \at
           Sony Computer Science Laboratories Inc, \\
            %  Tokyo, Japan \\
             \email{frank.nielsen@acm.org}
}
 
\date{}
%Received: 22 September 2022 / Revised: 22 December 2023 / Accepted: 27 December 2023 \\ © The Author(s), under exclusive licence to Springer Nature Singapore Pte Ltd. 2024}

\maketitle

%\vspace*{-1cm} 

\begin{abstract}
Markov Chain Monte Carlo methods for sampling from complex distributions and estimating normalization constants often simulate samples from a sequence of intermediate distributions along an \textit{annealing path}, which bridges between a tractable initial distribution and a target density of interest.  Prior works have constructed annealing paths using quasi-arithmetic means, and interpreted the resulting intermediate densities as minimizing an expected divergence to the endpoints.  
To analyze these variational representations of annealing paths, we extend known results showing that the arithmetic mean over arguments minimizes the expected Bregman divergence to a single representative point.   In particular, we obtain an analogous result for quasi-arithmetic means, when the inputs to the Bregman divergence are transformed under a monotonic embedding function.  Our analysis highlights the interplay between quasi-arithmetic means, parametric families, and divergence functionals using the rho-tau representational Bregman divergence framework of \citet{zhang2004divergence,zhang2013nonparametric}, and associates common divergence functionals with intermediate densities along an annealing path.
%the results of \citet{Banerjee2005} 
%of \citet{zhang2004divergence, zhang2013nonparametric}, 
%\vspace{1.25cm}

\keywords{Bregman divergence
\and Bregman Information
\and monotone embedding 
\and quasi-arithmetic means 
%\and quasi-arithmetic mixtures
\and non-parametric information geometry
\and gauge freedom  
\and annealing paths 
%\and variational divergence minimization
\and Markov Chain Monte Carlo }
\end{abstract}

%\vheader 
%\vspace*{10pt}
\section{Introduction}
%\vspace*{10pt}
\vheader

\gls{MCMC} methods such as \gls{AIS} \citep{neal2001annealed}, \gls{SMC} \citep{del2006sequential}, \gls{TI} \citep{ogata1989monte,gelman1998simulating} and \gls{PT} \citep{earl2005parallel} are fundamental tools in statistical physics and machine learning, which can be used to sample from complex distributions, estimate normalization constants, and calculate physical quantities such as entropy or free energy.   Such tasks appear in the context of Bayesian inference or model selection, where the posterior over latent variables or model parameters is usually intractable to sample and evaluate.

 \gls{MCMC} algorithms often 
 decompose these problems
 %break the problem of sampling from the target
 into a sequence of easier subproblems along an \textit{annealing path} of intermediate densities $\{\tpi_{\beta_t}(\vz)\}_{t=0}^T$, which bridge between a tractable (often normalized) density $\tpi_0(\vz)$ and the complex target density of interest $\tpi_1(\vz)$.  Using the notation $\tpi$ to indicate the density of an unnormalized measure with respect to the Lebesgue measure,  we are interested in sampling from $\pi_1(\vz) \propto \tpi_1(\vz)$ or estimating the normalization constant {$\mathcal{Z}_1 = \int \tpi_1(\vz) d\vz$}, its logarithm $\log \mathcal{Z}_1$, or the ratio $\mathcal{Z}_1/\mathcal{Z}_0$, where $\mathcal{Z}_0 = \int \tpi_0(\vz) d\vz$.   
Transition kernels $\mathcal{T}_t(\vz_t|\vz_{t-1})$ such as importance resampling, Langevin dynamics \citep{rossky1978brownian, welling2011bayesian}, \gls{HMC} \citep{duane1987hybrid, neal2011mcmc, betancourt2017geometric}, and accept-reject steps are used to transform samples to more accurately simulate the target density.
For example, we describe the annealed importance sampling algorithm (\gls{AIS}, \citep{neal2001annealed, jarzynski1997equilibrium}) in \myalg{ais}, which provides approximate target samples and an unbiased estimator of the ratio of normalization constants.

\begin{figure*}[b]
\centering
    \begin{minipage}{0.99\textwidth}
  %\vspace{15pt}
            %\hspace*{-10pt}
            %\normalem 
            \begin{algorithm}[H]
                %\SetAlgoLined
                \SetKwInOut{Input}{input}
                \SetKwInOut{Ret}{return}
                \Input{Endpoint densities $\tpizero(\vz), \tpi_1(\vz)$ \\
                Schedule $ \{ \beta_t \}_{t=0}^T$ and 
                %, $\beta_0 =0 \leq \beta_t \leq \beta_T = 1$ \\
                Annealing Path $\beta_t \mapsto \tpi_{\beta_t}(\vz)$ (e.g. \cref{eq:general_paths}) \\
                Transition Kernels $\tfwd(\vz_{t} | \vz_{{t-1}})$ leaving $\pi_{\beta_{t-1}}$ invariant}
                 \For{k = 1 \text{to} K}{
                        $\quad\,\,\,~ \vz_0^{(k)} \sim \tpizero(\vz)$,
                        \qquad \qquad \,\,\,\,\,~
                        %\\[1.5ex]
                        $w_0^{(k)} \leftarrow 1$
                    }
                % \\
                % \hspace*{-1.1cm} \scriptsize (e.g. geometric averaging, schedule $\scriptsize \{ \beta_t \}_{t=0}^T$)} 
                \For{t = 1 to T}{
                    \For{k = 1 \text{to} K}{
                    %$\vz_0 \sim \tpizero(\vz)$\\[1.5ex]
                    %$w_0^{(k)} \leftarrow 1$ \\ %\mathcal{Z}_0$\\
                    % $w_ \leftarrow \mathcal{Z}_0$\\
                    $\vz_t^{(k)} \sim \tfwd(\vz^{(k)}_{t} | \vz_{{t-1}}^{(k)})$,
                    \qquad
                    %\\ %^{(k)}
                    $w_t^{(k)} \leftarrow$ $w_{t-1}^{(k)} \frac{\tpi_{\beta_{t}}(\vz_{t }^{(k)})}{\tpi_{\beta_{t-1}}(\vz_{t }^{(k)})}$\\
                    }
                    }
                    \Return{ 
                    %\text{\small Approximate samples} 
                    $\text{\upshape Approximate samples:}$
                    \,\, $\vz_T \sim \pi_T(\vz)$
                    \\[1.1ex]
                    $\text{\upshape \qquad \quad \,\,\, Unbiased Estimator:}$ 
                    %\text{\upshape\quad \,\,\, Unbiased Estimator:}
                    \,\, $\mathcal{Z}_T / \mathcal{Z}_0 = \mathbb{E}\Big[ \frac{1}{K} \sum_{k=1}^K w_T^{(k)} \Big]$}
                    %\quad \, $\mathcal{Z}_T / \mathcal{Z}_0 \approx \frac{1}{K} \sum \limits_{k=1}^K w_T^{(k)}$}
                    \vspace{-2pt}
                    \caption{Annealed Importance Sampling}
                    \label{alg:ais}
            \end{algorithm}
        \end{minipage}
        \vspace{-35pt}
\end{figure*}

Most commonly, 
intermediate unnormalized densities are
constructed using geometric averaging $\tpigeo(\vz) = \tpi_0(\vz)^{1-\beta} \tpi_1(\vz)^{\beta}$ of the initial 
and target 
densities, with $\beta\in [0,1]$.
Viewing the geometric path as a quasi-arithmetic mean \citep{kolmogorov1930} under transformation by the natural logarithm, \citet{masrani2021q} propose annealing paths using the deformed logarithm transformation function rooted in nonextensive thermodynamics (\citep{tsallis2009introduction},\citep{naudts2011generalised} Ch 7, defined in \myeq{qlogexpdef}),
\begin{align}
%\begin{array}{rl}
\begin{split}
    \log \tpigeo(\vz) &= (1-\beta) \log \tpi_0(\vz) \hphantom{_q}  +  \beta   \log \tpi_1(\vz) \\
    \log_q \qpath(\vz) &= (1-\beta) \log_{q} \tpi_0(\vz)  + \beta  \log_q \tpi_1(\vz).
   %    \end{array}
   \end{split}
        \label{eq:general_paths}
\end{align}
Choosing a suitable path may facilitate
more accurate \gls{MCMC} estimators with fewer intermediate densities, as evidenced by experiments in \citep{grosse2013annealing, masrani2021q, syed2021parallel}.

\begin{figure}[t]
    \centering
    \includegraphics[width=0.95\textwidth, center]{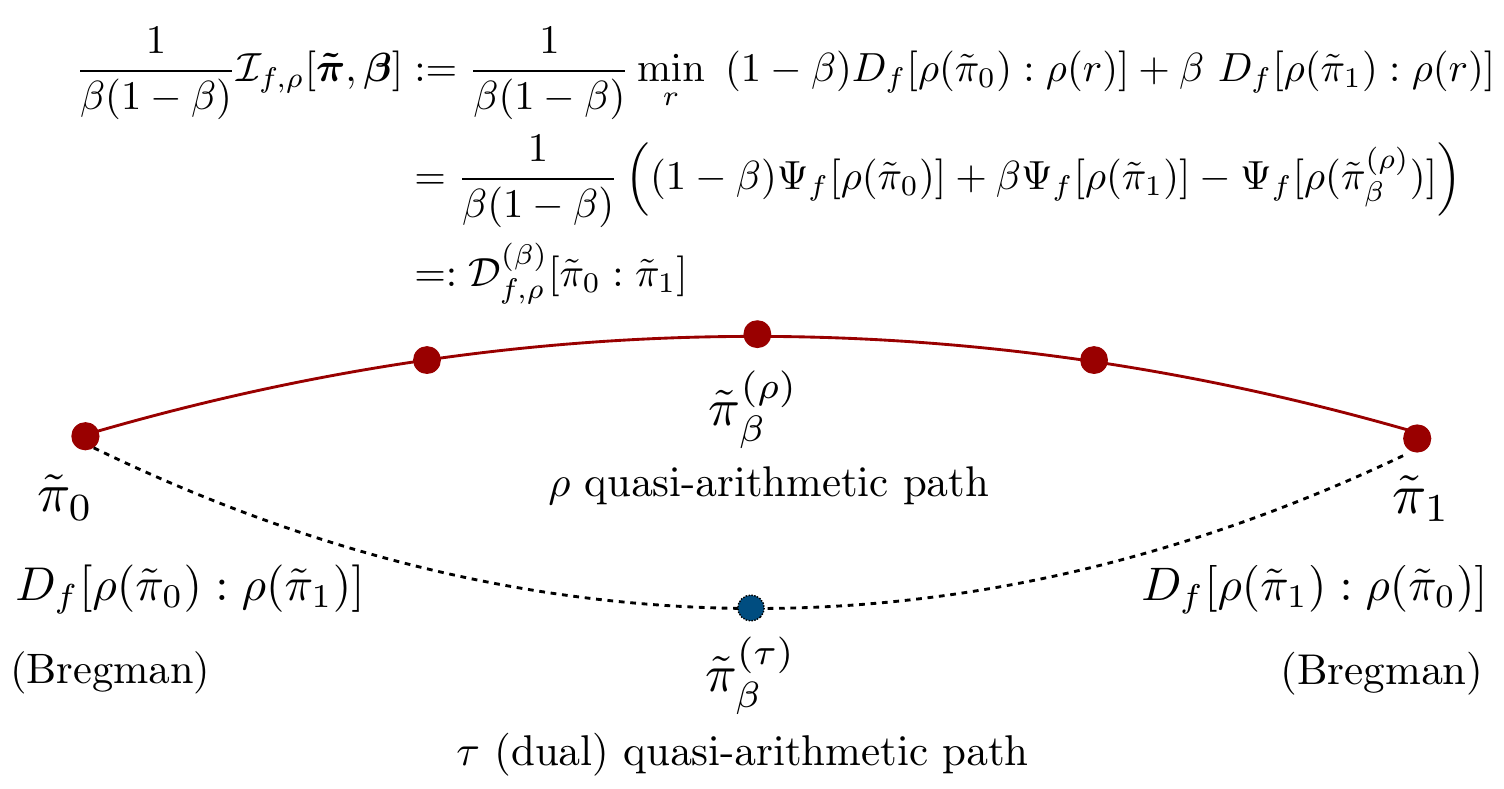}
    \caption{ Illustration of our main results.
    %the divergences associated with a quasi-arithmetic annealing path (\mycor{zhang_div}).  
    The choice of a monotonic representation function $\rho$, along with either a convex function $f$ or a second monotonic function $\tau$, specifies a Bregman divergence for $\beta = \{0,1\}$.   The quasi-arithmetic path $\tpi_{\beta}^{(\rho)}$ with mixing weight $\beta$ minimizes the expected Bregman divergence $D_{\Psi_f}[\rho(\tpi_a):\rho(\tpi_b)]$ to the endpoints $\{ \tpi_0, \tpi_1 \}$, where optimization is performed over the second argument and $\Psi_f[\rho(\tpi)] \coloneqq \int f(\rho(\tpi(\vx)) d\vx$.   
    The (scaled) value of this objective is called the Bregman Information $\mathcal{I}_{f,\rho}^{(\beta)}$, and 
    associates to each intermediate density $\tpi_{\beta}^{(\rho)}$ a divergence $D_{f,\rho}^{(\beta)}[\tpi_0:\tpi_1] = \frac{1}{\beta(1-\beta)} \mathcal{I}_{f,\rho}^{(\beta)}[\bm{\tpi}, \bm{\beta}]$ \citep{zhang2004divergence} which compares $\tpi_0$ and $\tpi_1$ (\cref{example:zhang_div}).
    %with each intermediate density $\tpi_{\beta}^{(\rho)}$ along an annealing path between $\tpi_0$ and $\tpi_1$ (\cref{example:zhang_div}).
    Finally, the quasi-arithmetic path in the $\rho$-representation is a geodesic with respect to the affine connection induced by the Bregman divergence $D_{\Psi_f}[\rho(\tpi_a):\rho(\tpi_b)]$ for any $\Psi_f$ (\mythm{geodesic}), which is also true for the $\tau$-representation, dual divergence, and dual connection.
    %Minimizing the expected Bregman divergence to the endpoints associates a divergence $D_{f,\rho}^{(\beta)}$ or Bregman Information $\mathcal{I}_{f,\rho}$ with each intermediate
    % The $\rho$-representation quasi-arithmetic path is an affine coordinate system with respect to the primal connection induced by the Bregman divergence with
    % %geodesic with respect to the primal affine connection induced by the Bregman divergence for
    % $\beta = 1$ (\mythm{geodesic}), and similarly for the $\tau$-representation and  dual connection.
    %(specified by $f^*$ and $\beta=0$).  
    % the Bregman divergence for
    % %$D_f[\rho(\tpi_1):\rho(\tpi_0)]$ for 
    % $\beta = 1$ (\mythm{geodesic}), and similarly for the $\tau$-representation and the dual connection specified by $f^*$ and $\beta=0$.  
    %Each intermediate density along the path minimizes the expected Bregman divergence to the endpoints, associating a divergence $D_{f,\rho}^{(\beta)}$ or Bregman Information $\mathcal{I}_{f,\rho}^{(\beta)}$ with each $\tpi_{\beta}^{(\rho)}$ (\mycor{zhang_div}). 
    }
    \label{fig:my_label}
    \vspace*{-.1cm}
    %\vspace*{-.225cm}
\end{figure}

Intriguingly, \citet{grosse2013annealing, masrani2021q} show that densities along the paths in \myeq{general_paths} minimize the expected divergence to the endpoints
\normalsize
\begin{align}
% \left\{ 
\hspace*{-.3cm} 
\begin{split}
\tpi_{\beta}^{(\text{geo})}(\vz) &=  \argmin \limits_{\tpi} \,\, (1-\beta) \DKL\big[ \tpi: \tpizero \big] + \beta  \DKL[ \tpi : \tpi_1 \big]  \\
     \qpath(\vz) &=  \argmin \limits_{\tpi} \,\, (1-\beta)  D_{A}^{(\alpha)} \big[ \tpizero: \tpi \big] + \beta D_{A}^{(\alpha)}[ \tpi_1: \tpi \big] 
        \end{split}
       \label{eq:general_div_min}
\end{align}
\normalsize
where $D_{A}^{(\alpha)}$ indicates Amari's $\alpha$-divergence \citep{amari1982differential, amari2007integration, amari2016information} using the reparameterization $\alpha \leftarrow \frac{1+\hat{\alpha}}{2}$ with $\alpha = q$,
\small 
\begin{align}
   D_{A}^{(\alpha)}[\firstarg:\secondarg]   &=\frac{1}{\alpha} \int \firstarg(\vz) d\vz + \frac{1}{1-\alpha} \int \secondarg(\vz) d\vz - \frac{1}{\alpha}\frac{1}{1-\alpha} \int \firstarg(\vz)^{1-\alpha}\secondarg(\vz)^{\alpha} d\vz.  \label{eq:amari_alpha}
\end{align}
\normalsize
The minimization in \myeq{general_div_min} is reminiscent of the problem of finding the barycenter, or $\argmin_{\mu} \mathbb{E}_{\nu(u)} D[u:\mu]$, of a random variable $U \sim \nu$ with respect to a statistical divergence $D$ and sampling measure $\nu(u)$ \citep{sibson1969information, nock2005fitting, nielsen2011burbea}.   While properties of this optimization depend on the choice of divergence in general, \citet{banerjee2005optimality,Banerjee2005} show that for \textit{any} Bregman divergence, minimization in the second argument yields the arithmetic mean {$\mu^* = \mathbb{E}_{\nu}[u]$} as the unique optimal barycenter. 
The property of the arithmetic mean as the barycenter  in fact
characterizes the Bregman divergences among divergences or losses \citep{banerjee2005optimality}.
At this minimizing argument, the expected divergence in \myeq{general_div_min} reduces to a gap in Jensen's inequality known as the Bregman Information \citep{Banerjee2005}.

Our main result in \mythm{breg_info} uses the rho-tau Bregman divergence framework of \citet{zhang2004divergence, zhang2013nonparametric} to extend the Bregman Information results of \citet{Banerjee2005} to quasi-arithmetic means.   A key observation is that the minimizing arguments in \myeq{general_paths} are arithmetic means after applying a monotonic \textit{representation function} $\rho(\tpi) = \log_q \tpi$.  
The Bregman Information further associates a divergence functional with each intermediate density along a $q$-path between $\tpi_0$ and $\tpi_1$ (\myfig{my_label}), encompassing many common divergences as special cases (\mytab{jensen_gaps}). 
Our analysis highlights the interplay between quasi-arithmetic means and divergence functionals, and naturally bridges between parametric \citep{amari2000methods} and nonparametric information geometry \citep{zhang2013nonparametric, naudts2018rho}.  
Via the intuitive example of annealing paths, we seek to familiarize a wider machine learning audience with the referential-representational biduality in information geometry \citep{zhang2004divergence, zhang2013nonparametric}.

While the foundation for many of the above results appears in \citet{zhang2004divergence, zhang2013nonparametric}\citep{zhang2021entropy}, our analysis emphasizes divergence minimization properties (\mysec{rho_tau_breg}) and parametric interpretations (\mysec{parametric}) by connecting the Bregman Information and representational $\alpha$-divergence studied in \citep{zhang2004divergence}.   
In \mythm{geodesic}, we show that annealing paths derived from quasi-arithmetic means are geodesic curves with respect to the affine connections induced by rho-tau Bregman divergences. Finally, we provide novel variational representations of $q$-paths (\myeq{general_paths}, \citep{masrani2021q}) as the solution to an expected $\beta$-divergence \citep{basu1998robust, eguchi2006information} or Cichocki-Amari $(\alpha,\beta)$-divergence \citep{cichocki2010families} minimization in \mysec{examples}.   Since \textsc{mcmc} applications consider unnormalized densities as input with the goal of calculating normalization constants, we discuss in \cref{example:normalized} how our analysis differs from prior work involving quasi-arithmetic means of normalized densities \citep{sibson1969information, amari2007integration, eguchi2016information, wong2021tsallis}. 
 \vspace*{.1cm}

\textit{Notation}   \hl{Throughout this work, we consider a measure space $(\cX, \Sigma, dx)$.  We denote 
%over a sample space 
probability densities with respect to $dx$ as  $\pi(x)$ and indicate unnormalized densities, which may not integrate to 1, using $\tpi(x)$.  All integrals $\int (\cdot)~dx$ are assumed to be over the domain $\cX$.}

\vspace*{-.4cm}
\section{Bregman Divergence
under Monotonic Embedding}
\vheader
We begin by reviewing the notion of a quasi-arithmetic mean associated with a monotonic embedding function in \mysec{generalized_means}, and describe how such embedding functions can be used to define divergence functionals in the rho-tau Bregman divergence framework of \citep{zhang2004divergence, zhang2013nonparametric, naudts2018rho, zhang2021entropy} in \mysec{rhotau}. 

\vheader
\vheader
\robsubsection{{Quasi-Arithmetic Means}}\label{sec:generalized_means}
\vheader
Consider a \hl{strictly monotonic, continuously differentiable \textit{representation function} $\rho :  \cX_\rho \subset \bbR \rightarrow \cY_\rho \subset\bbR$.}~
%: \bbR^d \mapsto \mathcal{Y}_\rho \subset \bbR^d$.  
Given a set of input elements ${\bm{u}} = (u_1, ..., u_N)$ with $u_i \in \cX_\rho$ and nonnegative mixing weights ${\bm{\wi}} = (\wi_1, ..., \wi_N)$ which are normalized $\sum_{i=1}^N \wi_i = 1$, 
the \textit{quasi-arithmetic mean} \citep{kolmogorov1930} is 
%defined
\begin{align}
\mu_{\h}({\bf{u}, \bm{\wi}} ) = \h^{-1} \left(\sum_{i=1}^N \wi_i \cdot \h(u_i)\right). \label{eq:abstract_mean}
\end{align}
We will primarily focus on the setting where each input $u_i = \tpi_i(\vz)$ is a density function.
Since the function $\rho$
is monotonic and thus invertible, we may represent a given density $\tpi(\vz)$ 
%equivalently 
using $\rho\big(\tpi(\vz)\big)$ without loss of information.   
This 
%$\rho(\tpi)$ 
is known as the \textit{$\rho$-representation} of $\tpi$ \citep{amari2007integration, amari2016information}.

The property that $\mu_{\h}(\bm{u},\bm{\wi})$ is \textit{$\rho$-affine}, or linear in the $\rho$-representation of $u_i$'s \citep{zhang2004divergence}, will play a key role in our later analysis.   In other words, $\mu_{\h}(\bf{u}, \bm{\wi})$ is an arithmetic mean of the inputs after transformation by $\h(u)$
\begin{align}
\h(\mu_{\h}({\bf{u},\bm{\wi}} ))=\sum_{i=1}^N \wi_i \cdot \h(u_i), \label{eq:abstract_affine}
\end{align}
%In other words, 
   The quasi-arithmetic mean is also invariant to the affine transformations of $\rho$, with $\mu_{\h}(\bf{u},\bm{\wi}) = \mu_{\h_{c,a}}(\bf{u},\bm{\wi})$ for $\h_{c,a}(u) = c \rho(u) + a$, $c > 0$, and $a \in \mathbb{R}$.

As a primary example of the representation function $\rho$, we will consider the family of $q$-deformed logarithms \citep{tsallis2009introduction, naudts2011generalised} from nonextensive thermodynamics.   
% which plays a crucial role in nonextensive thermodynamics.  
For $q, u > 0$, the $q$-logarithm and its inverse, the $q$-exponential, are defined
\begin{align}
   \rho_q(u) \coloneqq \log_q(u) = \frac{1}{1-q}\left( u^{1-q} - 1\right) \quad \qquad  \exp_q(t) = \expqof{t} \, ,\label{eq:qlogexpdef}
\end{align}
where $[\cdot]_+=\max(\cdot,0)$. 
While $\log_q(u)$ is an affine transformation of the $\alpha$-representation \citep{amari2007integration} (for $\alpha=2q-1$), we will use the parameter $q$ to avoid later confusion as to the role of the parameter $\alpha$ \citep{zhang2004divergence, zhang2013nonparametric}.
It can be shown that $\log_q(u)$ is concave and strictly increasing in $u$ \citep{naudts2011generalised}  
while, taking the limiting behavior as $q\rightarrow 1$, we recover the natural logarithm $\log(u)$ and standard exponential $\exp(t)$ as its inverse. 
The more general family of $\phi$-deformed logarithms (\citep{naudts2004estimators, naudts2011generalised, naudts2018rho}
), including the $\kappa$-logarithm of \cite{kaniadakis2002new}, might also be considered.   Our main results apply for arbitrary monotonic representation functions $\rho(u)$.

\vheader
\robsubsection{{Rho-Tau Bregman Divergence}}\label{sec:rhotau}
%\vheader 
The rho-tau Bregman divergence of \citet{zhang2004divergence,zhang2013nonparametric} provides an elegant framework for understanding the relationship between divergence functions and representations of probability densities under monotonic embedding, and will form the basis 
of our later analysis.  

Consider $\rho$ and $\tau$  to be scalar functions which, applied to an unnormalized density function $\tpi(\vz)$, map $\tpi(\vz) \mapsto \rho\of{\tpi(\vz)}$.  
We consider a proper, strictly convex, lower semi-continuous function $f : \mathbb{R} \mapsto \mathbb{R}$ applied to $\rho(\tpi) \in \mathbb{R}$, and define $f^*$ to be its convex conjugate.  
Using the Fenchel-Moreau biconjugation theorem, we can express $f$ or $f^*$ via a conjugate optimization,
%Expressing $f$ via a conjugate optimization, we have
\begin{align}
    % f^*(\tau) = \sup \limits_{\rho} \, \langle \rho, \tau \rangle - f(\rho)
    f^*(\tau) = \sup \limits_{\rho} \, \rho \cdot \tau - f(\rho)    \qquad \qquad  f(\rho) =  \sup \limits_{\tau} \, \rho \cdot \tau - f^*(\tau)  . \label{eq:conj_tau}
\end{align}
Solving for the optimizing arguments above suggests the conjugacy conditions,
\begin{align}
    \tau = f^{\prime}(\rho) = \left( ( f^{*} )^{\prime}\right)^{-1}(\rho) \qquad \qquad \rho= (f^{*})^{\prime}(\tau) =  (f^\prime)^{-1}(\tau) . \label{eq:conj_condition}
\end{align}
\citet{zhang2004divergence, zhang2013nonparametric} refer to these choices of $\rho$ and $\tau$ as \textit{conjugate representations} with respect to the convex function $f$, emphasizing that the choice of two of these functions $(\rho, \tau)$ or $(\rho, f)$ determines the third one.
%\citep{zhang2021entropy}.
In fact, for any choice of $\rho$, $\tau$, it is possible to find an appropriate $f$, with $f^{\prime} = \tau \circ \rho^{-1}$ due to the fact that set of strictly monotonic functions form a group, where the group action is function composition \citep{zhang2015monotone}.  
Finally, note that monotonicity of $\rho = (f^*)^{\prime}(\tau)$ is guaranteed if $f^*$ is strictly convex.  We will indeed interpret $\rho(\tpi)$ as the representation function for the quasi-arithmetic mean in \mysec{examples}.

We consider using the convex functions $f(\rho)$ or $f^*(\tau)$ to generate \textit{decomposable} Bregman divergences \citep{zhang2004divergence,zhang2013nonparametric}, where the arguments are now expressed as the $\rho(\tpi)$ or $\tau(\tpi)$ representations of the input density functions.  
{Restricting attention to $\tpi$ such that $\int f(\rho(\tpi(x))) dx < \infty$ and $\int f^*(\tau(\tpi(x))) dx < \infty$, we consider the decomposable Bregman divergence
%The Bregman divergence generated by $\potentialf$ becomes,
\begin{align}
    \rhobreg &\big[\rho \big(\firstarg\big):\rho \ofc{\secondarg}\big] \label{eq:rhotau_tau} \\
    &= \int f\big(\rho\of{\firstarg(\vz)} \big) -  f\big(\rho\of{\secondarg(\vz)} \big) - \big( \rho\of{\firstarg(\vz)} - \rho\of{\secondarg(\vz)} \big) 
f^{\prime}\left(\rho\of{\secondarg(\vz)}\right) d\vz, \nonumber
\end{align}
\normalsize
We will use the notation $f(\rho_{\tpi}(\vz)) \coloneqq f(\rho(\tpi(\vz)))$ and $f^*(\tau_{\tpi}(\vz)) \coloneqq f^*(\tau(\tpi(\vz)))$ moving forward, as a shorthand which emphasizes that our convex dualities are 
%constructed after applying the representation functions $\rho$ and $\tau$.
with respect to the scalar output of the representation functions $\rho$ and $\tau$.

Finally, we define the negative rho-tau entropy functionals from \cite{naudts2018rho} as
\begin{align}
     \potentialf[\rho_{\tpi}] = \int f\of{\rho\of{\tpi(\vz)}} d\vz \qquad  \qquad \potentialdual[\tau_{\tpi}] = \int f^*\of{\tau\of{\tpi(\vz)}} d\vz. \label{eq:decomposable_generators}
    %  \potentialf({\tpi}) = \int f\of{\rho_{\tpi}(\vz)} d\vz.  \qquad \potentialdual(\tpi) = \int f^*\of{\tau_{\tpi}(\vz)} d\vz.
\end{align}
While we primarily work with the definition in \cref{eq:rhotau_tau}, it might also be viewed as a functional Bregman divergence
$\rhobreg \big[\rho \big(\firstarg\big):\rho \ofc{\secondarg}\big] = \potentialf[\rho_{\firstarg}] - \potentialf[\rho_{\secondarg}] - \delta\potential_f[\rho_{\secondarg}]( \rho_{\firstarg} - \rho_{\secondarg})$
for an appropriate notion of functional derivative $\delta\potential_f$ (see \cite{frigyik2008functional}).  
As a result, we will also refer to \cref{eq:decomposable_generators} as decomposable generators.

The rho-tau Bregman divergence may also be written with a mixed parameterization or canonical form (\citep{amari2000methods} 3.4) where, recognizing the definition of the conjugate $f^*(\tau_{\secondarg})$  in \cref{eq:rhotau_tau}, we have
\begin{align}
\hspace*{-.2cm}
\canonical \big[\rho({\firstarg}):\tau({\secondarg})\big]  \coloneqq
\potentialf[\rho({\firstarg})] + \potentialdual[\tau(\secondarg)] - \int \rho\of{{\firstarg}(\vz)} \tau\of{\secondarg(\vz)}
     d\vz .
\end{align}
Finally, the \textit{dual} rho-tau Bregman divergence $\taubreg \big[\tau({\secondarg}):\tau({\firstarg})\big]$, which uses the conjugate function $f^{*}(\tau)$ and representation $\tau$ in \cref{eq:rhotau_tau} or the decomposable generator $\potentialdual[\tau]$, is equivalent to the original divergence up to reversing the order of arguments.
The various divergences are related as follows, 
% \small 
\begin{equation}
%\hspace*{-.2cm}
\resizebox{\textwidth}{!}{$
   \rhobreg \big[\rho({\firstarg}): \rho({\secondarg})\big]= \canonical \big[\rho({\firstarg}):\tau({\secondarg})\big] = \canonicaldual \big[\tau({\secondarg}):\rho({\firstarg})\big] = \taubreg \big[\tau({\secondarg}):\tau({\firstarg})\big] . \nonumber
$}
\end{equation}
\normalsize

\begin{example}[\textrm{Forward and Reverse KL Divergence}]\label{example:kl}
Consider the following representations functions $\rho, \tau$, which are conjugate with respect to convex functions $f(\rho), f^*(\tau)$,
% \begin{equation}
% \small
% \begin{aligned}
%    \hspace*{-.2cm} 
%        \rho(\tpi) &= \log \tpi \qquad   f(\rho) = \exp\{\rho\}-\rho -1 \qquad  \potentialf\ofb{\rho_{\tpi}} = 
%      %\int f\of{\rho(\tpi(\vz)} d\vz = 
%      - \int \log \tpi(\vz) d\vz + \int \tpi(\vz) d\vz - 1    \\
%    \tau(\tpi) &= \tpi  \qquad  f^*(\tau)= \tau \log \tau - \tau + 1   \qquad  \potentialdual \ofb{\tau_{\tpi}} = 
%     %\int f\of{\rho(\tpi(\vz)} d\vz \, = 
%     \int \tpi(\vz) \log \tpi(\vz) d\vz - \int \tpi(\vz) d\vz + 1,  
%     \nonumber 
% \end{aligned}%\label{eq:jsd_rho_f1}
% \end{equation}
% \normalsize
%\begin{equation}
%\small
\begin{equation}
\begin{aligned}
  % \hspace*{-.2cm} 
       \rho(\tpi) &= \log \tpi \qquad   \\
       f(\rho) &= \exp\{\rho\}-\rho -1 \qquad 
       % \potentialf\ofb{\rho_{\tpi}} &= 
     %\int f\of{\rho(\tpi(\vz)} d\vz = 
     %- \int \log \tpi(\vz) d\vz + \int \tpi(\vz) d\vz - 1 
\end{aligned}\qquad 
\begin{aligned}
  % \hspace*{-.2cm} 
      \tau(\tpi) &= \tpi   \qquad  \\
      f^*(\tau) &= \tau \log \tau - \tau + 1   \qquad  
     % \potentialdual \ofb{\tau_{\tpi}}&= 
    %\int f\of{\rho(\tpi(\vz)} d\vz \, = 
    %\int \tpi(\vz) \log \tpi(\vz) d\vz - \int \tpi(\vz) d\vz + 1,  
\end{aligned}
\end{equation}
%\end{equation}
\normalsize
which leads to the following decomposable generators,  
%\small
\begin{align}
\begin{split}
     \potentialf\ofb{\rho_{\tpi}} &= 
   %  \int f\of{\rho(\tpi(\vz)} d\vz = 
     - \int \log \tpi(\vz) d\vz + \int \tpi(\vz) d\vz - 1  
    \\
     \potentialdual \ofb{\tau_{\tpi}}&= 
   % \int f^*\of{\rho(\tpi(\vz)} d\vz \, = 
    \int \tpi(\vz) \log \tpi(\vz) d\vz - \int \tpi(\vz) d\vz + 1, 
    \end{split}
\end{align}
\normalsize
where $\potentialdual[\tau({\tpi})]=\potentialdual[{\tpi}]$ is the negative Shannon entropy.
%up to an affine term which does not affect the Bregman divergence.

Using these generators for the rho-tau Bregman divergence in \myeq{rhotau_tau}, we recover the \kl divergence with different ordering of the arguments,
\begin{align}
\rhobreg[\rho({\firstarg}):\rho(\secondarg)] &= \DKL[\secondarg:\firstarg] \qquad 
\taubreg[\tau({\firstarg}):\tau(\secondarg)] = \DKL[\firstarg:\secondarg] . \label{eq:kls}
\end{align}
where the \kl divergence between unnormalized densities is defined as
\begin{align}
\DKL[\tpi_{a} : \tpi_{b}] =  \int \tpi_{a}(\vx) \log \frac{\tpi_{a}(\vx)}{\tpi_{b}(\vx)} d\vx- \int \tpi_{a}(\vx)d\vx + \int \tpi_{b}(\vx)d\vx. 
\end{align}
\end{example}
Our notation in \cref{eq:rhotau_tau}-(\ref{eq:kls}) is suggestive, as treating the $\rho$-representation as input to the Bregman divergence will lead to our main result in \cref{thm:breg_info}.

%\vheader
\vheader
\section{Main Result}\label{sec:rho_tau_breg}

The well-known results of \citet{banerjee2005optimality, Banerjee2005} show that the family of Bregman divergences is characterized by the fact that the arithmetic mean over inputs minimizes the expected divergence to a representative barycenter in the second argument.
In this section, we use the rho-tau Bregman divergence framework to extend this result to quasi-arithmetic means (\mythm{breg_info}, \ref{thm:breg_info_vector}), which will clarify the variational interpretations of the annealing paths in \myeq{general_div_min} (see \mysec{examples}).
Following \citep{Banerjee2005}, we refer to the value of the expected divergence minimization as the rho-tau Bregman Information.   

In the case of two unnormalized probability densities $\tpi_0$ and $\tpi_1$ as input, the rho-tau Bregman Information matches the representational $\alpha$-divergence of \citet{zhang2004divergence, zhang2013nonparametric}, 
thereby associating a divergence functional for comparing $\tpi_0$ and $\tpi_1$ with each quasi-arithmetic mean along an annealing path.
Analyzing the statisical manifold structure induced by these divergences, we show in \mythm{geodesic} that the quasi-arithmetic mean under a monotonic embedding function $\rho$ traces a geodesic with respect to the primal affine connection induced by the Bregman divergence $\rhobreg[\rho({\firstarg}):\rho({\secondarg})]$ for \textit{any} choice of $\tau$ or $f$.

We begin by stating our main result, which shows that the quasi-arithmetic mixture of densities $\tpi_i$ minimizes the expected rho-tau Bregman divergence.  While this is similar in spirit to Prop. 1 of \citet{Banerjee2005} for the arithmetic mean and standard Bregman divergence, we prove a vector-valued version of our theorem in \myapp{thm_pf} \mythm{breg_info_vector} which is more directly analogous to \citep{Banerjee2005}.
We specialize to the case of two unnormalized densities $\tpi_0, \tpi_1$ as input in \cref{example:zhang_div}, which is the focus of our analysis of annealing paths in \mysec{examples}.

\begin{mybox}
\begin{restatable}[\textup{Rho-Tau Bregman Information}]{thm}{rtbreginfo}
\label{thm:breg_info}
Consider a monotonic representation function 
$\rho : \cX_\rho \subset \bbR \mapsto \cY_\rho\subset \bbR$
%$\rho : \bbR^d \rightarrow \rhorange \subset \bbR^d$ 
%over a domain $\rtu \in \bbR^d$ 
and a convex function $f : \cY_\rho \rightarrow \mathbb{R}$.     
Consider discrete mixture weights $\vwbeta = \{\beta_i \}_{i=1}^N$ over $N$ inputs $\bm{\pi}(x) = \{ \pi_i(x) \}_{i=1}^N$, $\pi_i(x) \in \cX_\rho$, with $\sum_i \beta_i = 1$.
Finally, assume the expected value 
%$\mu_{\rho}(\bm{\tpi}, \vwbeta) \coloneqq \int \rho(\tpi) d\vwbeta(\tpi) \in \text{ri}(\cX_\rho)$
$\mu_{\rho}(\bm{\pi}, \vwbeta) \coloneqq \sum_{i=1}^N \beta_i ~ \rho(\tpi_i(x))  \in \mathrm{ri}(\cY_\rho)$
is in the relative interior of the range of $\rho$ for all $x \in \cX$. % and domain of $f$.
Then, we have the following results,
\begin{itemize}%[(i)]
\item[(i)] For a given %(decomposable) 
Bregman divergence $\rhobreg[\rho(\tpi_a):\rho(\tpi_b)]$ with generator $\potential_f$, the optimization 
\begin{align}
    \BregInfo{{f,\rho}} \of{{\bm{\tpi}, \vwbeta}} &\coloneqq \min \limits_\mu \sum \limits_{i=1}^N \, \beta_i  \, \rhobreg \left[ \rho(\tpi_i) : \rho(\mu)  \right] \, . \label{eq:breginfo_optimization}
\end{align}
% \intertext
{has a unique minimizer given by the quasi-arithmetic mean with representation function $\rho(\tpi)$}
\begin{align}
       %\tpi^{(\rho)}_{\vwbeta}(\vz) &\coloneqq 
       \mu^*_{\rho}(\bm{\tpi}, \vwbeta) = 
       \rho^{-1}\left( \sum \limits_{i=1}^N \beta_i \, \rho\of{\tpi_i} \right) 
       = \argmin  \limits_{\mu} \sum \limits_{i=1}^N \, \beta_i  \, \rhobreg \left[ \rho(\tpi_i) : \rho(\mu)  \right] \label{eq:opt_mu} .
\end{align}
The arithmetic mean is recovered for $\rho(\tpi_i) = \tpi_i$ and any $f$ \citep{Banerjee2005}.
\item[(ii)] At this minimizing argument $\mu_\rho^*$, the value of the expected divergence (or right-hand side of \myeq{breginfo_optimization}) is called the  \textup{Rho-Tau Bregman Information} and is equal to a
gap in Jensen's inequality
%for the convex function $f$,
for mixture weights $\vwbeta$, inputs $\rho(\vupi) = \{ \rho(\tpi_i) \}_{i=1}^N$, and the convex functional
$\potential_f[\rho_\tpi] = \int f\of{\rho_\tpi(\vz))} d\vz$,
\begin{align}
    %\BregInfo{{f,\rho}} \of{{\bm{\tpi}, \vwbeta}} = \sum \limits_{i=1}^N \beta_i \, f\of{\rho({\tpi_i})} -   f\of{\sum \limits_{i=1}^N \beta_i \rho({\tpi_i})} 
    \BregInfo{{f,\rho}} \of{{\vupi, \vwbeta}} = \sum \limits_{i=1}^N \beta_i \, \potentialf \ofb{\rho({\tpi_i})} 
    %- \potentialf \ofb{\rho_{\mu_{\rho}}} 
    -  \potentialf \ofb{\rho(\mu^*_{\rho}) } %\rho_{\tpi_{\vwbeta}^{(\rho)}}} 
    %(\vupi, \vwbeta)
    . 
    \label{eq:breginfo_jg_vector}
\end{align}
\item[(iii)] Using $\mu \neq \mu^*_{\rho}$
%(\bm{\tpi}, \vwbeta)$ 
as the representative in \myeq{breginfo_optimization}, 
%$\mu_\rho(\vz; \vupi,\vwbeta)$, 
the suboptimality gap 
%in \myeq{breginfo_optimization} 
is a rho-tau Bregman divergence
\begin{align}
\rhobreg\big[
 \rho\of{\mu_\rho^*}  :  \rho\of{\mu}\big] = \sum \limits_{i=1}^N \beta_i ~ \rhobreg\left[ \rho\of{\tpi_i} : \rho\of{\mu}  \right] - \BregInfo{{f,\rho}}\of{\bm{\tpi}, \vwbeta} , \label{eq:bias_var}
\end{align}
where $\BregInfo{{f,\rho}}\of{\bm{\tpi}, \vwbeta}$ is evaluated at $\mu_\rho^*$ 
%\textup{(}\cref{eq:opt_mu}\textup{)} 
as in \cref{eq:breginfo_jg_vector}.
\end{itemize}
%\end{enumerate}
\end{restatable}
%\end{theorem}
\end{mybox}
See \myapp{thm_pf} for proof.     
We now consider the case of $N=2$ with unnormalized densities $\vupi = \{\tpi_0, \tpi_1 \}$ as input,
which is our main object of interest in \mysec{examples}.

\begin{restatable}
[\textrm{Rho-Tau Bregman Information and Divergences}]{example}{divergence}
\label{example:zhang_div}
Consider the expected rho-tau Bregman divergence minimization for a decomposable generator $\potential_f[\rho(\tpi)]$, inputs $\vupi = \{\tpi_0, \tpi_1 \}$, and weights $\vwbeta = \{ 1-\beta, \beta \}$.   
\begin{align}
\begin{split}
\BregInfo{{f,\rho}} \of{{\bm{\tpi}, \vwbeta}} &\coloneqq \min \limits_{\tpi} ~ (1-\beta)  \rhobreg[\rho(\tpi_0): \rho(\tpi)] + \beta ~  \rhobreg[\rho(\tpi_1): \rho(\tpi)] \\
&=(1-\beta) ~ \potentialf \ofb{\rho\of{\tpi_0}} + \beta \, \potentialf \ofb{\rho\of{\tpi_1}} - \potentialf\ofb{\rho\of{{\tpi_{\beta}^{(\rho)}}}} 
\end{split}
\end{align}
where the optimizing argument $\mu^*_\rho$ is given by the quasi-arithmetic mean
\begin{align}
       \tpi^{(\rho)}_{\beta}(\vz) &\coloneqq 
       %\mu_{\rho}(\vz; \vupi, \vwbeta) 
       \rho^{-1}\Big( (1-\beta)\rho\of{{\tpi_0}(\vz)} + \beta~ \rho\of{{\tpi_1}(\vz)} \Big) . \label{eq:quasi_q}
\end{align}
Introducing scaling factors to induce limiting behavior for $\beta\rightarrow\{0,1\}$, the rho-tau Bregman Information matches the representational $\alpha$-divergence  $D_{f,\rho}^{(\beta)}[\tpi_0:\tpi_1]$ of \citet{zhang2004divergence, zhang2013nonparametric},
\begin{align}
  %\hspace*{-.2cm} 
   &D_{f,\rho}^{(\beta)}[\tpi_0:\tpi_1] \coloneqq \frac{1}{\beta(1-\beta)} \, 
   \BregInfo{{f,\rho}} \of{{\vupi, \vwbeta}} 
  % \coloneqq \medmath{\frac{1}{\beta(1-\beta)} \BregInfoarg{f}{\rho}({\bracket{\tpi_0, \tpi_1},  \bracket{1-\beta, \beta}})}
  %\BregInfoarg{\potential_{f}}{\rho} \of{\bracket{\tpi_0, \tpi_1},  \bracket{1-\beta, \beta}} 
  \label{eq:scaled_breg} \\
  &\phantom{===} = \begin{cases} \frac{1}{\beta(1-\beta)} \Big( (1-\beta) \potentialf \ofb{\rho\of{\tpi_0}} + \beta \, \potentialf \ofb{\rho\of{\tpi_1}} - \potentialf\ofb{\rho\of{{\tpi_{\beta}^{(\rho)}}}} \Big) \hfill \quad \,\, (\beta \neq \, 1)  \\
      \rhobreg\,\,[\rho(\tpi_1):\rho(\tpi_0)] = \taubreg[\tau(\tpi_0):\tau(\tpi_1)] 
    \hfill \qquad \quad (\beta \rightarrow 1)\\
    \taubreg[\tau(\tpi_1):\tau(\tpi_0)] = \rhobreg\,\,[\rho(\tpi_0):\rho(\tpi_1)]
    \hfill \qquad \quad (\beta \rightarrow 0) 
    \end{cases} 
      \nonumber
\end{align}
where ${{\tpi_{\beta}^{(\rho)}}}$ is the quasi-arithmetic mean in \cref{eq:quasi_q}.
%= \rho^{-1}\big( (1-\beta) \rho({\tpi_0}) + \beta \rho({\tpi_1}) \big)$ is the quasi-arithmetic mean.   
%From the annealing path perspective, 
\myeq{scaled_breg} thus associates a divergence functional comparing $\tpi_0$ and $\tpi_1$ with each quasi-arithmetic mean with weight $\beta$ and representation function $\rho$. %along the path.

Following \citet{zhang2004divergence, zhang2013nonparametric}, we emphasize the `referential duality' in terms of the mixing parameter $\beta$, 
  where it is clear that, for example,
   $D_{f,\rho}^{(\beta)}[\tpi_0:\tpi_1]= D_{f,\rho}^{(1-\beta)}[\tpi_1:\tpi_0]$ .  `Representational duality' is expressed by the fact that the divergences in the $\rho$ and $\tau$ representations, respectively, are equivalent up to reordering of the arguments $\rhobreg[\rho(\firstarg):\rho(\secondarg)] = \taubreg[\tau(\secondarg):\tau(\firstarg)]$ \citep{zhang2004divergence, zhang2013nonparametric}.   
\end{restatable}

Using the well-known Eguchi relations \citep{eguchi1983second, eguchi1985differential}, \citet{zhang2004divergence, zhang2013nonparametric} analyze the statistical manifold structures $(\mathcal{M}, g, \nabla, \nabla^*)$ induced by the
family of $D_{f,\rho}^{(\beta)}$ divergences
%representational $\alpha$-divergence 
in \myeq{scaled_breg} on the space of unnormalized %\textcolor{red}{measurable functions} 
density functions
$\mathcal{M}$ (see \myapp{info_geo}).   
We prove the following proposition in \myapp{qgeo}, which identifies the $\rho$- and $\tau$-representations as a pair of affine coordinate systems for the dually flat geometry induced by the Bregman divergence $\rhobreg[\rho(\firstarg):\rho(\secondarg)]$. %(\citep{amari2016information} Ch. 1.5, 4.4).

\begin{mybox}
\begin{restatable}[\normalfont \textup{Geodesics for Rho-Tau Bregman Divergence}]{thm}{rhotau}
\label{thm:geodesic} %[Bregman Information]
 The curve $\gamma_t = \rho^{-1} \big( (1-t) \rho(\tpi_0) + t \, \rho(\tpi_1) \big)$ (with time derivative $\dot{\gamma_t} = \frac{d}{dt}\gamma_t$) is auto-parallel with respect to the primal affine connection $\nabla^{(1)}$ induced by the rho-tau Bregman divergence $\rhobreg[\rho(\firstarg):\rho(\secondarg)]$ for $\beta = 1$.   In other words, the following geodesic equation holds
 \begin{align}
    \nabla^{(1)}_{\dot{\gamma}_t} \dot{\gamma}_t  
   % = d_{\dot{\gamma}_t} \dot{\gamma_t} + \big(\dot{\gamma}_t\big)^{2} \cdot  \left(  \frac{\rho^{\prime\prime}(\gamma_t)}{\rho^{\prime}(\gamma_t)} \right) %\log \rho^{\prime}(\gamma)
     = 0. \label{eq:geodesic_eq}
\end{align}
The $\rho$-representation of the unnormalized density thus provides an affine coordinate system for the geometry induced by $\rhobreg[\rho(\firstarg):\rho(\secondarg)]$. 
\end{restatable}
\end{mybox}

\newcommand{\mymu}{\tpi}
\newcommand{\textsizerb}{\fontsize{9pt}{9pt}\selectfont}
\begin{table}[t]
\vspace*{-3pt}
\centering
%\begin{tabular}{c} \toprule
 \label{eq:jensen_gap_table}
  \resizebox{\textwidth}{!}{
\begin{tabular}{lccc}\toprule
   \multicolumn{4}{c}{\textsizerb Familiar Divergences from scaled Bregman Information 
   %or $\alpha$-Divergence
   } \\[2.1ex]
     \multicolumn{4}{c}{\qquad \qquad \text{\textsizerb $D_{f,\rho}^{(\beta)}[\tpi_0:\tpi_1] = \dfrac{1}{\beta(1-\beta)} \BregInfo{f,\rho}[\{\tpi_0,\tpi_1\}, \beta] \qquad \qquad \qquad \qquad 
     $}} \\[1.5ex]  \multicolumn{4}{c}{  \quad~ \qquad \qquad \qquad \qquad\qquad \qquad\text{\textsizerb  $= \dfrac{1}{\beta (1-\beta )} \Big( (1-\beta ) \, 
    \Z_q(0) + \beta  \, \Z_q(1) - \Z_q(\beta )  \Big) \geq 0. $} \qquad  }\\[5ex]
    \multicolumn{2}{c}{\text{\textsizerb and Quasi-Arithmetic Means:}} &  \multicolumn{2}{l}{ \text{\textsizerb inputs: $\{ \, \tpi_0, \,  \tpi_1 \,  \}$}}\\[1.55ex]
    \multicolumn{2}{c}{\multirow{2}{*}{\text{\textsizerb $\tpi^{(q)}_{\beta }(\vz)= \rho_q^{-1}\Big( (1-\beta) \rho_q\of{\tpi_0(\vz)} + \beta \, \rho_q\of{\tpi_1(\vz)} \Big)$}}} &  \multicolumn{2}{l}{  \text{\textsizerb weights: $\{ 1-\beta, \beta \}$ } } \\[1.55ex]
    \multicolumn{2}{c}{} & \multicolumn{2}{l}{\text{\textsizerb representation: $\rho_q(\tpi) = \log_q \tpi$}}
    % \multicolumn{4}{c}{ \text{for inputs:} $\{ \, \tpi_0, \,  \tpi_1 \,  \}$ \quad \text{weights:} $\{ 1-\beta, \beta \}$ }\\[1.5ex]
    %  \multicolumn{4}{c}{\text{and embedding function: } $\rho(\tpi) = \log_q \tpi$}
    \\[1.5ex] \midrule
    & \multicolumn{2}{c}{Bregman Divergence} & 
    %$\alpha$-Divergence 
    \text{ Breg. Info.}
    \\[0.5ex]
   Convex Function 
   %$\Psi(\mu_{\beta}^{(q)})$ 
   & $\beta  \rightarrow 0$ & $\beta  \rightarrow 1$ & \footnotesize $\beta \not \in \{0,1\}$\\ \cmidrule[\heavyrulewidth]{1-4} 
  \text{\scriptsize $\log \Z_1(\beta) = \log \int \mymu^{(\text{geo})}_{\beta}(\vz) d\vx $} 
  &  $D_{\KL}[\pi_0 : \pi_1]$&  $D_{\KL}[\pi_1 : \pi_0]$ & $D_R^{(\beta)}[\pi_0 : \pi_1]$ \\[2.5ex]
   $\Z_1(\beta)
   %= \potential_f[\rho(\mymu^{(\text{geo})}_{\beta})] 
   =\int \mymu^{(\text{geo})}_{\beta}(\vz) d\vx$ & $D_{\KL}[\tpi_0 : \tpi_1]$ & $D_{\KL}[\tpi_1 : \tpi_0]$ & $D_A^{(\beta)}[\tpi_0 : \tpi_1]$ \\[2.5ex]
   % $\potential\ofb{\rho(\tpi^{(q)}_\beta)} = \int \mymu^{(q)}_{\beta}(\vz) d\vx$
    $\Z_q(\beta)
    %= \potential_f[\rho(\mymu^{(q)}_{\beta})] 
    = \int \mymu^{(q)}_{\beta}(\vz) d\vx$ 
    & $D_A^{(q)}[\tpi_0 : \tpi_1]$ &  $D_A^{(q)}[\tpi_1 : \tpi_0]$ & $ D_Z^{(\beta, q)}[\tpi_0 : \tpi_1]$ \\[2.5ex]
   $\Z_0(\beta)
   %= \potential_f[\rho(\mymu^{(\text{arith})}_{\beta})] 
   = \int \mymu^{(\text{arith})}_{\beta}(\vz) d\vx$ & $D_{\KL}[\tpi_1 : \tpi_0]$ & $D_{\KL}[\tpi_0 : \tpi_1]$ & $D_{\mathrm{JS}}^{(\beta)}[\tpi_0 : \tpi_1]$ \\[2.75ex] 
   $\potential_f\ofb{\rho(\tpi^{(q)}_\beta)} = \int \mymu^{(q)}_{\beta}(\vz)^{2-q} d\vx$ & $D_{B}^{(2-q)}[\tpi_1:\tpi_0] $  &  $D_{B}^{(2-q)}[\tpi_0:\tpi_1] $ & see  Ex. \ref{example:beta_bi}  
   \\[2.5ex] 
      {\text{\small $\potential_f\ofb{\rho(\tpi^{(q)}_\beta)} = \int \mymu^{(q)}_{\beta}(\vz)^{\lambda+1-q} d\vx$ }} & $D_{C}^{(q, \lambda)}[\tpi_1:\tpi_0] $  &  $D_{C}^{(q, \lambda)}[\tpi_0:\tpi_1] $ & 
      see \myeq{breg_info_cichocki}  \\
   \bottomrule
\end{tabular}
}
\caption{ Bregman Information and Divergence Functionals (see examples throughout \mysec{examples}).
For a convex generator $\Psi_f[\rho]$ (\mysec{rhotau}) or $\Z_q(\beta)$ (\mysec{parametric}), the Bregman Information is a gap in Jensen's inequality (top box) obtained by minimizing the expected Bregman divergence as in \myeq{general_div_min}.
The scaled Bregman Information $\frac{1}{\beta(1-\beta)} \BregInfo{f,\rho}$ matches the representational $\alpha$-divergence $D_{f,\rho}^{(\beta)}$ from \cite{zhang2004divergence} (see \cref{example:zhang_div}), and encompasses all examples in the lower box.
Note that $D_R^{(\alpha)}$ is R\'enyi's $\alpha$-divergence, $D_A^{(\alpha)}$ 
% (or $D_A^{(q)}$) 
is Amari's $\alpha$-divergence, $D_{\mathrm{JS}}^{(\alpha)}$ is the Jensen-Shannon divergence with mixture weight $\alpha$,  $D_Z^{(\beta, q)}$ is the $(\alpha,\beta)$ divergence of \citet{zhang2004divergence}, $D_{B}^{(q)}$ is the Beta-divergence of order $q$ \citep{basu1998robust}, and $D_{C}^{(\lambda, q)}$ is the $(\alpha,\beta)$ divergence of \citet{cichocki2010families}.} \label{tab:jensen_gaps} %Familiar Divergences from gaps in Jensen's Inequality \\[1.5ex] 
\vspace*{-15pt}
\end{table}

\vheader
\vheader
\section{Annealing Paths and Divergence Minimization}\label{sec:examples}
Using the rho-tau Bregman Information framework from the previous section, we are now able to understand and extend the variational representations of \gls{MCMC} paths from \citet{grosse2013annealing, masrani2021q}.   
%We emphasize the following observations
\begin{itemize}
    \item Following \mythm{breg_info}, we derive $q$-annealing paths \citep{masrani2021q} from an expected divergence minimization for \textit{any} rho-tau Bregman divergence with $\rho(\tpi) = \log_q \tpi$ in \mysec{cichocki}.   
    %The Amari $\alpha$-divergence minimization interpretation from \citet{masrani2021q} or \cref{eq:qpath_minimization} corresponds to a rho-tau Bregman divergence for $\alpha = q$ and $\beta=1$.
    For different choices of convex function $f$, we recover the Amari $\alpha$-divergence minimization interpretation of $q$-paths from \citet{masrani2021q} or \cref{eq:general_div_min}, or novel interpretations via minimization the $\beta$-divergence or the ($\alpha,\beta$)-divergence from \citet{cichocki2010families}\cite{cichocki2011generalized}.
    %\item  Amari $\alpha$-divergence of order $\alpha$ may be derived as \textit{either} a rho-tau Bregman divergence for $\alpha = q$ and $\beta=1$, or as a rho-tau Bregman Information as under geometric averaging with $q=1$ and $\rho(\tpi) = \log \tpi$.
    \item In the above interpretation, the Amari $\alpha$-divergence corresponds to a rho-tau Bregman divergence for $\alpha = q$ and $\beta=1$.  In \cref{example:geo},  we further show that the $\alpha$-divergence can arise from the choice of mixture parameter $\beta$.  In particular, under geometric averaging with $q=1$ and $\rho(\tpi) = \log \tpi$, the Amari $\alpha$-divergence corresponds to the rho-tau Bregman Information with order  $\alpha = \beta$.% instead of the representation  
    \item In \mysec{parametric}, we provide a parametric interpretation of the rho-tau Bregman Information, where the mixing weight $\beta$ is the natural parameter for a deformed $q$-exponential family constructed from arbitrary endpoint densities. 
\end{itemize}
We summarize our examples in \cref{tab:jensen_gaps}, and refer the reader to \citet{masrani2021q} for an empirical study of $q$-paths for \gls{MCMC} applications in Bayesian inference and marginal likelihood evaluation.   

We begin by providing two common examples of the rho-tau Bregman Information for the arithmetic and geometric mixture paths.  First, recall the choices leading to the forward and reverse \kl divergences in \cref{example:kl},
\begin{equation}
\begin{aligned}
  % \hspace*{-.2cm} 
       \rho(\tpi) &= \log \tpi \qquad   \\
       f(\rho) &= \exp\{\rho\}-\rho -1 \qquad 
       % \potentialf\ofb{\rho_{\tpi}} &= 
     %\int f\of{\rho(\tpi(\vz)} d\vz = 
     %- \int \log \tpi(\vz) d\vz + \int \tpi(\vz) d\vz - 1 
\end{aligned}\qquad 
\begin{aligned}
  % \hspace*{-.2cm} 
      \tau(\tpi) &= \tpi   \qquad  \\
      f^*(\tau) &= \tau \log \tau - \tau + 1   \qquad  
     % \potentialdual \ofb{\tau_{\tpi}}&= 
    %\int f\of{\rho(\tpi(\vz)} d\vz \, = 
    %\int \tpi(\vz) \log \tpi(\vz) d\vz - \int \tpi(\vz) d\vz + 1,  
\end{aligned}
\end{equation}
where the dual decomposable generator $\potentialdual \ofb{\tau({\tpi})}$ is a functional of the density directly, and matches the negative Shannon entropy
\begin{align} 
 \potentialdual \ofb{{\tpi}} &= 
    %\int f\of{\rho(\tpi(\vz)} d\vz \, = 
    \int \tpi(\vz) \log \tpi(\vz) d\vz - \int \tpi(\vz) d\vz + 1.\label{eq:shannon}
\end{align}

\begin{example}[\textrm{Jensen-Shannon Divergence as Bregman Information for the Mixture Path}]\label{example:jsd}
The (weighted) \gls{JSD} \citep{burbearao1982, lin1991divergence} is the most natural example of a Bregman Information, with $ \DJSb\big[\tpizero : \tpi_1 \big] =  \BregInfo{{f^*,\tau}} \of{{\vupi, \vwbeta}} $ corresponding to the forward \kl divergence in \cref{example:kl}.   First, recall the definition
\begin{align}
\DJSb\big[\tpizero : \tpi_1 \big] 
    %&= \BregInfoBeta{\potentialf_0}{\beta}\big[\tpi_0 : \tpi_1
    % \{\tpi_0, \tpi_1\}, \{1-\beta,\beta\} 
    %\big] 
    \coloneqq  
  \int (1-\beta) \tpizero(\vz) \log \frac{\tpizero(\vz)}{\tpimix(\vz)}d\vz + \beta \, \tpi_1(\vz) \log \frac{\tpi_1(\vz)}{\tpimix(\vz)} d\vz. \nonumber
\end{align}
The arithmetic mixture $\tpimix(\vz) \coloneqq (1-\beta) \tpizero(\vz) + \beta \tpi_1(\vz)$ minimizes the expected divergence in the $\tau(\tpi)=\tpi$ representation, which yields   
%Since $\tau(\tpi) = \tpi$, this mixture is the minimizing argument of the expected divergence minimization defining ,
%Since $\tau(\tpi) = \tpi$, the arithmetic mixture  is the minimizing argument and we have   %Minimizing the expected divergence 
\begin{align}
   \BregInfo{{f^*, \tau}} \of{{\vupi, \vwbeta}}    
   %\medmath{  \nonumber \\ %\\[1.25ex
   %\label{eq:jsd}\\
   &= \min \limits_{\tpi}~ (1-\beta) \taubreg[\tau(\tpizero) : \tau(\tpimix) \big] + \beta  \taubreg[  \tau(\tpi_1) : \tau(\tpimix) \big] \nonumber \\
    &= (1-\beta) \DKL \big[ \tpizero : \tpimix \big] + \beta  \DKL[  \tpi_1 : \tpimix \big]  \nonumber \\ %\label{eq:jsd_argmin2}\\
    &=(1-\beta)\potentialdual[\tpi_0]  + \beta \potentialdual[\tpi_1]- \potentialdual[\tpimix] , \nonumber
\end{align}
where $\potentialdual[\tpi]$ is the negative Shannon entropy as in \cref{eq:shannon}.
\end{example}

\begin{example}[\textrm{Amari $\alpha$-Divergence as Bregman Information for the Geometric Path}]\label{example:geo}
%Using the $\rho(\tpi)= \log \tpi$ representation and 
When minimizing the reverse \kl divergence $\rhobreg[\rho({\firstarg}):\rho(\secondarg)] = \DKL[\secondarg:\firstarg]$, note that optimizing over the second argument of the rho-tau Bregman divergence corresponds to optimizing over the first argument of the \kl divergence.  
The geometric mixture, or quasi-arithmetic mean for $\rho(\tpi) = \log \tpi$, given by
\begin{align} %$\tpigeo(\vz) = \tpizero(\vz)^{1-\beta} \tpi_1(\vz)^{\beta}$
   \hspace*{-.35cm} \tpigeo(\vz) &\coloneqq \tpizero(\vz)^{1-\beta} \tpi_1(\vz)^{\beta} \\
   &= \argmin \limits_{\tpi} \, (1-\beta) \DKL[ \tpi : \tpizero] + \beta \,  \DKL[ \tpi : \tpi_1] \, , \nonumber %\label{eq:kl_arbitrary} 
\end{align}
provides the minimizing argument as in \citet{grosse2013annealing}.   After simplifying, the scaled Bregman Information recovers the Amari $\alpha$-divergence in \cref{eq:amari_alpha},
\small
\begin{align}
   \hspace*{-.1cm} 
   %D_{f,\rho}^{(\beta)}[\tpi_0:\tpi_1]
& \frac{1}{\beta(1-\beta)} \BregInfo{{f, \rho}} \of{{\vupi, \vwbeta}} \label{eq:amari_breg}\\
&\phantom{==}= \frac{1}{\beta(1-\beta)} \bigg( (1-\beta) \DKL[ \tpigeo: \tpizero] + \beta \,  \DKL[ \tpigeo : \tpi_1] \bigg) \nonumber \\
  &\phantom{==}=  \frac{1}{\beta(1-\beta)} \bigg( (1-\beta)  \potentialf\ofb{\rho_{\tpi_0}} + \beta \potentialf\ofb{\rho_{\tpi_1}} -  \potentialf\ofb{\rho_{\tpi_\beta^{(\text{geo})}}} \bigg) \nonumber \\ 
   %&= \frac{1}{\beta(1-\beta)} \bigg( (1-\beta) Z(0) + \beta Z(1) -  Z(\beta) \bigg) \nonumber \\
   &\phantom{==} = \frac{1}{\beta(1-\beta)} \bigg( (1-\beta) \int \tpi_0(\vz) d\vz + \beta \int \tpi_1(\vz) d\vz - \int 
    \tpi_0(\vz)^{1-\beta} \tpi_1(\vz)^{\beta}  % \left[ (1-\beta) \tpizero(\vz)^{1-q} + \beta \, \tpi_1(\vz)^{1-q} \right]_+^{\frac{1}{1-q}} 
    d\vz \bigg) \nonumber \\
    &\phantom{==}= D_A^{(\beta)}[\tpi_0 : \tpi_1].  \nonumber
\end{align}
\normalsize
Note that the order of the $\alpha$-divergence is set by the mixture parameter $\beta$, which is analogous to an inverse temperature parameter in maximum entropy or lossy compression applications  \citep{jaynes1957information, tishby1999information, bercher2012simple, alemi2018fixing}.
\end{example}

\vheader
\robsubsection{{$q$-Paths from a Divergence Minimization Perspective}}\label{sec:cichocki}
\vheader
Following \mythm{breg_info}-\ref{thm:geodesic} and the examples above, the $q$-annealing paths from \citet{masrani2021q}, which correspond to a quasi-arithmetic mean of two endpoint unnormalized densities $\bm{\tpi}= \{\tpi_0, \tpi_1 \}$ with weight $\beta$, should arise from \textit{any} rho-tau Bregman divergence in the $\rho_q(\tpi) = \log_q \tpi$ representation,
\begin{align}
\begin{split}
 \qpath(\vz) &= \exp_q \bracket{(1-\beta) \log_q \tpi_0(\vz) + \beta \log_q \tpi_1(\vz) } 
    \label{eq:qpath_minimization}, \\
    &= \argmin \limits_{\tpi} (1-\beta) \rhobreg\big[\rho({\tpi_0}):\rho({\tpi})\big] + \beta \rhobreg\big[\rho({\tpi_1}):\rho({\tpi})\big] 
    \end{split}
\end{align}  
The choice of convex function $f$ or dual representation $\tau$ is an additional degree of freedom in specifying the divergence, suggesting that the $q$-annealing path can be viewed as the Bregman barycenter for a wide range of divergences.

In \mysec{rt_breg}, we describe a family of rho-tau Bregman divergences corresponding to the $(\alpha,\beta)$ divergence of \citet{cichocki2010families}\cite{cichocki2011generalized}, which includes the Amari $\alpha$-divergence and Beta divergence as special cases.   We derive the corresponding rho-tau Bregman Informations and divergence minimization interpretations of $q$-paths in \mysec{rt_bi}.

\robsubsubsection{{Rho-Tau Bregman Divergence}}\label{sec:rt_breg}
Consider the family of dual representations $\tau(\tpi)$ defined by another deformed $q$-logarithm of order $\lambda$,
% \small 
\begin{align}
\rho(\tpi) = \log_q(\tpi) = \frac{1}{1-q}\tpi(\vz)^{1-q}-\frac{1}{1-q} \qquad \tau(\tpi) = \log_{1-\lambda}(\tpi) = \frac{1}{\lambda} \tpi(\vz)^{\lambda} - \frac{1}{\lambda} . 
\nonumber
%\label{eq:rhotau_alphabeta}
\end{align}
\normalsize
%Calculating  
As in \mysec{rhotau}, the above $\rho$ and $\tau$ representations will be conjugate with respect to the convex function for which $f^{\prime} = \tau \circ \rho^{-1}$ (see \myeq{conj_condition}).  Choosing an additive constant to induce limiting behavior as $q \rightarrow 1$ or $\lambda \rightarrow 0$ (see \myapp{info_geo} \cref{tab:limiting_kl}),
%\myapp{limits_rhotau}), 
we write the convex function $f$ and decomposable generator $\potentialf[\rho({\tpi})]$ as %$ = \int f\of{\rho(\tpi(\vx))}d\vx $
\small
\begin{align}
    f(\rho) &= \frac{1}{\lambda}\frac{1}{\lambda + 1-q} \exp_q\{ \rho \}^{{\lambda+1-q}
    %\frac{1-q}
    } - \frac{1}{\lambda} \rho - \frac{1}{\lambda(\lambda+1-q)}. \label{eq:alpha_beta_potential}\\[1.25ex]
    \potentialf[\rho_{\tpi}] 
    %= \int f\of{\rho(\tpi(\vx))}d\vx 
    &= \frac{1}{\lambda}\frac{1}{\lambda + 1-q} \int \tpi(\vx)^{\lambda+1-q} d\vx - \frac{1}{\lambda}\frac{1}{1-q}\int \tpi(\vx)^{1-q} d\vx - \frac{1}{\lambda(\lambda+1-q)}. \nonumber    
\end{align}
\normalsize
Note that the Bregman divergence generated by $f(\rho)$ or $\Psi_f[\rho]$ are invariant to the second and third terms, which are affine in $\rho$ \citep{Banerjee2005}.

Finally, the rho-tau Bregman divergence $\rhobreg[\rho(\firstarg):\rho(\secondarg)]$ derived from \cref{eq:alpha_beta_potential} matches the $(\alpha,\beta)$ divergence of \citet{cichocki2010families}\cite{cichocki2011generalized},
\begin{align}
 \medmath{ D_C^{(q, \lambda)}[\firstarg: \secondarg] \coloneqq \frac{1}{\lambda(1-q) (\lambda+1-q)} \Big( }&\medmath{(1-q) \int \firstarg(\vz)^{\lambda+1-q} d\vx + \lambda \int \secondarg(\vz)^{\lambda+1-q} d\vx } \nonumber \\
&\phantom{===} \medmath{- (\lambda+1-q) \int \firstarg(\vz)^{\lambda}\secondarg(\vz)^{1-q}  d\vx  \Big). } \label{eq:cichocki-amari}
\end{align}
\normalsize
%To the best of our knowledge, 
%Note that this recovers the $\tau$-deformed gauge and $\alpha$-divergence of order $q$ for $\lambda = q$, or the $\tau$-id gauge and the Beta divergence of order $1-q$ for $\lambda = 1$.  

\begin{example}[Amari $\alpha$-Divergence as a Rho-Tau Bregman Divergence]\label{example:alpha}   Choosing $\lambda = q$, the dual representation becomes the $1-q$ logarithm $\tau(\tpi) = \log_{1-q}(\tpi) =  \frac{1}{q}\tpi(\vz)^{q} - \frac{1}{q}$, which is known as the $\tau$-deformed gauge in \citet{naudts2018rho}.   Using the generator in \cref{eq:alpha_beta_potential}, the Bregman divergence $\rhobreg\big[\rho({\firstarg}):\rho({\secondarg})\big]$ is the Amari $\alpha$-divergence
\begin{align}
 D_{A}^{(q)}[\firstarg:\secondarg] 
  &=\frac{1}{q} \int \firstarg(\vz) d\vz + \frac{1}{1-q} \int \secondarg(\vz) d\vz - \frac{1}{q}\frac{1}{1-q} \int \firstarg(\vz)^{1-q}\secondarg(\vz)^{q} d\vz \nonumber 
  %&=\rhobreg\big[\rho({\firstarg}):\rho({\secondarg})\big]
\end{align}
%The $q$-annealing path now arises from an expected $\alpha$-divergence minimization $\min_{\tpi} (1-\beta) D_A^{(q)}\big[\tpi_0: \tpi \big] + \beta D_A^{(q)}\big[\tpi_1: \tpi \big]$ as in \cref{eq:masrani} and \citet{masrani2021q}.
Note that, in contrast to \cref{example:geo}, the Amari $\alpha$-divergence is a rho-tau Bregman divergence (instead of a Bregman Information), with the order set by the representation parameter $q$ (instead of the mixing weight $\beta$).  
    
\end{example}

\begin{example}[Beta-Divergence as a Rho-Tau Bregman Divergence] \label{example:beta}
Choosing $\lambda = 1$ and modifying the dual representation to be $\tau(\tpi) = \tpi$, we have the following choices of $f$ and $f^*$
\begin{align}
          \rho_{\tpi}(\vz) &=  \log_{q} \tpi(\vz)  \qquad \qquad  \tau_{\tpi}(\vz) =  \tpi(\vz)
     \label{eq:beta_rho}. \\
    {f(\rho)} &=  {\frac{1}{2-q} \left[ 1 + (1-q) \rho \right]_{+}^{\frac{2-q}{1-q}} -\frac{1}{2-q} \quad \,\,\, = \log_{q-1} \left( \exp_{q}\{\rho\} \right)}  \qquad  \nonumber
\end{align}
which is known as the $\tau$-identity gauge in \citep{naudts2018rho, zhang2021entropy}.
The associated rho-tau Bregman divergence is the Beta-divergence (\citep{basu1998robust, murata2004information,eguchi2006information}, \citep{naudts2011generalised} Ch. 8) matches \cref{eq:cichocki-amari} for $\lambda = 1$,
\begin{align}
\medmath{ D_{B}^{2-q}[\secondarg:\firstarg] }
        & 
        \medmath
        {\coloneqq 
        \frac{1}{1-q}\frac{1}{2-q} \int  \secondarg(\vz)^{2-q} d\vz + \frac{1}{2-q}\int  \firstarg(\vz)^{2-q}d\vz  - \frac{1}{1-q} \int  \secondarg(\vz) \firstarg(\vz)^{1-q} d\vz 
        } \nonumber . \\
        &=\rhobreg[ \rho({\firstarg}):\rho({\secondarg})]\label{eq:beta}
\end{align}
where we note that the order of the arguments is reversed in $D_B$ compared to $\rhobreg$.
%will recover as the rho-tau Bregman divergence $\rhobreg[ \rho({\firstarg}):\rho({\secondarg})]$
We emphasize that the representation parameter $q$ sets the order of the Beta divergence, rather than the mixing parameter $\beta$. 

For the $\tau$-identity case, we call attention to the form of the dual generator
\begin{equation}
\begin{aligned}
    \medmath{f^*(\tau)} &=\frac{1}{1-q}\frac{1}{2-q} \tau^{2-q}-\frac{1}{1-q} \tau + \frac{1}{2-q} \,\, %= \frac{1}{1-q} \log_{q-1}\left( \tau \right) - \frac{1}{1-q}\tau + \frac{1}{1-q} .} 
    \\
    \potential^*_{f^*}[\tpi] &= 
    \frac{1}{(1-q)(2-q)} \int \tpi(\vz)^{2-q} d\vz - \int \frac{1}{1-q}\tpi(\vz) d\vz + \frac{1}{2-q}.  \nonumber
\end{aligned} 
\end{equation}
In particular, since $\tau(\tpi) = \tpi$, we can directly interpret the dual generator $\potential^*_{f^*}[\tpi] $ as the negative Tsallis entropy functional of order $2-q$ \citep{tsallis1988possible, tsallis2009introduction, naudts2011generalised}.   The Beta divergence thus corresponds to the Bregman divergence generated by the negative Tsallis entropy (\citep{naudts2011generalised} Sec. 8.7, compare with \cref{example:kl}).
\footnote{The identity dual representation leads to favorable properties for divergence minimization under linear constraints. \citet{csiszar1991least} characterize Beta divergences as providing scale-invariant projection onto the set of positive measures satisfying expectation constraints.  \citet{naudts2011generalised} Ch. 8 discuss related thermodynamic interpretations.   Here, the Beta divergence is preferred in place of the $\alpha$-divergence, which induces \textit{escort} expectations due to the deformed dual representation $\tau(\tpi) = \log_{1-q}\tpi= \frac{1}{q}\tpi(\vz)^{q} - \frac{1}{q}$ from \cref{example:alpha} \citep{naudts2011generalised, zhang2021entropy}.}
\end{example}

\robsubsubsection{{Rho-Tau Bregman Information}}\label{sec:rt_bi}
Using \mythm{breg_info} with $\rho(\tpi) = \log_q(\tpi)$, we conclude that the $q$-annealing path minimizes the expected Cichocki-Amari divergence for any choice of $(q,\lambda)$,
%can derive a further Cichocki-Amari $(\alpha,\beta)$-divergence minimization interpretation of the $q$-annealing path,
\begin{align}
    \qpath(\vz) 
    &= \argmin \limits_{\tpi} \, (1-\beta) D_C^{(q, \lambda)}[\tpi_0 : \tilde{\pi}] + \beta D_C^{(q, \lambda)}[\tpi_1 : \tilde{\pi}] ,
\end{align}

Interpreting the Bregman Information as a gap in Jensen's inequality for the generator $\potentialf[\rho_{\tpi}]$ of the Cichocki-Amari divergence in \cref{eq:alpha_beta_potential}, we have
\begin{align}
   \hspace*{-.2cm} 
   %\frac{1}{\beta(1-\beta)}  
   \BregInfo{f,\rho} \of{\vupi, \vwbeta}& = 
   (1-\beta) \potentialf \ofb{\rho\of{\tpi_0}} + \beta \, \potentialf \ofb{\rho\of{\tpi_1}} - \potentialf\ofb{\rho\of{\tpi_\beta^{(q)}}} \label{eq:breg_info_cichocki}
   \\
   &=
    %\frac{1}{\beta(1-\beta)} \Big( 
    %{\frac{1}{\beta(1-\beta)} 
   \frac{1}{\lambda (\lambda+1-q)}\Big(  (1-\beta) \int \tpi_0(\vx)^{\lambda+1-q} d\vx   \, + \, \beta \int \tpi_1(\vx)^{\lambda+1-q} d\vx \nonumber  \\
    &\phantom{=\beta \frac{1}{\lambda (\lambda+1-q)}\Big(  (1-\beta) \int \tpi_0(\vx)^{\lambda+1-q} d\vx   \, } -  \int \qpath(\vx)^{\lambda+1-q} d\vx \nonumber
    \Big), %\nonumber %\label{eq:breg_info_cichocki}
    %\Big)
%\end{split}
\end{align}
\normalsize
The rho-tau Bregman Information thus reduces to a gap in Jensen's inequality for a functional $\Psi_f: \rho(\tpi) \rightarrow \bbR$ which, expressed directly in terms of $\tpi$, integrates each unnormalized density raised to the $\lambda +1 - q$ power. 
An analogous Bregman Information can be derived for the dual Bregman divergence in the $\tau = \log_{1-\lambda}$-representation (see examples below).

As in \cref{example:zhang_div}, scaling the Bregman Information in \myeq{breg_info_cichocki} by $\frac{1}{\beta(1-\beta)}$ suggests a further family of divergence functionals $D_{f,\rho}^{(\beta)}[\tpi_0:\tpi_1]$ for comparing $\tpi_0$ and $\tpi_1$.   We summarize the Eguchi relations for these divergences in \myapp{info_geo}.

\begin{example}[Bregman Information induced by Amari $\alpha$-Divergence]\label{example:alpha_bi}
For $\lambda = q$ and the Amari $\alpha$-divergence from \cref{example:alpha}, we recover the divergence minimization interpretation from \cref{eq:qpath_minimization} and \citet{masrani2021q}
\begin{align}
\qpath(\vz) = \argmin \limits_{\tpi} (1-\beta) D_A^{(q)}\big[\tpi_0: \tpi \big] + \beta \, D_A^{(q)}\big[\tpi_1: \tpi \big] 
%= \argmin \limits_{\tpi} (1-\beta) \rhobreg\big[\rho({\tpi_0}):\rho({\tpi})\big] + \beta \rhobreg\big[\rho({\tpi_1}):\rho({\tpi})\big] 
\end{align}
\normalsize
The $(\alpha,\beta)$ divergence from \citet{zhang2004divergence, zhang2013nonparametric} arises as the corresponding scaled Bregman Information, where we rename the $(\alpha,\beta)$ divergence as $D_Z^{(\beta,q)}[\tpi_0 : \tpi_1]$ to clarify the role of each parameter.
%Letting $D_Z^{(\beta,q)}[\tpi_0 : \tpi_1] = \frac{1}{\beta(1-\beta)} \BregInfo{\potentialf} \of{\vupi, \vwbeta}$ 
\begin{align}
  %\frac{1}{\beta(1-\beta)} \BregInfo{\potentialf} \of{\vupi, \vwbeta} 
  D_Z^{(\beta,q)}[\tpi_0 : \tpi_1] &:= 
    \frac{1}{\beta(1-\beta)}\frac{1}{q} \left( \int  (1-\beta) \, \tpi_0(\vz)  + \beta \, \tpi_1(\vz)-  
    \qpath(\vz) \, % \left[ (1-\beta) \tpizero(\vz)^{1-q} + \beta \, \tpi_1(\vz)^{1-q} \right]_+^{\frac{1}{1-q}} 
    d\vz \right)   \label{eq:qpath_normalizer} \\
    &\phantom{:}=  \frac{1}{\beta(1-\beta)} \BregInfo{f,\rho} \of{\vupi, \vwbeta} . \nonumber
\end{align}
Since $\lambda + 1 - q = 1$ for $\lambda=q$, the divergence measures the difference between a mixture of normalization constants and the normalization constant of a quasi-arithmetic mean.
Note that $D_Z^{(\beta,q)}$ is also an $f$-divergence using  $f^{(\beta, q)}(u) = \frac{1}{\beta(1-\beta)}\frac{1}{q} \big( (1-\beta) + \beta u - \big( (1-\beta) + \beta u^{1-q} \big)^{\frac{1}{1-q}} \big)$, and is shown by \citet{zhang2004divergence} to be the unique family of measure-invariant divergences satisfying the homogeneity condition $D[c \tpi_a : c \tpi_b] = c D[ \tpi_a :  \tpi_b]$ (see \citep{hardy1953} pg. 68, \citep{amari2007integration}).   

Interpretations for the dual representation $\tau(\tpi) = \log_{1-q}(\tpi)$ are analogous, with $\frac{1}{\beta(1-\beta)} \BregInfo{f^*,\tau} \of{\vupi, \vwbeta} = D_Z^{(\beta,1-q)}[\tpi_0 : \tpi_1]$.  Further identities can be derived using the referential-representational dualities described below \cref{example:zhang_div}.
\end{example}

\begin{example}[Bregman Information induced by the Beta-Divergence]\label{example:beta_bi}
For $\lambda = 1$, the $q$-annealing path arises from the expected Beta-divergence minimization over the first argument of $D_B^{2-q}$ or second argument of $D_f$
\begin{align}
    \qpath(\vz) 
    &= \argmin \limits_{\tpi} \,  (1-\beta) \DBeta{2-q}[\tpi :\tpi_0] + \beta \, \DBeta{2-q}[\tpi :\tpi_1] .
\end{align}
The resulting Bregman Information is a Jensen gap for a functional which raises each density to the power $2-q$ and integrates,
%for normalization constants of $\tpi(x)^{2-q}$,
\begin{equation}
%\begin{align}
\resizebox{.95\textwidth}{!}{$
    \hspace*{-.2cm} 
    %\frac{1}{\beta(1-\beta)} 
    \BregInfo{f,\rho} \of{\vupi, \vwbeta} 
    = 
    \frac{1}{2-q} \left( (1-\beta) \int \tpi_0(\vz)^{2-q} d\vz + \beta \int \tpi_1(\vz)^{2-q} d\vz - \int \qpath(\vz)^{2-q}  d\vz \right). \nonumber
$}
\end{equation}

\normalsize
Finally, we note that the mixture path minimizes the expected dual divergence $\taubreg[\tau({\firstarg}):\tau({\secondarg})]$, where the resulting Bregman Information becomes
\small 
\begin{align}
    \hspace*{-.1cm} 
    %\frac{1}{\beta(1-\beta)} 
    \BregInfo{f^*,\tau} \of{\vupi, \vwbeta}  &= 
     \frac{1}{1-q}\frac{1}{2-q} \left( (1-\beta) \int \tpi_0(\vz)^{2-q} d\vz + \beta \int \tpi_1(\vz)^{2-q} d\vz - \int \tpimix(\vz)^{2-q}  d\vz \right). \nonumber 
\end{align}
\normalsize
\end{example}

\robsubsection{{ Parametric Interpretation using $q$-Exponential Family with Parameter $\beta$}}\label{sec:parametric}

In this section, we interpret $q$-paths between given endpoints $\bm{\tpi}=\{\tpi_0, \tpi_1\}$ as a one-parameter deformed exponential family with natural parameter $\beta$, allowing us to highlight connections with Rényi's $\alpha$-divergence and the $\alpha$-mixture family \citep{amari2007integration, wong2021tsallis}.

Our starting point is to observe that the $q$-path in \cref{eq:qpath_minimization} can be written as a one-dimensional deformed `likelihood ratio exponential family' \citep{masrani2021q, brekelmans2020lkd, brekelmans2020tvo}, with $\tpi_0(\vz)$ as a base density, the $q$-log likelihood ratio $T(\vx) = \log_q \frac{\tpi_1(\vz)}{\tpi_0(\vz)}$ as the sufficient statistic, and the mixing weight $\beta$ as the natural parameter.   Assuming that $\tpi_1$ is absolutely continuous with respect to $\tpi_0$, 
\begin{align}
    \qpath(\vz) = \tpi_0(\vz) \exp_q \Big\{ \beta \cdot \log_q \frac{\tpi_1(\vz)}{\tpi_0(\vz)} \Big\} \label{eq:q_lref} .
\end{align}
%The density ratio $\qpath(\vz)/\tpi_0(\vz)$ is $\rho$-affine when considering $\rho(\qpath(\vz)/\tpi_0(\vz))= \beta \cdot T(\vz)$.
A key observation is
that the density ratio $\qpath(\vz)/\tpi_0(\vz)$ is $\rho$-affine in the $\log_q$ representation $\rho(\qpath(\vz)/\tpi_0(\vz))= \beta \cdot T(\vz)$ which, for given endpoint densities $\{\tpi_0, \tpi_1\}$, 
links the parametric family to the quasi-arithmetic mean and rho-tau Bregman Information divergence functionals (also see \myapp{parametric_nonparametric}).

\begin{example}[Parametric View of $q$-Paths and Zhang $(\alpha,\beta)$ Divergence]\label{example:q_path_para}
We will consider the multiplicative normalization constant for the $q$-likelihood ratio family 
%in \myeq{q_lref} 
as the convex generating function,
%in our later analysis, with
%with
\begin{align}
\frac{1}{q} {\Z}_q(\beta) = \frac{1}{q}\int \qpath(\vz) d\vz = \frac{1}{q}\int \tpi_0(\vz) \exp_q \Big\{ \beta \cdot \log_q \frac{\tpi_1(\vz)}{\tpi_0(\vz)} \Big\} d\vz, \label{eq:z_qexp}
\end{align}
%For the geometric path, we recover $Z_1(\beta) = \int \tpi_0(\vz)^{1-\beta} \tpi_1(\vz)^{\beta} d\vz$ (see \mysec{geo_path}).   
%We can also write the normalization constant as $Z_q(\btheta)$ for parametric deformed families 
where for the geometric path, we recover $\Z_1(\beta) = \int \tpi_0(\vz)^{1-\beta} \tpi_1(\vz)^{\beta} d\vz$.

The Amari $\alpha$-divergence arises as the Bregman divergence generated by the scaled normalization constant $\frac{1}{q}\Z_q(\beta)$ (see \myapp{parametric_breg} for derivations),
\begin{align}
    \hspace*{-.2cm} D_{\frac{1}{q}\Z_q}&[\beta_a : \beta_b]  = D_{A}^{(q)}[\tpi_{\beta_a}:\tpi_{\beta_b}] . \label{eq:alpha_zexp} 
\end{align}
%As in \cref{example:zhang_div}, this Bregman divergence is recovered from a Bregman Information in the limit as $\beta = 1$.

%Treating the mixing weight $\beta$ as the input to the parametric divergence $D_{\frac{1}{q}\Z_q}$, 
As an alternative to the interpretations in \cref{example:alpha} and \ref{example:alpha_bi},
we may now represent the $q$-path intermediate density $\qpath(\vz)$ using the arithmetic mixture of parameters $\beta$ and the solution to the expected parametric Bregman divergence minimization \citep{Banerjee2005},
%so that the geometric averaging path in \cref{example:geo} corresponds to a simple arithmetic mixture of the endpoints $\bm{u} = \{ 0, 1\}$ with weights $\bm{\beta}=\{1-\beta, \beta \}$
\begin{align}
     \beta &= (1-\beta) \cdot 0 + \beta \cdot 1 = \argmin \limits_{\beta_r} \, (1-\beta) D_{\frac{1}{q} \Z_q}[ 0 : \beta_r] + \beta \,  D_{\frac{1}{q} \Z_q}[ 1 : \beta_r] \, . \label{eq:kl_breg_min_unnorm} 
\end{align}
%As discussed in \myapp{parametric_nonparametric}, this arises f
This further leads to a parametric interpretation of Zhang's ($\alpha$,$\beta$) divergence \cref{eq:qpath_normalizer} as the Bregman Information,
\begin{align}
 D_Z^{(\beta,q)}[\tpi_0 : \tpi_1] &:= 
    \frac{1}{\beta(1-\beta)}\frac{1}{q} \left( \int  (1-\beta) \, \Z_q(0)  + \beta \, \Z_q(1) - 
    \Z_q(\beta) \right).
\end{align}
\end{example}
For arbitrary endpoint densities $\tpi_0, \tpi_1$ satisfying absolutely continuity and integrability conditions, we can thus construct a rich family of divergence functionals using the deformed likelihood ratio exponential family.

\begin{example}[Parametric View of the Geometric Path]\label{example:geo_para}
The Bregman divergence induced by $\Z_1(\beta)$ yields the reverse \kl divergence (see \myapp{parametric_breg}),
\begin{align}
    D_{\Z_1}[\beta_a : \beta_b]&= \DKL[\tpi_{\beta_b}^{(\text{geo})} : \tpi_{\beta_a}^{(\text{geo})}] .
    %=  \int \tpi_{\beta_b}^{(\text{geo})}(\vx) \log \frac{\tpi_{\beta_b}^{(\text{geo})(\vx)}}{\tpi_{\beta_a}^{(\text{geo})}(\vx)} d\vx- \int \tpi_{\beta_b}^{(\text{geo})}(\vx)d\vx + \int \tpi_{\beta_b}^{(\text{geo})}(\vx) d\vx.  \nonumber
    %\int \tpi_{\bthetaprime}(\vx)d\vx - \int \tpi_{\boldtheta}(\vx)d\vx + \int \tpi_{\boldtheta}(\vx) \log \frac{\tpi_{\boldtheta}(\vx)}{\tpi_{\bthetaprime}(\vx)} d\vx
\end{align}
\end{example}
As in \cref{example:q_path_para} above, we represent the geometric averaging path $\tpi_{\beta}^{\text{geo}}$ using the simple arithmetic mixture of endpoints $\bm{u} = \{ 0, 1\}$ with weights $\bm{\beta}=\{1-\beta, \beta \}$.
%we consider the parameter $\beta$ as the input to the Bregman divergence, so that the geometric averaging path in \cref{example:geo} corresponds to a simple arithmetic mixture of the endpoints $\bm{u} = \{ 0, 1\}$ with weights $\bm{\beta}=\{1-\beta, \beta \}$
% \begin{align}
%      \beta &= (1-\beta) \cdot 0 + \beta \cdot 1 = \argmin \limits_{\beta_r} \, (1-\beta) D_{\Z_1}[ 0 : \beta_r] + \beta \,  D_{\Z_1}[ 1 : \beta_r] \, . \label{eq:kl_breg_min_unnorm} 
% \end{align}
At this minimizer, the scaled Bregman Information matches the Amari $\alpha$-divergence $D_A^{(\beta)}[\tpi_0 : \tpi_1]$, as in \cref{eq:amari_breg}
\small
\begin{align}
   \hspace*{-.1cm} 
   D_{f,\rho}^{(\beta)}[\tpi_0:\tpi_1]
   %&= \frac{1}{\beta(1-\beta)} \bigg( (1-\beta) \DKL[ \tpigeo: \tpizero] + \beta \,  \DKL[ \tpigeo : \tpi_1] \bigg) \label{eq:geo_info} \\
   &= \frac{1}{\beta(1-\beta)} \bigg( (1-\beta) \Z_1(0) + \beta \Z_1(1) -  \Z_1(\beta) \bigg) =  D_A^{(\beta)}[\tpi_0 : \tpi_1] \nonumber \\
   &= \frac{1}{\beta(1-\beta)} \bigg( (1-\beta) \int \tpi_0(\vz) d\vz + \beta \int \tpi_1(\vz) d\vz - \int 
    \tpi_0(\vz)^{1-\beta} \tpi_1(\vz)^{\beta}  % \left[ (1-\beta) \tpizero(\vz)^{1-q} + \beta \, \tpi_1(\vz)^{1-q} \right]_+^{\frac{1}{1-q}} 
    d\vz \bigg) . \nonumber
    % \label{eq:amari_breg}
\end{align}
\normalsize
We again contrast the interpretation of the Amari $\alpha$-divergence as a Bregman divergence for the $q$-exponential family in the previous example (corresponding to $\beta = 1$ in \cref{eq:scaled_breg}), and as a Bregman Information with weight $\beta$ for the exponential family and geometric path ($q=1$) in this example.

\begin{example}[Rényi $\alpha$-Divergence as Bregman Information for Geometric Path between Normalized Densities]\label{example:renyi}
For the geometric averaging path, we find that there is no distinction between treating the inputs as normalized probability densities $\bm{\pi}=\{\pi_0, \pi_1\}$ instead of unnormalized $\tpi$ (see \cref{example:normalized} for differences in the case of $q$-paths).  In particular, treating the \textit{log} partition function as the generating function and restricting input to \textit{normalized} densities, we have
\begin{align}
   & D_{\psi}[\beta_a : \beta_b]= \DKL[\pi_{\beta_b}^{(\text{geo})} : \pi_{\beta_a}^{(\text{geo})}] \quad \text{for} \quad \psi(\beta) \coloneqq \log \Z_1(\beta), %\\[1.25ex]
 \end{align}
which leads to the expected divergence minimization over $\beta$
\begin{align}
&\beta = (1-\beta) \cdot 0 + \beta \cdot 1 = \argmin \limits_{\beta_r} \, (1-\beta) \Breg{\norm}[ 0 :\beta_r ] + \beta \, \Breg{\norm}[ 1 :\beta_r ] \label{eq:arbirary_breg} .
\end{align}
We recognize the resulting scaled Bregman Information as the R\'enyi divergence \citep{renyi1961, van2014renyi} of order $\beta$,\footnote{Note, we order the arguments to match the Amari $\alpha$-divergence \citep{zhang2004divergence, amari2007integration}, and use constant factors to induce limiting behavior of $\DKL[\firstargn : \secondargn]$ as $\alpha \rightarrow 0$ and $\DKL[\secondargn:\firstargn]$ as $\alpha \rightarrow 1$.}
\small
\begin{align}
   \hspace*{-.1cm} 
   D_{f,\rho}^{(\beta)}[\tpi_0:\tpi_1]
   %&= \frac{1}{\beta(1-\beta)} \bigg( (1-\beta) \DKL[ \tpigeo: \tpizero] + \beta \,  \DKL[ \tpigeo : \tpi_1] \bigg) \label{eq:geo_info} \\
   = \frac{1}{\beta(1-\beta)} \Big( (1-\beta) \psi(0) + \beta \psi(1) -  \psi(\beta) \Big) &=\frac{ -1}{\beta(1-\beta)}   \log \int \pi_0(\vz)^{1-\beta} \pi_1(\vz)^{\beta} d\vz  . 
    \nonumber \\
& 
 =  D_{R}^{(\beta)} \big[\pi_0(\vz) : \pi_1(\vz) \big] 
\end{align} 
\normalsize
which matches the result in \citet{nielsen2011renyi} that the R\'enyi divergence between normalized distributions in the same exponential family is proportional to a gap in Jensen's inequality.
\end{example}

\begin{example}[$q$-Paths between Normalized Probability Densities]\label{example:normalized}
Quasi-arithmetic means in the $q$- or $\alpha$-representation have been studied extensively \citep{renyi1961, amari2007integration, eguchi2015path, eguchi2016information, wong2021tsallis}, often in the context of normalized probability densities. 
Similarly to our results, it has been shown that the 
%$\alpha$-mixture
quasi-arithmetic mean in the $q$-representation 
minimizes the expected Amari $\alpha$-divergence to a normalized density in the second argument \citep{amari2007integration}.  

However, we highlight that the quasi-arithmetic mean of normalized $\bm{\pi}=\{\pi_0, \pi_1\}$ or unnormalized densities $\bm{\tpi}=\{\tpi_0, \tpi_1\}$ \textit{do not match} in the case of $\rho(u) = \log_q u$. 
For normalized inputs, the quasi-arithmetic mean $\bar{\pi}_{\beta}^{(q)}$ becomes
\begin{align}
\hspace*{-.2cm} \bar{\pi}_{\beta}^{(q)}(\vz) 
%&= \frac{1}{z_q(\beta)} \exp_q \big\{ (1-\beta) \log_q \pi_0(\vz) + \beta \log_q \pi_1(\vz) \big\} \\
&= \frac{1}{\bar{\Z}_q(\beta)} \pi_0(\vz) \exp_q \Big\{\beta \cdot \log_q \frac{\pi_1(\vz)}{\pi_0(\vz)} \Big\} \label{eq:mix_of_norm}\\
&= \frac{1}{c \cdot {\bar{\Z}_q(\beta)}} \tpi_0(\vz) \exp_q \Big\{\frac{\beta \Z(1)^{q-1} }{(1-\beta) \Z(0)^{q-1} + \beta \Z(1)^{q-1}} \log_q \frac{\tpi_1(\vz)}{\tpi_0(\vz)} \Big\}  \label{eq:mix_of_normalized}
\end{align}
where $\Z(0)$ and $\Z(1)$ are the normalization constants for $\tpi_0$ and $\tpi_1$ and $\bar{\Z}_q(\beta)$ normalizes \myeq{mix_of_norm}.  

Note that \myeq{mix_of_norm} contains the $q$-logarithmic ratio of normalized densities, while \myeq{mix_of_normalized} contains the ratio of unnormalized densities.
Moving between the two requires an \textit{additional} factor of $c = [(1-\beta) \Z(0)^{-(1-q)} + \beta \Z(1)^{-(1-q)}]^{\frac{-1}{1-q}}$, which is a quasi-arithmetic mean of $\Z^{-1}$ at the endpoints.
%(see \citet{wong2021tsallis} Sec. 4 for additional discussion).    
Thus, from \myeq{mix_of_normalized}, we observe
that the $q$-mixture of normalized distributions, $\bar{\pi}_{\beta}^{(q)}(\vz)$, does not directly coincide with the $q$-mixture of unnormalized endpoints, $\pi_{\beta}^{(q)}(\vz) \propto \tpi_{\beta}^{(q)}(\vz).$     Instead, we need to 
adjust the mixing 
weight $\beta^{\prime} = \frac{\beta \Z(1)^{q-1}}{(1-\beta) \Z(0)^{q-1}+\beta \Z(1)^{q-1} }$ to obtain $\bar{\pi}_{\beta}^{(q)}(\vz) = \pi_{\beta^{\prime}}^{(q)}(\vz) \propto \tpi_{\beta^{\prime}}^{(q)}(\vz)$.  

A similar time-reparameterization appears in \citet{wong2021tsallis}, which interprets the normalized $q$-exponential family in terms of a deformation of the standard convex duality.
This reparameterization of the mixing parameter $\beta$ is usually intractable to calculate explicitly due to need to calculate normalization constants $\Z(0)$, $\Z(1)$.
%\citet{chehab2023provable}
By contrast, for the geometric path ($q=1$) in \cref{example:geo} or \cref{example:geo_para}-\ref{example:renyi}, the mixing parameter is the same for annealing between unnormalized and normalized endpoint densities.
% \end{remark}
\end{example}

\begin{example}[Moment-Averaging Path of \citet{grosse2013annealing}]
Finally, consider the special case of constructing an annealing path between two (normalized) endpoints $\pi_{\boldtheta_0}$ and $\pi_{\boldtheta_1}$ within an exponential family with sufficient statistics $\bt = \{ T^i(\vx) \}_{i=1}^d$.    In this case, we can recover the moment-averaging path of \citet{grosse2013annealing} as a quasi-arithmetic mean $ \btheta_\beta = \boldeta^{-1}((1-\beta) \, \boldeta(\btheta_0) + \beta  \, \boldeta(\btheta_1)) $  using the expectation parameter mapping  $\rho(\btheta) = \boldeta(\btheta) = \mathbb{E}_{\expfam}\left[ \bt \right]$.   See \myapp{grosse} for details.
\end{example}

\begin{example}[Annealing within (Deformed) Exponential Families]\label{example:annealing_main}
While we constructed one-dimensional deformed exponential families from arbitrary endpoint densities in \cref{eq:q_lref} and \cref{example:q_path_para}-\ref{example:geo_para},  the $\rho$-affine property of deformed exponential families yields a similar simplification for the special case of endpoints $\tpi_{\boldtheta_0}^{(q)}, \tpi_{\boldtheta_1}^{(q)}$ which belong to the same parametric family.
In particular, for the $\rho(\tpi) = \log_q \tpi$ path within the $\exp_q$ family, we have
\begin{align}
    \btheta_\beta &= (1-\beta) \, \btheta_0 + \beta  \, \btheta_1  = \argmin \limits_{\btheta_r} ~ (1-\beta) D_{\frac{1}{q}\mathcal{Z}_q}[\btheta_0 : \btheta_r] + \beta ~ D_{\frac{1}{q}\mathcal{Z}_q}[\btheta_1 : \btheta_r], \nonumber
\end{align}
where $D_{\frac{1}{q}\mathcal{Z}_q} = D_{A}^{(q)}$ is the Amari $\alpha$-divergence as in \cref{eq:alpha_zexp} and the \kl divergence or exponential family case is recovered for $q=1$.  See \myapp{parametric_breg} and \cref{example:annealing} for detailed discussion.
\end{example}

%\vheader
%\vheader
\section{Conclusion and Discussion}
%\vheader
%%%%%
%\RB{Revisit}
In this work, we have generalized the Bregman divergence barycenter results of \citet{Banerjee2005} to arbitrary monotonic embedding functions and quasi-arithmetic means.   We identified annealing paths from the MCMC literature \citep{neal2001annealed, grosse2013annealing, masrani2021q} as a natural setting where such quasi-arithmetic means appear.  In particular, for two unnormalized density functions as input, we related the rho-tau Bregman divergence framework \citep{zhang2004divergence,zhang2013nonparametric,naudts2018rho} to the Bregman Information from \citep{Banerjee2005}, and highlighted how various divergence functionals comparing $\tpi_0$ and $\tpi_1$ are associated with intermediate densities $\tpi_\beta$ along an annealing path.

We have seen that Amari's $\alpha$-divergence arises via two different approaches.   For the geometric averaging path ($q=1$), the $\alpha$-divergence appears as a Bregman Information with its order set by the mixture parameter $\beta$ (\cref{example:geo}).   For the $q$-path, the $\alpha$-divergence appears as a rho-tau Bregman divergence with order set by the deformation parameter $q$ (\cref{example:alpha}).      
In both cases, we provided parametric interpretations 
involving the Bregman divergence generated by the multiplicative normalization constant $\mathcal{Z}_q(\beta)$ of a one-dimensional (deformed) likelihood ratio exponential family.

However, for the geometric path, we also constructed a Bregman divergence using the \textit{log} partition function of the one-dimensional exponential family $\log Z(\beta)$, where the corresponding scaled Bregman Information recovers the R\'enyi divergence as the gap in Jensen's inequality for $\psi(\beta)$ \citep{nielsen2011renyi}.   This derivation in terms of the mixing parameter $\alpha = \beta$ is distinct from the approach of \citet{wong2021tsallis}, where the 
R\'enyi divergence arises using the deformation parameter $\alpha= q$, the potential function $\log \mathcal{Z}_q(\beta)$, and deformed $c$-duality.
Further understanding the relationship between these constructions in terms of the $\beta$ and $q$ parameters remains an interesting question for future work.   

Moving beyond the ubiquitous use of the \kl divergence in machine learning, it would be interesting to further explore the 
use of rho-tau divergences 
% implications of the identity dual representation for the $\beta$-divergence 
in applications such as variational inference \citep{knoblauch2019generalized}, constructing prediction losses \citep{blondel2020fenchelyoung, amid2022layerwise}, 
%continuous attention mechanisms \citep{martins2021sparse}, 
and regularized reinforcement learning \citep{geist2019theory}.   Clustering approaches based on quasi-arithmetic means have been proposed in \citep{xu2019power, vellal2022bregman}, and our insights might be used to develop probabilistic interpretations or algorithms similar to the original clustering motivations of Bregman Information \citep{Banerjee2005}.
Finally, future work might consider paths based on $\phi$-deformed logarithms \citep{naudts2011generalised, naudts2018rho}, or investigate ways to adaptively choose or learn a suitable path  \citep{syed2021parallel} or annealing schedule \citep{goshtasbpour2023adaptive} based on statistics of a given sampling problem. 

\clearpage
\small
\vheader
\bibliographystyle{spbasic}      
\bibliography{_ref} 

\begin{thebibliography}{80}
\providecommand{\natexlab}[1]{#1}
\providecommand{\url}[1]{{#1}}
\providecommand{\urlprefix}{URL }
\expandafter\ifx\csname urlstyle\endcsname\relax
  \providecommand{\doi}[1]{DOI~\discretionary{}{}{}#1}\else
  \providecommand{\doi}{DOI~\discretionary{}{}{}\begingroup
  \urlstyle{rm}\Url}\fi
\providecommand{\eprint}[2][]{\url{#2}}

\bibitem[{Adlam et~al.(2022)Adlam, Gupta, Mariet, and
  Smith}]{adlam2022understanding}
Adlam B, Gupta N, Mariet Z, Smith J (2022) Understanding the bias-variance
  tradeoff of bregman divergences. arXiv preprint arXiv:220204167

\bibitem[{Alemi et~al.(2018)Alemi, Poole, Fischer, Dillon, Saurous, and
  Murphy}]{alemi2018fixing}
Alemi A, Poole B, Fischer I, Dillon J, Saurous RA, Murphy K (2018) {Fixing a
  Broken ELBO}. In: International Conference on Machine Learning, pp 159--168

\bibitem[{Amari(1982)}]{amari1982differential}
Amari Si (1982) {Differential geometry of curved exponential
  families-curvatures and information loss}. The Annals of Statistics pp
  357--385

\bibitem[{Amari(2007)}]{amari2007integration}
Amari Si (2007) {Integration of stochastic models by minimizing
  $\alpha$-divergence}. Neural computation 19(10):2780--2796

\bibitem[{Amari(2016)}]{amari2016information}
Amari Si (2016) Information geometry and its applications, vol 194. Springer

\bibitem[{Amari and Nagaoka(2000)}]{amari2000methods}
Amari Si, Nagaoka H (2000) {Methods of information geometry}, vol 191. American
  Mathematical Soc.

\bibitem[{Amid et~al.(2022)Amid, Anil, Fifty, and Warmuth}]{amid2022layerwise}
Amid E, Anil R, Fifty C, Warmuth MK (2022) {Layerwise Bregman Representation
  Learning of Neural Networks with Applications to Knowledge Distillation}.
  Transactions on Machine Learning Research

\bibitem[{Ay et~al.(2017)Ay, Jost, V{\^a}n~L{\^e}, and
  Schwachh{\"o}fer}]{ay2017information}
Ay N, Jost J, V{\^a}n~L{\^e} H, Schwachh{\"o}fer L (2017) {Information
  geometry}, vol~64. Springer

\bibitem[{Banerjee et~al.(2005{\natexlab{a}})Banerjee, Guo, and
  Wang}]{banerjee2005optimality}
Banerjee A, Guo X, Wang H (2005{\natexlab{a}}) {{On the optimality of
  conditional expectation as a Bregman predictor}}. IEEE Transactions on
  Information Theory 51(7):2664--2669

\bibitem[{Banerjee et~al.(2005{\natexlab{b}})Banerjee, Merugu, Dhillon, and
  Ghosh}]{Banerjee2005}
Banerjee A, Merugu S, Dhillon IS, Ghosh J (2005{\natexlab{b}}) {Clustering with
  Bregman Divergences}. Journal of Machine Learning Research 6:1705--1749

\bibitem[{Basu et~al.(1998)Basu, Harris, Hjort, and Jones}]{basu1998robust}
Basu A, Harris IR, Hjort NL, Jones M (1998) {Robust and efficient estimation by
  minimising a density power divergence}. Biometrika 85(3):549--559

\bibitem[{Bercher(2012)}]{bercher2012simple}
Bercher JF (2012) {A simple probabilistic construction yielding generalized
  entropies and divergences, escort distributions and $q$-Gaussians}. Physica
  A: Statistical Mechanics and its Applications 391(19):4460--4469

\bibitem[{Betancourt et~al.(2017)Betancourt, Byrne, Livingstone, Girolami
  et~al.}]{betancourt2017geometric}
Betancourt M, Byrne S, Livingstone S, Girolami M, et~al. (2017) {Geometric
  foundations of {H}amiltonian {M}onte {C}arlo}. Bernoulli 23(4A):2257--2298

\bibitem[{Blondel et~al.(2020)Blondel, Martins, and
  Niculae}]{blondel2020fenchelyoung}
Blondel M, Martins AF, Niculae V (2020) {{Learning with Fenchel-Young losses}}.
  J Mach Learn Res 21(35):1--69

\bibitem[{Brekelmans et~al.(2020{\natexlab{a}})Brekelmans, Masrani, Wood,
  Ver~Steeg, and Galstyan}]{brekelmans2020tvo}
Brekelmans R, Masrani V, Wood F, Ver~Steeg G, Galstyan A (2020{\natexlab{a}})
  All in the exponential family: Bregman duality in thermodynamic variational
  inference. In: Proceedings of the 37th International Conference on Machine
  Learning, JMLR.org, ICML'20

\bibitem[{Brekelmans et~al.(2020{\natexlab{b}})Brekelmans, Nielsen, Galstyan,
  and Steeg}]{brekelmans2020lkd}
Brekelmans R, Nielsen F, Galstyan A, Steeg GV (2020{\natexlab{b}}) {Likelihood
  Ratio Exponential Families}. In: NeurIPS Workshop on Information Geometry in
  Deep Learning, \urlprefix\url{https://openreview.net/forum?id=RoTADibt26_}

\bibitem[{Brekelmans et~al.(2022)Brekelmans, Huang, Ghassemi, Steeg, Grosse,
  and Makhzani}]{brekelmans2022improving}
Brekelmans R, Huang S, Ghassemi M, Steeg GV, Grosse RB, Makhzani A (2022)
  {Improving Mutual Information Estimation with Annealed and Energy-Based
  Bounds}. In: International Conference on Learning Representations

\bibitem[{Burbea and Rao(1982)}]{burbearao1982}
Burbea J, Rao C (1982) Entropy differential metric, distance and divergence
  measures in probability spaces: A unified approach. Journal of Multivariate
  Analysis 12(4):575--596, \doi{https://doi.org/10.1016/0047-259X(82)90065-3}

\bibitem[{Chatterjee and Diaconis(2018)}]{chatterjee2018sample}
Chatterjee S, Diaconis P (2018) {The sample size required in importance
  sampling}. The Annals of Applied Probability 28(2):1099--1135

\bibitem[{Cichocki and Amari(2010)}]{cichocki2010families}
Cichocki A, Amari Si (2010) {Families of alpha-beta-and gamma-divergences:
  Flexible and robust measures of similarities}. Entropy 12(6):1532--1568

\bibitem[{Cichocki et~al.(2011)Cichocki, Cruces, and
  Amari}]{cichocki2011generalized}
Cichocki A, Cruces S, Amari Si (2011) {Generalized alpha-beta divergences and
  their application to robust nonnegative matrix factorization}. Entropy
  13(1):134--170

\bibitem[{Csisz\'ar(1991)}]{csiszar1991least}
Csisz\'ar I (1991) {Why least squares and maximum entropy? An axiomatic
  approach to inference for linear inverse problems}. The annals of statistics
  19(4):2032--2066

\bibitem[{Del~Moral et~al.(2006)Del~Moral, Doucet, and
  Jasra}]{del2006sequential}
Del~Moral P, Doucet A, Jasra A (2006) {{Sequential monte carlo samplers}}.
  Journal of the Royal Statistical Society: Series B (Statistical Methodology)
  68(3):411--436

\bibitem[{Duane et~al.(1987)Duane, Kennedy, Pendleton, and
  Roweth}]{duane1987hybrid}
Duane S, Kennedy AD, Pendleton BJ, Roweth D (1987) {{Hybrid Monte Carlo}}.
  Physics letters B 195(2):216--222

\bibitem[{Earl and Deem(2005)}]{earl2005parallel}
Earl DJ, Deem MW (2005) {Parallel tempering: Theory, applications, and new
  perspectives}. Physical Chemistry Chemical Physics 7(23):3910--3916

\bibitem[{Eguchi(1983)}]{eguchi1983second}
Eguchi S (1983) {Second order efficiency of minimum contrast estimators in a
  curved exponential family}. The Annals of Statistics pp 793--803

\bibitem[{Eguchi(1985)}]{eguchi1985differential}
Eguchi S (1985) {A differential geometric approach to statistical inference on
  the basis of contrast functionals}. Hiroshima mathematical journal
  15(2):341--391

\bibitem[{Eguchi(2006)}]{eguchi2006information}
Eguchi S (2006) {Information geometry and statistical pattern recognition}.
  Sugaku Expositions 19(2):197--216

\bibitem[{Eguchi and Komori(2015)}]{eguchi2015path}
Eguchi S, Komori O (2015) {Path connectedness on a space of probability density
  functions}. In: International Conference on Geometric Science of Information,
  pp 615--624

\bibitem[{Eguchi et~al.(2016)Eguchi, Komori, and Ohara}]{eguchi2016information}
Eguchi S, Komori O, Ohara A (2016) {Information geometry associated with
  generalized means}. In: Information Geometry and its Applications IV,
  Springer, pp 279--295

\bibitem[{Frigyik et~al.(2008)Frigyik, Srivastava, and
  Gupta}]{frigyik2008functional}
Frigyik BA, Srivastava S, Gupta MR (2008) Functional bregman divergence and
  bayesian estimation of distributions. IEEE Transactions on Information Theory
  54(11):5130--5139

\bibitem[{Geist et~al.(2019)Geist, Scherrer, and Pietquin}]{geist2019theory}
Geist M, Scherrer B, Pietquin O (2019) {{A theory of regularized Markov
  decision processes}}. In: International Conference on Machine Learning, PMLR,
  pp 2160--2169

\bibitem[{Gelman and Meng(1998)}]{gelman1998simulating}
Gelman A, Meng XL (1998) {Simulating normalizing constants: From importance
  sampling to bridge sampling to path sampling}. Statistical science pp
  163--185

\bibitem[{Goshtasbpour et~al.(2023)Goshtasbpour, Cohen, and
  Perez-Cruz}]{goshtasbpour2023adaptive}
Goshtasbpour S, Cohen V, Perez-Cruz F (2023) Adaptive annealed importance
  sampling with constant rate progress. International Conference on Machine
  Learning

\bibitem[{Grasselli(2010)}]{grasselli2010dual}
Grasselli MR (2010) {Dual connections in nonparametric classical information
  geometry}. Annals of the Institute of Statistical Mathematics 62(5):873--896

\bibitem[{Grosse et~al.(2013)Grosse, Maddison, and
  Salakhutdinov}]{grosse2013annealing}
Grosse RB, Maddison CJ, Salakhutdinov RR (2013) {Annealing between
  distributions by averaging moments}. In: Advances in Neural Information
  Processing Systems, pp 2769--2777

\bibitem[{Hardy et~al.(1953)Hardy, Littlewood, and Pólya}]{hardy1953}
Hardy G, Littlewood J, Pólya G (1953) Inequalities. The Mathematical Gazette
  37(321):236–236

\bibitem[{Jarzynski(1997{\natexlab{a}})}]{jarzynski1997equilibrium}
Jarzynski C (1997{\natexlab{a}}) {Equilibrium free-energy differences from
  nonequilibrium measurements: A master-equation approach}. Physical Review E
  56(5):5018

\bibitem[{Jarzynski(1997{\natexlab{b}})}]{jarzynski1997equality}
Jarzynski C (1997{\natexlab{b}}) {Nonequilibrium equality for free energy
  differences}. Physical Review Letters 78(14):2690

\bibitem[{Jaynes(1957)}]{jaynes1957information}
Jaynes ET (1957) {Information theory and statistical mechanics}. Physical
  review 106(4):620

\bibitem[{Kaniadakis and Scarfone(2002)}]{kaniadakis2002new}
Kaniadakis G, Scarfone A (2002) {A new one-parameter deformation of the
  exponential function}. Physica A: Statistical Mechanics and its Applications
  305(1-2):69--75

\bibitem[{Knoblauch et~al.(2019)Knoblauch, Jewson, and
  Damoulas}]{knoblauch2019generalized}
Knoblauch J, Jewson J, Damoulas T (2019) {Generalized variational inference:
  Three arguments for deriving new posteriors}. arXiv preprint arXiv:190402063

\bibitem[{Kolmogorov(1930)}]{kolmogorov1930}
Kolmogorov AN (1930) {Sur la notion de la moyenne}. G. Bardi, tip. della R.
  Accad. dei Lincei

\bibitem[{Lin(1991)}]{lin1991divergence}
Lin J (1991) {{Divergence measures based on the Shannon entropy}}. IEEE
  Transactions on Information theory 37(1):145--151

\bibitem[{Loaiza and Quiceno(2013{\natexlab{a}})}]{loaiza2013q}
Loaiza GI, Quiceno H (2013{\natexlab{a}}) {A $q$-exponential statistical Banach
  manifold}. Journal of Mathematical Analysis and Applications 398(2):466--476

\bibitem[{Loaiza and Quiceno(2013{\natexlab{b}})}]{loaiza2013riemannian}
Loaiza GI, Quiceno HR (2013{\natexlab{b}}) {A Riemannian geometry in the
  $q$-Exponential Banach manifold induced by $q$-Divergences}. In: Geometric
  Science of Information. First International Conference, GSI 2013, Paris,
  France, August 28-30, 2013. Proceedings, pp 737-742, Springer Berlin
  Heidelberg

\bibitem[{Masrani et~al.(2021)Masrani, Brekelmans, Bui, Nielsen, Galstyan,
  Steeg, and Wood}]{masrani2021q}
Masrani V, Brekelmans R, Bui T, Nielsen F, Galstyan A, Steeg GV, Wood F (2021)
  {q-Paths: Generalizing the Geometric Annealing Path using Power Means}.
  Uncertainty in Artificial Intelligence

\bibitem[{Murata et~al.(2004)Murata, Takenouchi, Kanamori, and
  Eguchi}]{murata2004information}
Murata N, Takenouchi T, Kanamori T, Eguchi S (2004) {Information geometry of
  U-Boost and Bregman divergence}. Neural Computation 16(7):1437--1481

\bibitem[{Naudts(2004)}]{naudts2004estimators}
Naudts J (2004) {Estimators, escort probabilities, and phi-exponential families
  in statistical physics}. arXiv preprint math-ph/0402005

\bibitem[{Naudts(2011)}]{naudts2011generalised}
Naudts J (2011) {Generalised thermostatistics}. Springer Science \& Business
  Media

\bibitem[{Naudts and Zhang(2018)}]{naudts2018rho}
Naudts J, Zhang J (2018) {Rho--tau embedding and gauge freedom in information
  geometry}. Information geometry 1(1):79--115

\bibitem[{Neal(2001)}]{neal2001annealed}
Neal RM (2001) {Annealed importance sampling}. Statistics $\&$ computing
  11(2):125--139

\bibitem[{Neal(2011)}]{neal2011mcmc}
Neal RM (2011) {{MCMC} Using {H}amiltonian Dynamics}. {Handbook of Markov Chain
  Monte Carlo} p 113

\bibitem[{Nguyen et~al.(2010)Nguyen, Wainwright, and
  Jordan}]{nguyen2010estimating}
Nguyen X, Wainwright MJ, Jordan MI (2010) {Estimating divergence functionals
  and the likelihood ratio by convex risk minimization}. IEEE Transactions on
  Information Theory 56(11):5847--5861

\bibitem[{Nielsen(2020)}]{nielsen2018elementary}
Nielsen F (2020) {An elementary introduction to information geometry}. Entropy
  22(10)

\bibitem[{Nielsen and Boltz(2011)}]{nielsen2011burbea}
Nielsen F, Boltz S (2011) {The burbea-rao and bhattacharyya centroids}. IEEE
  Transactions on Information Theory 57(8):5455--5466

\bibitem[{Nielsen and Nock(2011)}]{nielsen2011renyi}
Nielsen F, Nock R (2011) {On R\'enyi and Tsallis entropies and divergences for
  exponential families}. arXiv preprint arXiv:11053259

\bibitem[{Nock and Nielsen(2005)}]{nock2005fitting}
Nock R, Nielsen F (2005) Fitting the smallest enclosing bregman ball. In:
  European Conference on Machine Learning, Springer, pp 649--656

\bibitem[{Nock et~al.(2017)Nock, Cranko, Menon, Qu, and
  Williamson}]{nock2017fgan}
Nock R, Cranko Z, Menon AK, Qu L, Williamson RC (2017) {{$f$-GANs in an
  information geometric nutshell}}. Advances in Neural Information Processing
  Systems

\bibitem[{Nowozin et~al.(2016)Nowozin, Cseke, and Tomioka}]{nowozin2016f}
Nowozin S, Cseke B, Tomioka R (2016) {{$f$-GAN: Training generative neural
  samplers using variational divergence minimization}}. Neural Information
  Processing Systems 29

\bibitem[{Ogata(1989)}]{ogata1989monte}
Ogata Y (1989) {{A Monte Carlo method for high dimensional integration}}.
  Numerische Mathematik 55(2):137--157

\bibitem[{Pfau(2013)}]{pfau2013generalized}
Pfau D (2013) {A generalized bias-variance decomposition for Bregman
  divergences}

\bibitem[{Pistone and Sempi(1995)}]{pistone1995infinite}
Pistone G, Sempi C (1995) {An infinite-dimensional geometric structure on the
  space of all the probability measures equivalent to a given one}. The annals
  of statistics

\bibitem[{Poole et~al.(2019)Poole, Ozair, Van Den~Oord, Alemi, and
  Tucker}]{poole2019variational}
Poole B, Ozair S, Van Den~Oord A, Alemi A, Tucker G (2019) {On Variational
  Bounds of Mutual Information}. In: International Conference on Machine
  Learning, pp 5171--5180

\bibitem[{Rossky et~al.(1978)Rossky, Doll, and Friedman}]{rossky1978brownian}
Rossky PJ, Doll J, Friedman H (1978) {{Brownian dynamics as smart Monte Carlo
  simulation}}. The Journal of Chemical Physics 69(10):4628--4633

\bibitem[{Rényi(1961)}]{renyi1961}
Rényi A (1961) {On Measures of Entropy and Information}. In: Proceedings of
  the Fourth Berkeley Symposium on Mathematical Statistics and Probability,
  Berkeley, Calif., pp 547--561,
  \urlprefix\url{https://projecteuclid.org/euclid.bsmsp/1200512181}

\bibitem[{Sibson(1969)}]{sibson1969information}
Sibson R (1969) {Information radius}. Zeitschrift f{\"u}r
  Wahrscheinlichkeitstheorie und verwandte Gebiete 14(2):149--160

\bibitem[{Syed et~al.(2021)Syed, Romaniello, Campbell, and
  Bouchard-C{\^o}t{\'e}}]{syed2021parallel}
Syed S, Romaniello V, Campbell T, Bouchard-C{\^o}t{\'e} A (2021) {Parallel
  Tempering on Optimized Paths}. International Conference on Machine Learning

\bibitem[{Tishby et~al.(1999)Tishby, Pereira, and
  Bialek}]{tishby1999information}
Tishby N, Pereira FC, Bialek W (1999) {The information bottleneck method}. In:
  Allerton Conference on Communications, Control and Computing, pp 368--377

\bibitem[{Tsallis(1988)}]{tsallis1988possible}
Tsallis C (1988) {Possible generalization of {B}oltzmann-{G}ibbs statistics}.
  Journal of statistical physics 52(1-2):479--487

\bibitem[{Tsallis(2009)}]{tsallis2009introduction}
Tsallis C (2009) {Introduction to nonextensive statistical mechanics:
  approaching a complex world}. Springer Science \& Business Media

\bibitem[{Van~Erven and Harremos(2014)}]{van2014renyi}
Van~Erven T, Harremos P (2014) {R\'enyi divergence and Kullback-Leibler
  divergence}. IEEE Transactions on Information Theory 60(7):3797--3820

\bibitem[{Vellal et~al.(2022)Vellal, Chakraborty, and Xu}]{vellal2022bregman}
Vellal A, Chakraborty S, Xu JQ (2022) Bregman power k-means for clustering
  exponential family data. In: International Conference on Machine Learning,
  PMLR, pp 22103--22119

\bibitem[{Welling and Teh(2011)}]{welling2011bayesian}
Welling M, Teh YW (2011) {Bayesian learning via stochastic gradient Langevin
  dynamics}. In: Proceedings of the 28th international conference on machine
  learning (ICML-11), Citeseer, pp 681--688

\bibitem[{Wong and Zhang(2021)}]{wong2021tsallis}
Wong TKL, Zhang J (2021) {{Tsallis and R\'enyi deformations linked via a new
  $\lambda$-duality}}. arXiv preprint arXiv:210711925

\bibitem[{Xu and Lange(2019)}]{xu2019power}
Xu J, Lange K (2019) Power k-means clustering. In: International conference on
  machine learning, PMLR

\bibitem[{Zhang(2004)}]{zhang2004divergence}
Zhang J (2004) {Divergence function, duality, and convex analysis}. Neural
  computation 16(1):159--195

\bibitem[{Zhang(2013)}]{zhang2013nonparametric}
Zhang J (2013) {Nonparametric information geometry: From divergence function to
  referential-representational biduality on statistical manifolds}. Entropy
  15(12):5384--5418

\bibitem[{Zhang(2015)}]{zhang2015monotone}
Zhang J (2015) {On monotone embedding in information geometry}. Entropy 17(7)

\bibitem[{Zhang and Matsuzoe(2021)}]{zhang2021entropy}
Zhang J, Matsuzoe H (2021) {Entropy, cross-entropy, relative entropy:
  Deformation theory (a)}. Europhysics Letters 134(1):18001

\end{thebibliography}
\normalsize 

\newpage
\appendix
\normalsize
\textbf{Summary of Appendix}
In \myapp{ais_intro}, we review annealed importance sampling as an example \gls{MCMC} technique.  In \myapp{thm_pf}, we prove our main result (\cref{thm:breg_info}).   We discuss parameteric Bregman divergences and annealing paths between deformed exponential families from our representational perspective in \myapp{parametric_breg}, where the $\rho$-affine property plays a crucial role.  
%We interpret the family of $f$-divergences in our framework in \myapp{fdiv}, which highlights a relationship between the Beta and $\alpha$-divergences.
In \myapp{info_geo}, we review the Eguchi relations and information-geometric structures induced by the rho-tau Bregman divergence and rho-tau Bregman Information functionals (see \cref{tab:connections}).   Finally, in \myapp{qgeo}, we prove \cref{thm:geodesic} showing that quasi-arithmetic paths in the $\rho$-representation are geodesics with respect to affine connections induced by the rho-tau Bregman divergence.

\normalsize
\vheader 
\vheader
\section{Annealed Importance Sampling}\label{app:ais_intro}
We briefly present annealed importance sampling (\gls{AIS}) \citep{neal2001annealed} as a representative example of an \gls{MCMC} method where the choice of annealing path can play a crucial role  \citep{grosse2013annealing, masrani2021q}.  \gls{AIS} relies on similar insights as the Jarzynski equality in nonequilibrium thermodynamics \citep{jarzynski1997equilibrium,jarzynski1997equality}, and may be used to estimate (log) normalization or partition functions or sample from complex distributions.  
% Remarkably, \gls{AIS} allows us to estimate `equilibrium' properties such as partition functions from `nonequilibrium' trajectories or Markov transition kernels which do not sample the stationary distribution.

More concretely, consider an initial distribution $\pi_0(\vz) \propto \tpi_0(\vz)$ which is tractable to sample and is often chosen to have normalization constant $\mathcal{Z}_0= 1$. 
We are often interested in estimating the normalizing constant $\mathcal{Z}_1 = \int \tpi_1(\vz)$ of a target distribution $\pi_1(\vz) \propto \tpi_1(\vz)$, where only the unnormalized density $\tpi_1$ is available.
% We denote the target distribution as $\pi_1(\vz) \propto \tpi_1(\vz)$ with unnormalized density $\tpi_1$, and we are often interested in estimating the normalizing constant $\mathcal{Z}_1 = \int \tpi_1(\vz)$.   
Since direct sampling from $\pi_0(\vz)$ may require prohibitive sample complexity to accurately estimate the normalization constant ratio $\mathcal{Z}_1/\mathcal{Z}_0$  \citep{chatterjee2018sample, brekelmans2022improving}, \gls{AIS} decomposes the estimation problem into a sequence of easier subproblems using a \textit{path} of intermediate distributions $\{\tpi_{\beta_t}(\vz)\}_{\beta_0=0}^{\beta_T=1}$ between the endpoints $\tpi_0(\vz)$ and $\tpi_1(\vz)$.   Most often, the geometric averaging path is used,
\begin{align}
     \tpi_{\beta_t} (\vz) = \frac{{\tpi}_0(\vz)^{1-\beta_t} \, {\tpi}_1( \vz)^{\beta_t}}{\mathcal{Z}({\beta_t})}  \quad \text{where} \quad 
     \mathcal{Z}({\beta_t}) %= \int \tpi_{\beta_t}(\vz) d\vz 
     = \int {\tpi}_0(\vz)^{1-\beta_t} \, {\tpi}_1(\vz)^{\beta_t} d\vz .  \label{eq:def_geo}
\end{align}
\gls{AIS} proceeds by constructing a sequence of Markov transition kernels $\tfwd(\vz_{t+1}|\vz_{t})$ which leave $\pi_{\beta_t}$ invariant, with $\int \pi_{\beta_t}(\vz_t) \tfwd(\vz_{t+1}|\vz_{t}) d\vz_t= \pi_{\beta_t} (\vz_{t+1})$.  Commonly, this is achieved using kernels such as \gls{HMC} or Langevin dynamics \citep{neal2011mcmc} which transform the samples, with Metropolis-Hastings accept-reject steps to ensure invariance.   
To interpret \gls{AIS} as importance sampling in an extended state space \citep{neal2001annealed, brekelmans2022improving},
we define the reverse kernel as $\trev(\vz_{t} | \vz_{t+1}) = \frac{\pi_{{\beta_t}}(\vz_t)\tfwd(\vz_{t+1} | \vz_{t}) }{\int  \pi_{{\beta_t}}(\vz_t) \tfwd(\vz_{t+1} | \vz_{t})d\vz}$. %{\pi_{t}(\vz_{t+1})}
Using the invariance of $\tfwd$, we observe that $\pi_{{\beta_t}}(\vz_{t+1}) \trev(\vz_{t} | \vz_{t+1})  = \pi_{{\beta_t}}(\vz_t) \tfwd(\vz_{t+1} | \vz_{t})$.
% $ \pi_{{\beta_t}}(\vz_t) \tfwd(\vz_{t+1} | \vz_{t}) =  \pi_{{\beta_t}}(\vz_{t+1}) \trev(\vz_{t} | \vz_{t+1})$.

To construct an estimator of $\mathcal{Z}_T/\mathcal{Z}_0$ using \gls{AIS}, we sample from $\vz_0 \sim \pi_0(\vz)$, run the transition kernels in the forward direction to obtain samples $\vz_{1:T}$, and calculate the importance sampling weights along the path,
\begin{equation}
    \resizebox{.93\textwidth}{!}{$
    w(\vz_{0:T}) = \frac{\tpi_1(\vz_T) \prod \limits_{t=0}^{T-1} \trev(\vz_{t} | \vz_{t+1})}{\tpi_0(\vz_0) \prod \limits_{t=0}^{T-1} \tfwd(\vz_{t+1} | \vz_{t})} = \prod \limits_{t=1}^{T}  \dfrac{\tpi_{\beta_t}(\vz_{t})}{\tpi_{\beta_{t-1}}(\vz_{t})} = \prod \limits_{t=1}^{T} \bigg( \dfrac{\pi_T( \vz_{t})}{\pi_0(\vz_{t})} \bigg)^{\beta_{t}-\beta_{t-1}} . 
    $}
\end{equation}
Note that we have used the above identity relating $\tfwd$ and $\trev$ in the second equality, and the definition of the geometric averaging path in the last equality.   

Finally, it can be shown that $w(\vz_{0:T})$ provides an unbiased estimator of $\mathcal{Z}_1/\mathcal{Z}_0$, with $\mathbb{E}[w(\vz_{0:T})] = \mathcal{Z}_1/\mathcal{Z}_0$ \citep{neal2001annealed}.   We can thus estimate the partition function ratio using the empirical average over $K$ annealing chains,  $\mathcal{Z}_1/\mathcal{Z}_0 \approx \frac{1}{K} \sum w_{0:T}^{(k)}$.  We detail the complete \gls{AIS} procedure in \myalg{ais}.

\gls{AIS} is considered among the gold standard methods for estimating normalization constants.
%, and the variance of the importance weights can be shown to vanish as $T \rightarrow \infty$. \textcolor{red}{Cite /refine}  
Closely related \gls{MCMC} methods involving path sampling \citep{gelman1998simulating} include Sequential Monte Carlo \citep{del2006sequential}, which may involve resampling steps to prioritize higher-probability $\vz_t$, or parallel tempering \citep{earl2005parallel}, which runs $T$ parallel sampling chains in order to obtain accurate samples from each $\pi_{\beta_t}(\vz)$.

\vspace*{-.25cm} 
\vheader
\vheader
\section{Proof of Theorem 1}\label{app:thm_pf}
In the main text and below, we present and prove \mythm{breg_info} in terms of scalar inputs and decomposable Bregman divergences.   As we show in \myapp{thm1_vector} \mythm{breg_info_vector}, a similar Bregman divergence-minimization interpretation of quasi-arithmetic means holds for vector-valued inputs, where the representation function is applied element-wise.   Most commonly, vectorized divergences are constructed between parameter vectors $\btheta$ of some (deformed) exponential family.   However, we argue in \myapp{parametric_nonparametric} that these cases are best understood using representations of unnormalized densities (as in \mythm{breg_info}) and the $\rho$-affine property of parametric families.   Nevertheless, we provide proof of \mythm{breg_info_vector}  for completeness.

\rtbreginfo*
\begin{proof}
\textbf{(ii):} We first show the optimal representative $\mu^{*}
%\coloneqq \rho \of{\mu_{\rho}(\vx;\vupi, \vwbeta)} 
= \rho^{-1}\big($ $\sum_{i=1}^N \beta_i \, \rho(\tpi_i)\big)$ yields a Jensen diversity in $(ii)$, before proving this choice is the unique minimizing argument.  

Expanding the expected divergence in \cref{eq:breginfo_optimization} for $\mu = {\mu}^*$, we have
$\sum_{i=1}^N  \beta_i \rhobreg [ \rho(\tpi_i) : \rho \of{{\mu}^*}  ] = \sum_{i=1}^N \beta_i \potentialf[\rho(\tpi_i)] - \potentialf[\rho({\mu}^*)]$
$- \int \big( \sum_{i=1}^N  \beta_i \rho(\tpi_i) - \rho({\mu}^*)
\big) \tau({\mu}^*) d\vz $.   
Since $\sum_i \beta_i \rho(\tpi_i) = \rho({\mu}^*)$, the final term cancels to yield
\small 
\begin{align}
\sum \limits_{i=1}^N  \beta_i \rhobreg\left[ \rho(\tpi_i) : \rho \of{{\mu}^*} \right] 
&= \sum \limits_{i=1}^N \beta_i \potentialf[\rho(\tpi_i)] - \potentialf[\rho({\mu}^*)]. \label{eq:jg}
\end{align}
\normalsize
\textbf{(i):}  For any other representative $\mu$, we write the difference in expected divergence and use \myeq{jg} to simplify,
\small
\begin{align}
  &\sum \limits_{i=1}^N \beta_i  \, \rhobreg \left[ \rho(\tpi_i) : \rho(\mu)  \right] - \sum \limits_{i=1}^N \, \beta_i  \, \rhobreg \left[ \rho(\tpi_i) : \rho \of{{\mu}^*}
  \label{eq:displaced_jensen_gap}
  %\of{\mu_{\rho}(\vupi, \vwbeta)}  
  \right]\\
  &= \sum \limits_{i=1}^N \cancel{\beta_i \potentialf[\rho(\tpi_i)]}  - \potentialf[\rho(\mu)] - \int \Big( \sum \limits_{i=1}^N  \beta_i \rho\of{\tpi_i(\vz)} - \rho\of{\mu(\vz)} \Big) \tau(\mu(\vz)) d\vz  \\
  &\phantom{=\sum \limits_{i=1}^N \cancel{\beta_i \potentialf[\rho(\tpi_i)]}  - \int \potentialf[\rho(\mu)]} -\Big( \sum \limits_{i=1}^N\, \cancel{\beta_i \potentialf[\rho(\tpi_i)]} - \potentialf[\rho({\mu}^*)] \Big) \nonumber \\
  &= \potentialf[\rho({\mu}^*)] - \potentialf[\rho(\mu)] - \int \Big( \rho\of{{\mu}^*(\vz)} - \rho\of{\mu(\vz)} \Big) \tau(\mu(\vz)) d\vz \nonumber \\
  &= \rhobreg\big[
  %\rho \of{\mu_\rho(\vupi,\vwbeta)} 
  \rho\of{{\mu}^*}
  : \rho \of{\mu}\big] . \label{eq:breg_gap}
\end{align}
\normalsize
where we note that
%\rho_{\mu_\rho}
$\rho({\mu}^*) = \sum_{i=1}^N  \beta_i \rho(\tpi_i)$.  
The rho-tau divergence is minimized if and only if 
$\rho({\mu}^*) = \rho(\mu)$
%$\rho \of{\mu_\rho(\vupi,\vwbeta)} = \rho \of{\mu}$ 
\citep{zhang2004divergence}, thus proving $(i)$.

\textbf{(iii):} Finally, we can express the suboptimality gap in \myeq{displaced_jensen_gap} or rho-tau Bregman divergence in \myeq{breg_gap} as the gap in a conjugate optimization.   
Considering the conjugate expansion of $\potentialf[\rho({{\mu}^*})]$, we have
\small
\begin{align}
\hspace*{-.2cm} \potentialf[\rho({\mu}^*)] &= \sup \limits_{\tau(\mu)} \int \rho\of{{{\mu}^*}(x)} \tau(\mu(x)) d\vz - \potentialdual[\tau(\mu)] \\
&\geq  \int \rho{\of{{\mu}^*(\vz)}} \tau\of{\mu(\vz)} d\vz - \potentialdual[\tau({\mu})] \nonumber
\end{align}
\normalsize
for any choice of $\tau({\mu})$.   This provides a lower bound on $\potential[\rho({\mu}^*)]$, where the gap in the lower bound is the canonical form of the Bregman divergence.  Indeed, substituting 
$\potentialdual[\tau({\mu})] =  \int  \rho(\mu(\vz)) \tau(\mu(\vz)) d\vz - \potentialf[\rho({\mu})]$ in \myeq{displaced_jensen_gap}, we have
\small 
\begin{align}
   \rhobreg\big[ &\rho \of{\mu^*
   %_\rho(\vupi,\vwbeta)
   } : \rho \of{\mu}\big]  = \potentialf[\rho({\mu}^*)] + \potentialdual[\tau({\mu})] - \int\rho\of{{\mu}^*(\vz)} \tau\of{\mu(\vz)} d\vz \nonumber\\
    &\phantom{====} =\potentialf[\rho({\mu}^*)] - \potentialf[\rho(\mu)] - \int \Big( \rho\of{{\mu}^*(\vz)} - \rho\of{\mu(\vz)}\Big) \tau\of{\mu(\vz)} d\vz \, . \nonumber \qquad \hfill \qed
\end{align}
\normalsize

\end{proof}
\robsubsection{{Interpretations of \mythm{breg_info}(iii):}} Conjugate optimizations which treat $f$-divergences as a convex function of one argument 
are popular for providing variational lower bounds on divergences \citep{nguyen2010estimating, poole2019variational} or min-max optimizations for adversarial training \citep{nowozin2016f, nock2017fgan}.  Note however, that this proof provides a variational \textit{upper} bound on the Bregman Information, which includes the Jensen-Shannon divergence (\cref{example:jsd}) and mutual information (\citet{Banerjee2005} Ex. 6) as examples.  To our knowledge, this upper bound has not been used extensively in the literature. 

The equality in \myeq{bias_var} can also be interpreted as a generalized bias-variance tradeoff for Bregman divergences (\cite{pfau2013generalized, adlam2022understanding}).

\robsubsection{{Rho-Tau Bregman Information with Vector-Valued Inputs}}\label{app:thm1_vector}
A more standard setting is to consider a finite-dimensional Bregman divergence over a vector of inputs, 
such as the natural parameters $\btheta$ of a (deformed) exponential family $\tpi_{\btheta}^{(q)}(\vx) = \expbase \exp_q \{ \langle \boldtheta , \bt \rangle \}$. \textit{ However, we argue that this setting is best captured in our representational framework} (see \myapp{parametric_breg}), using $\rho(\tpi_{\btheta}^{(q)}(x)) = \log_q \frac{\tpi_{\btheta}^{(q)}(x)}{g(x)} = \langle \btheta, \bt \rangle = \sum_{j=1}^d \theta^j T^j(x)$ and the $\rho$-linearity of the density with respect to the appropriate base measure.

Nevertheless, we would also like to extend \cref{thm:breg_info} to hold for $N$ vector-valued, $d$-dimensional inputs.
%We give the following definitions, before stating and proving \cref{thm:breg_info_vector}.

\begin{thm}
 Consider a collection of inputs $\bm{\vectorinp} = \{ \bm{\vectorinp}_i \}_{i=1}^N$ where $\bm{\vectorinp}_{i} = \{\vectorinp_{i}^1,..., \vectorinp_{i}^j, ..., \vectorinp_{i}^d \}_{j=1}^d$.   In this case, 
consider applying the monotonic representation function $\rho : \cX_\rho^d \subset \mathbb{R}^d \rightarrow \cY_\rho^d \subset \mathbb{R}^d$ \textit{elementwise} $\rho(\bvector_i) \coloneqq \{\rho(\vectorinp_{i}^1), ..., \rho(\vectorinp_{i}^j), ... , \rho(\vectorinp_{i}^d) \}_{j=1}^d$.   For a convex generating function $F : \cY_\rho^d \subset \mathbb{R}^d \rightarrow \mathbb{R}$, 
define the Bregman divergence as
\begin{align}
%\hspace*{-.3cm} 
D_F[\rho(\bm{\vectorinp}_a) : \rho(\bm{\vectorinp}_b)] = F\of{\rho(\bm{\vectorinp}_a)} - F\of{\rho(\bm{\vectorinp}_b)} - \langle  \rho(\bm{\vectorinp}_a) - \rho(\bm{\vectorinp}_b), \nabla_{\rho} F\of{\rho(\bm{\vectorinp}_b)} \rangle \nonumber
\end{align}
where the inner product sums over dimensions $1 \leq j \leq d$.   Using analogous definition of a conjugate representation as in \cref{sec:rhotau}, we have $\tau(\bvector) = \nabla_{\rho} F(\rho(\bvector))$ with $\tau(\vectorinp^j)= \frac{\partial}{\partial (\rho(\bvector))^j} F(\rho(\bvector)) $.
%\end{definition}

Finally, consider discrete mixture weights $\vwbeta = \{\beta_i \}_{i=1}^N$ with $\sum_i \beta_i = 1$, and
assume the expected value 
%$\mu_{\rho}(\bm{\theta}, \vwbeta) \coloneqq \int \rho(\theta) d\vwbeta(\theta) \in \text{ri}(\cX_\rho)$
$\mu_{\rho}(\bvector, \vwbeta) \coloneqq \sum_{i=1}^N \beta_i ~ \rho(\bvector_i)  \in \mathrm{ri}(\cY_\rho^d)$
is in the relative interior of the range of $\rho$. % and domain of $f$.
Then, we have the following results,
\begin{itemize}%[(i)]
\item[(i)] For a %(decomposable) 
Bregman divergence with generator $F$, the optimization 
\begin{align}
    \BregInfo{{F,\rho}} \of{{\bm{\vectorinp}, \vwbeta}} &\coloneqq \min \limits_\mu \sum \limits_{i=1}^N \, \beta_i  \, D_F \left[ \rho(\bvector_i) : \rho(\bm{\mu})  \right] \, . \label{eq:breginfo_optimization_vec}
\end{align}
% \intertext
{has a unique minimizer given by the quasi-arithmetic mean with representation function $\rho(\tpi)$}
\begin{align}
       %\tpi^{(\rho)}_{\vwbeta}(\vz) &\coloneqq 
       \bm{\mu}_\rho^*(\bm{\vectorinp}, \vwbeta) = 
       \rho^{-1}\left( \sum \limits_{i=1}^N \beta_i \, \rho\of{\bvector_i} \right) 
       = \argmin  \limits_{\bm{\mu}} \sum \limits_{i=1}^N \, \beta_i  \, D_F \left[ \rho(\bvector_i) : \rho(\bm{\mu})  \right] \nonumber .
\end{align}
The arithmetic mean is recovered for $\rho(\bvector_i) = \bvector_i$ and any $F$ \citep{Banerjee2005}.
\item[(ii)] At this minimizing argument, the value of the expected divergence in \myeq{breginfo_optimization_vec} is called the  \textup{Rho-Tau Bregman Information} and is equal to a
gap in Jensen's inequality
%for the convex function $f$,
for the convex function $F$,
mixture weights $\vwbeta$, and inputs $\bvector = \{ \bvector_i\}_{i=1}^N$,
\begin{align}
    %\BregInfo{{F,\rho}} \of{{\bm{\vectorinp}, \vwbeta}} = \sum \limits_{i=1}^N \beta_i \, f\of{\rho({\vectorinp_i})} -   f\of{\sum \limits_{i=1}^N \beta_i \rho({\vectorinp_i})} 
    \BregInfo{{F,\rho}} \of{{\bvector, \vwbeta}} = \sum \limits_{i=1}^N \beta_i \, F\of{\rho(\bvector_i)}
    %- \potentialf \ofb{\rho_{\mu_{\rho}}} 
    -  F\of{\rho(\bm{\mu}_\rho^*)} 
    %(\vupi, \vwbeta)
    . 
    \label{eq:breginfo_jg_vector_v}
\end{align}
\item[(iii)] Using $\bm{\mu} \neq \bm{\mu}_\rho^*(\bm{\vectorinp}, \vwbeta)$ as the representative in \myeq{breginfo_optimization_vec}, 
%$\mu_\rho(\vz; \vupi,\vwbeta)$, 
the suboptimality gap 
%in \myeq{breginfo_optimization} 
is a rho-tau Bregman divergence
\begin{align}
D_F\big[
 \rho\of{\bm{\mu}_\rho^*}  :  \rho\of{\bm{\mu}}\big] &= \sum \limits_{i=1}^N \beta_i ~ D_F\left[ \rho\of{\bvector_i} : \rho\of{\bm{\mu}}  \right] - \BregInfo{{F,\rho}}\of{\bm{\vectorinp}, \vwbeta} \label{eq:bias_var_v}. %\\
 %&= F\of{\rho(\bm{\mu}_\rho^*)} + F^*\of{\tau(\bm{\mu})} - \big\langle \rho(\bm{\mu}_\rho^*), \tau(\bm{\mu}) \big\rangle \nonumber.
\end{align}
\end{itemize}
\label{thm:breg_info_vector}
\end{thm}
%\end{enumerate}
%\end{restatable}

\begin{proof}
\textbf{(ii):} Again, we start by showing that the optimal representative $\bm{\mu}_\rho^{*}
%\coloneqq \rho \of{\mu_{\rho}(\vx;\vupi, \vwbeta)} 
= \rho^{-1}\big($ $\sum_{i=1}^N \beta_i \, \rho(\bvector_i)\big)$ yields a Jensen diversity in $(ii)$.
Expanding the expected divergence in \cref{eq:breginfo_optimization_vec} for $\bm{\mu} = \bm{\mu}_\rho^{*}$, we have
$\sum_{i=1}^N  \beta_i D_F [ \rho(\bvector_i) : \rho \of{\bm{\mu}_\rho^*}  ] = \sum_{i=1}^N \beta_i  F\of{\rho(\bvector_i)} - F\of{\rho(\bm{\mu}_\rho^*)}$
$- \langle   \sum_{i=1}^N  \beta_i \rho(\bvector_i) - \rho(\bm{\mu}_\rho^*),  \tau(\bm{\mu}_\rho^*) \rangle $.   
Since $\sum_i \beta_i \rho(\bvector_i) = \rho(\bm{\mu}_\rho^*)$, the final term cancels to yield
\small 
\begin{align}
\sum \limits_{i=1}^N  \beta_i D_F\left[ \rho(\bvector_i) : \rho \of{\bm{\mu}_\rho^*} \right] 
&= \sum_{i=1}^N \beta_i  F\of{\rho(\bvector_i)} - F\of{\rho(\bm{\mu}_\rho^*)} . \label{eq:jgv}
%   \medmath{ 
% \min \limits_{\binforep} \sum \limits_{i=1}^N w_i   \Breg{\genericbreg} \left[ \bvector_i : \binforep  \right] &=  \min \limits_{\binforep} \sum \limits_{i=1}^N  w_i   \genericbreg\of{\bvector_i} - \genericbreg\of{\binforep}  - \big \langle \sum \limits_{i=1}^N w_i \bvector_i - \binforep , \nabla \genericbreg \of{\binforep} \big \rangle \nonumber
\end{align}
\normalsize
\textbf{(i, iii):} Writing the difference in expected divergence for a suboptimal representative $\bm{\mu}$ and using \myeq{jgv},
%\myeq{breg_expansion}-(\ref{eq:jg}), 
we have
%$\sum_{i} \beta_i \potentialf[\rho(\bvector_i)]$ terms cancel to yield
\small
\begin{align}
  &\sum \limits_{i=1}^N \beta_i  \, D_F \left[ \rho(\bvector_i) : \rho(\mu)  \right] - \sum \limits_{i=1}^N \, \beta_i  \, D_F \left[ \rho(\bvector_i) : \rho \of{\bm{\mu}_\rho^*}
\nonumber 
  %\label{eq:displaced_jensen_gap_v}
  %\of{\mu_{\rho}(\vupi, \vwbeta)}  
  \right]\\
  &= \sum \limits_{i=1}^N \cancel{\beta_i F\of{\rho(\bvector_i)} }  - F\of{\rho(\bm{\mu})} - \big\langle \sum \limits_{i=1}^N  \beta_i \rho\of{\bvector_i} - \rho\of{\bm{\mu}} , \tau(\bm{\mu})\big\rangle   - 
  \Big( \sum \limits_{i=1}^N\, 
   \cancel{F\of{\rho(\bvector_i)}} - F\of{\rho(\bm{\mu}_\rho^*)} \Big) \nonumber \\
  %\cancel{\beta_i \potentialf[\rho(\bvector_i)]} - \potentialf[\rho(\bm{\mu}_\rho^*)] \Big) \nonumber \\
  &= D_F\big[
  %\rho \of{\mu_\rho(\vupi,\vwbeta)} 
  \rho\of{\bm{\mu}_\rho^*}
  : \rho \of{\bm{\mu}}\big] . \label{eq:breg_gap_v}
\end{align}
\normalsize
The rho-tau divergence is minimized iff
$\rho(\bm{\mu}_\rho^*) = \rho(\bm{\mu})$
%$\rho \of{\mu_\rho(\vupi,\vwbeta)} = \rho \of{\mu}$ 
\citep{zhang2004divergence}, thus proving $(i)$.
\end{proof}

\vheader
\vheader
\section{Parametric Bregman Divergence and Annealing Paths within (Deformed) Exponential Families}\label{app:parametric_breg}
\vheader
%\section{Parametric Bregman Divergence from Normalization $Z_{q}(\beta)$}\label{app:parametric_breg}
Consider a $q$-exponential family with a $d$-dimensional natural parameter vector $\btheta \in \Theta \subset \bbR^d$, sufficient statistic vector $\bt$, and base density $\expbase$,
% \small
\begin{align}\label{eq:qexp_defn_app}
    \qexpfam(\vx) &= \frac{1}{\Z_q(\boldtheta)} \expbase \exp_q \{ \langle \boldtheta , \bt \rangle \} \qquad  \\
    \text{where} \quad \Z_q(\boldtheta) &= \int \expbase \exp_q \{ \langle \boldtheta , \bt \rangle \} d\vx . \nonumber
\end{align}
\normalsize
We let $\tqexpfam(\vx) = \expbase \exp_q \{ \langle \boldtheta , \bt \rangle \}$ denote the unnormalized density, often abbreviating to $\tpi_{\bm{\theta}}(\vx)$ for convenience.

From the convexity of $\exp_q$, it can be shown that the normalization constant $\Z_q(\boldtheta)$ is a convex function of the parameters $\boldtheta$, with first derivative
\small
\begin{align}
    \frac{\partial \Z_q(\boldtheta)}{\partial \theta^j} 
    %&= \int \frac{\partial }{\partial \theta_i} \expbase [1 + (1-q) \boldtheta \cdot \bt ]^{\frac{1}{1-q}} d\vx \\
    &=  \int \expbase [1 + (1-q) \boldtheta \cdot \bt ]^{\frac{q}{1-q}} \cdot T^j(\vx) d\vx = \int \expbase^{1-q} \tpi_{\bm{\theta}}(\vx)^{q} \cdot T^j(\vx) d\vx
\end{align}
\normalsize
\vspace*{5pt}
\robpara{Parametric Interpretation of Amari $\alpha$-Divergence}
We now show that the Bregman divergence induced by $\frac{1}{q} \Z_q(\boldtheta)$, for $q > 0$, corresponds to the Amari $\alpha$-divergence between parametric unnormalized densities,
\small
\begin{align}
    D_{\frac{1}{q} \Z_q}[\bthetaprime: \boldtheta] &= \frac{1}{q} \Z_q(\bthetaprime) - \frac{1}{q} \Z_q(\boldtheta) - \langle \nabla \frac{1}{q} \Z_q(\boldtheta), \bthetaprime - \boldtheta \rangle \\
    &=\frac{1}{q} \Z_q(\bthetaprime) - \frac{1}{q} \Z_q(\boldtheta) - \frac{1}{q}\int \expbase^{1-q} \tpi_{\bm{\theta}}(\vx)^{q} \Big( \langle  \bthetaprime, \bt \rangle  - \langle  \btheta, \bt \rangle \Big)  d\vx \nonumber \\
    &\overset{(1)}{=} \frac{1}{q} \Z_q(\bthetaprime) - \frac{1}{q} \Z_q(\boldtheta) - \frac{1}{q} \int \expbase^{1-q} \tpi_{\bm{\theta}}(\vx)^{q} \Big( \log_q \frac{\tilde{\pi}_{\bthetaprime}(\vx)}{\expbase} -  \log_q \frac{\tpi_{\bm{\theta}}(\vx)}{\expbase} \Big)  d\vx \nonumber \\
    &= \frac{1}{q} \Z_q(\bthetaprime) - \frac{1}{q} \Z_q(\boldtheta) - \frac{1}{q (1-q)} \int \expbase^{1-q} \tpi_{\bm{\theta}}(\vx)^{q} \Big( \frac{\tilde{\pi}_{\bthetaprime}(\vx)}{\expbase}^{1-q} -   \frac{\tpi_{\bm{\theta}}(\vx)}{\expbase}^{1-q} \Big)  d\vx \nonumber \\
    &= \frac{1}{q} \Z_q(\bthetaprime) - \frac{1}{q} \Z_q(\boldtheta) - \frac{1}{q (1-q)} \int \Big(  \tilde{\pi}_{\bthetaprime}(\vx)^{1-q} \tpi_{\bm{\theta}}(\vx)^{q} d\vx + \frac{1}{q(1-q)} \int \tpi_{\bm{\theta}}(\vx)   d\vx \nonumber \\
    &= \frac{1}{q}\int \tilde{\pi}_{\bthetaprime}(\vx) d\vx + \frac{1}{1-q} \int \tpi_{\bm{\theta}}(\vx) d\vx - \frac{1}{q(1-q)} \int   \tilde{\pi}_{\bthetaprime}(\vx)^{1-q} \tpi_{\bm{\theta}}(\vx)^{q} d\vx \nonumber \\
    &= D_{A}^{(q)}[\tilde{\pi}_{\bthetaprime} : \tpi_{\bm{\theta}} ] \nonumber
\end{align}
\normalsize
where in $(1)$ we use the fact that $\log_q \frac{\tpi_{\bm{\theta}}(\vx)}{\expbase} = \langle \btheta, \bt \rangle$.

For $q=1$ and the exponential family, we have $\Z(\boldtheta) = \int \tpi_{\boldtheta}(\vx) d\vx$ and $\frac{\partial \Z(\boldtheta)}{\partial \theta^j} =  \int \tpi_{\boldtheta}(\vx) T^j(\vx) d\vx$, 
    which leads to the Bregman divergence
    \begin{align}
    D_{Z}[\bthetaprime : \boldtheta]
    &= \DKL[\tpi_{\boldtheta}(\vx) : \tpi_{\bthetaprime}(\vx)]. \nonumber
\end{align}
%$Z(\boldtheta)$ with 
By contrast, the divergence generated by the \textit{log} partition function $\log \Z(\boldtheta)$ yields the \textsc{kl} divergence between normalized distributions.  Using $\frac{\partial}{\partial \theta^j} \log \Z(\boldtheta) = \frac{1}{\Z(\boldtheta)} \int \tpi_{\boldtheta}(\vx)T^j(\vx) d\vx = \int \expfam(\vx) T^j(\vx) d\vx$,
we recover the well-known result \citep{amari2016information}
    \begin{align}
    D_{\log \Z}[\bthetaprime : \boldtheta]
    &= \DKL[\pi_{\boldtheta}(\vx) : \pi_{\bthetaprime}(\vx)]. \nonumber
\end{align}

\robsubsection{Parametric Divergence in $\rho = \log_q$ Representation}\label{app:parametric_nonparametric}
\vheader
%\robpara{Parametric Divergence in $\rho = \log_q$ Representation} 
Alternatively, we may view the parametric divergence $D_{\frac{1}{q} \Z_q}[\bthetaprime: \boldtheta]$ as a decomposable divergence in the $\log_q$-representation $\rho(\tpi_{\btheta}^{(q)}) = \log_q \frac{\tpi_{\btheta}^{(q)}(x)}{g(x)}$, where
\begin{align}
\begin{split}
 \rho(\tpi_{\btheta}^{(q)}) &= \log_q \frac{\tpi_{\btheta}^{(q)}(\vz)}{\expbase}= \langle \btheta, \bt \rangle \qquad   f(\rho) = \frac{1}{q} \left( \exp_q\{\rho\}-\rho -1 \right) \qquad \\ \potentialf\ofb{\rho_{\tpi}} &=  - \frac{1}{q} \int  \langle \btheta, \bt \rangle  d\vz 
     %\int f\of{\rho(\tpi(\vz)} d\vz = 
     + \frac{1}{q} \Z_q(\btheta)  - \frac{1}{q}. 
    \end{split}
     \label{eq:parametric_rho}
\end{align}
Note, the order of the $\alpha$-divergence is set by the deformation parameter $q$ in the definition of the representation function $\rho_q$ or $q$-exponential family. %$\tpi_{\btheta}^{(q)}$.
%See \cref{example:annealing} below.
In \mysec{parametric} \cref{example:q_path_para}, we have used this 
Bregman divergence minimization to interpret the $q$-paths between \textit{arbitrary} endpoints from a parametric perspective.

This interpretation also suggests that the form of the deformed family in \cref{eq:qexp_defn_app}, particularly its $\log_q$-linearity in $\btheta$, is sufficient to derive the parametric divergence as a non-parametric rho-tau divergence using \cref{eq:parametric_rho}.

Due to this generality of the non-parametric perspective, we advocate viewing parametric annealing paths within the same $q$-exponential family through the lens of \cref{thm:breg_info} instead of the vector-valued perspective in \cref{thm:breg_info_vector}.   
Indeed, when annealing between parametric deformed exponential family endpoints in \cref{example:annealing}, the quasi-arithmetic mean in the $\rho= \log_q$ representation (using \cref{thm:breg_info}) suggests \textit{linear} or arithmetic mixing of the natural parameters.   This interpretation is more natural, and analogous to the exponential family case, compared to taking the quasi-arithmetic mean of the parameter vectors $\btheta$ directly (using \cref{thm:breg_info_vector}).

%\vheader 
\vheader
\robsubsection{Annealing Paths between (Deformed) Exponential Family Endpoint Densities}\label{app:grosse}
%\vheader

The above Bregman divergences can be used to analyze annealing paths in the special case where the endpoint densities
$\tpi_{\boldtheta_0}$ and $\tpi_{\boldtheta_1}$
belong to the same (deformed) exponential family in \cref{eq:qexp_defn_app}.

\begin{example}[Annealing within (Deformed) Exponential Families]\label{example:annealing}
Due to the $\rho$-affine property of deformed exponential families, it is natural to consider the $\rho(\tpi) = \log_q \tpi$ path within the $\exp_q$ family.   

Ignoring the normalization constant, the unnormalized density with respect to $\expbase$ is linear in $\btheta$ after applying the $\rho(\texpfam) = \log \frac{\texpfam(\vz)}{g(\vz)}$ representation function, with $\log \frac{\tpi_{\btheta}(\vz)}{\expbase} = \langle \btheta , \bt \rangle.$  
Since the quasi-arithmetic mean also has this $\rho$-affine property (\myeq{abstract_affine}), we can see that the $q$-path or geometric path between (deformed) exponential endpoints $\tpi_{\btheta_0}(\vz)$ and $\tpi_{\btheta_1}(\vz)$ 
is simply a linear interpolation in the natural parameters 
%of the endpoint exponential family densities,
%has the particularly simple form
\begin{align}
    \btheta_\beta &= (1-\beta) \, \btheta_0 + \beta  \, \btheta_1  = \argmin \limits_{\btheta_r} ~ (1-\beta) D_{\frac{1}{q}\mathcal{Z}_q}[\btheta_0 : \btheta_r] + \beta ~ D_{\frac{1}{q}\mathcal{Z}_q}[\btheta_1 : \btheta_r]. \nonumber
\end{align}
which includes the \textsc{kl} divergence and exponential family for $q=1$.
\end{example}

\begin{example}[Moment Averaging Path of \citet{grosse2013annealing}]\label{example:grosse}
In the case of the standard exponential family,
%\small
\begin{equation}
\begin{aligned}
    \expfam(\vz) &= \expbase \exp \{ \langle \btheta , \bt \rangle - \psi(\btheta) \} \,\,
    %\\
    \text{with} \,\,  \psi(\btheta) =  \log \int \expbase \exp \{ \langle \btheta , \bt \rangle \} d\vx , \nonumber
\end{aligned}%\label{eq:parametric_exp_family} 
\nonumber
\end{equation}
\normalsize
%i.e., linear interpolation on the natural parameters of the densities of the exponential family.
 \citet{grosse2013annealing} propose the \textit{moment averaging} path, which uses the dual parameter mapping $\rho(\btheta) = \boldeta(\btheta) = \mathbb{E}_{\expfam}\left[ \bt \right]$ as a representation function for the quasi-arithmetic mean,
%\small
\begin{align}
    \boldeta(\btheta_\beta) &= (1-\beta) \, \boldeta(\btheta_0) + \beta  \, \boldeta(\btheta_1) 
    % \label{eq:moment_avg} 
    % \\[1.25ex]
    % &
    =\argmin \limits_{\bm{\eta}_r} (1-\beta) D_{\psi^*}[\bm{\eta}_0 : \bm{\eta}_r] + \beta ~ D_{\psi^*}[\bm{\eta}_1 : \bm{\eta}_r]
    \nonumber
\end{align}
\normalsize
for an appropriate dual divergence based on the dual of the log partition function $\psi^*(\bm{\eta}(\bm{\theta})) = D_{KL}[\pi_{\bm{\theta}}(x): g(x)]$ (see \citep{grosse2013annealing}).   While \citet{grosse2013annealing} show 
performance gains
% improved performance for partition function estimation 
using the moment averaging path, additional sampling procedures may be required to find $\btheta_{\beta}$ via the inverse mapping $\boldeta^{-1}(\cdot)$.
\end{example}

\vheader
\vheader
\vspace*{-3pt}
\section{Information Geometry of Rho-Tau Divergences} \label{app:info_geo}
%\vspace*{-3pt}
\vheader 
We next review results from \citep{zhang2013nonparametric} describing the statistical manifolds induced by rho-tau Bregman Informations.  We summarize using our notation in \cref{tab:connections}.
%In this section, we review results from \citep{zhang2013nonparametric} describing the statistical manifold structures induced by the rho-tau Bregman Divergence.   We summarize using our notation in \cref{tab:connections}.
\vspace*{4pt}
%We summarize the Riemannian metrics and affine connections using our notation in \cref{tab:connections}.
%in preparation for our proof relating quasi-arithmetic of \cref{thm:geodesic}.

\robpara{Eguchi Relations}
The seminal Eguchi relations \cite{eguchi1983second, eguchi1985differential} describe the statistical manifold structure $(\mathcal{M}, g, \nabla, \nabla^*)$ induced by a divergence $D[\firstargn:\secondargn]$.   We first consider a manifold of parametric densities $\mathcal{M}$ represented by a coordinate system $\boldtheta(\pi): \mathcal{M}_{\btheta} \mapsto \Theta \subset \mathbb{R}^N$, with $\partial_i = \frac{\partial}{\partial \theta^i}$ as a basis for the tangent space.
The Riemannian metric is written $g_{ij}(\btheta)
= \langle \partial_i, \partial_j \rangle$, while the affine connection or covariant derivative is expressed using the scalar Christoffel symbols $\Gamma_{ij,k}(\btheta) = \langle \nabla_{\partial_i} \partial_j, \partial_k \rangle$ \citep{amari2000methods, nielsen2018elementary}.
For a given divergence, taking the second and third order differentials yield the following metric and conjugate pair of affine connections
\small
\begin{align}
    g_{ij}(\btheta) = - ( {\partial_j})_{\firstargparam} ( {\partial_k} )_\secondargparam  &D[\firstargparam:\secondargparam] \Big|_{\firstargparam=\secondargparam} \label{eq:metric_param}\\ 
    \Gamma_{ij,k}(\btheta) = - ( {\partial_i})_\firstargparam ( {\partial_j})_\firstargparam ( {\partial_k})_\secondargparam &D[\firstargparam:\secondargparam] \Big|_{\firstargparam=\secondargparam} \label{eq:gamma_param} \\ 
    \Gamma^*_{ij,k}(\btheta) = - ( {\partial_i})_\secondargparam ( {\partial_j})_\secondargparam ( {\partial_k})_\firstargparam &D[\firstargparam:\secondargparam] \Big|_{\firstargparam=\secondargparam} \label{eq:gammastar_param}
\end{align}
\normalsize
where $({\partial_j})_{\firstargparam}$ indicates partial differentiation with respect to the parameter $\theta_{a}^j$ with index $j$ of the first argument.

\robpara{Statistical Manifold from Rho-Tau Divergence} 
Following \citet{zhang2004divergence, zhang2013nonparametric},  
viewing $D_{f,\rho}^{(\beta)}[\tpi_0 : \tpi_1] \coloneqq \frac{1}{\beta(1-\beta)} \, 
   \BregInfo{{f,\rho}} \of{{\vupi, \vwbeta}}$ in \myeq{scaled_breg} as a divergence functional yields the following Riemannian metric and primal affine connection (expressed using the Christoffel symbols $\Gamma_{ij,k}(\boldtheta)$),
%\small
\begin{align}
  \hspace*{-.2cm}  g_{ij}(\boldtheta) &= \int \frac{\partial \rho_{\expfam}(\vx)}{\partial \theta^i} \frac{\partial \tau_{\expfam}(\vx)}{\partial \theta^j} d\vx = \int 
  %\frac{1}
  {\rho^{\prime}_{\expfam}(\vx) \tau^{\prime}_{\expfam}(\vx)}\frac{\partial \expfam(\vx)}{\partial \theta^i}\frac{\partial \expfam(\vx)}{\partial \theta^j} d\vz, \label{eq:rhotau_param_metric1} \\[1.5ex]
   \Gamma_{ij,k}(\btheta) &= {\int  {\rho^{\prime}_{\expfam} \tau^{\prime}_{\expfam}} \left( \frac{\partial^2 \expfam}{\partial \theta^i \partial \theta^j} \frac{\partial \expfam}{\partial \theta^k} - \alpha\big(\vx; \rho,\tau,\beta \big) \frac{\partial \expfam}{\partial \theta^i }\frac{\partial \expfam}{\partial \theta^j } \frac{\partial \expfam}{\partial \theta^k}  \right)d\vx} \label{eq:rt_affine_connection} \\
   & \qquad \text{where}\,\, \alpha\big(\vx; \rho,\tau,\beta \big) = -(1-\beta) \frac{\tau^{\prime\prime}_{\expfam}(\vx)}{\tau^{\prime}_{\expfam}(\vx)} - \beta \frac{\rho^{\prime\prime}_{\expfam}(\vx)}{\rho^{\prime}_{\expfam}(\vx)} . \nonumber
\end{align}
\normalsize

\begin{table}[t]
    \centering
    \vspace*{-2pt}
     \resizebox{\textwidth}{!}{
    \begin{tabular}{rccc}
      \small Divergence    &  \small $\tablelabel(\tpi)$ & \small $\tablelabel^*(\tpi)$ &  
      %$\rho^{\prime}(\tpi) \tau^{\prime}(\tpi)$
      \large $\substack{\text{Outer} \\ \text{Integration}}$
    %& \small Divergence    &  \small $\tablelabel$ & \small $\tablelabel^*$ 
    \\ \toprule
    $\rhobreg[\rho(\firstarg):\rho(\secondarg)]$  & $-\frac{\rho^{\prime\prime}(\tpi)}{\rho^{\prime}(\tpi)}$ &  $-\frac{\tau^{\prime\prime}(\tpi)}{\tau^{\prime}(\tpi)}$ & $\rho^{\prime}(\tpi) \tau^{\prime}(\tpi)$ \\[1.3ex] 
    $\frac{1}{\beta(1-\beta)}\BregInfo{f,\rho} \of{\vupi, \vwbeta}$ &  $-\big((1-\beta) \frac{\tau^{\prime\prime}(\tpi)}{\tau^{\prime}(\tpi)} +  \beta \frac{\rho^{\prime\prime}(\tpi)}{\rho^{\prime}(\tpi)} \big)$  & $-\big( (1-\beta) \frac{\rho^{\prime\prime}(\tpi)}{\rho^{\prime}(\tpi)} + \beta \frac{\tau^{\prime\prime}(\tpi)}{\tau^{\prime}(\tpi)} \big)$ & $\rho^{\prime}(\tpi) \tau^{\prime}(\tpi)$ \\
        %$\taubreg[\tau(\firstarg):\tau(\secondarg)]$ &   $\text{\small$(1-\beta)$} \frac{\rho^{\prime\prime}(\tpi)}{\rho^{\prime}(\tpi)} + \text{\small$\beta$} \frac{\tau^{\prime\prime}(\tpi)}{\tau^{\prime}(\tpi)}$ & $\text{\small$(1-\beta)$} \frac{\tau^{\prime\prime}(\tpi)}{\tau^{\prime}(\tpi)}  + \text{\small$\beta$} \frac{\rho^{\prime\prime}(\tpi)}{\rho^{\prime}(\tpi)}$  & $\rho^{\prime}(\tpi) \tau^{\prime}(\tpi)$ \\
     \midrule
     &  \multicolumn{2}{c}{\small (\text{all} $\alpha(\beta,\rho,\tau)$ \text{below omit a} $\tpi^{-1}$ \text{factor}) } & 
     \\ \midrule 
 $\DKL[\secondarg:\firstarg]$ ($q=1$) & $1$ & $0$ & $\tpi^{-1}$ \\
%   $\DKL[\firstarg:\secondarg]$ & 0 & 1 & $\tpi^{-1}$ \\
  $D_A^{(\alpha)}[\firstarg:\secondarg]$ ($\substack{ \, q = 1 \\ \, \beta = \alpha}$)
  & $\beta$ & $1-\beta$ & $\tpi^{-1}$ \\   \midrule
    %& $\alpha$ & $1-\alpha$ & $\tpi^{-1}$ \\[2ex]
     $D_A^{(\alpha)}[\firstarg:\secondarg]$ ($\substack{ \, q = \alpha \\ \, \beta = 1}$) & $q$ & $1-q$ & $\tpi^{-1}$ \\
     $D_Z^{(\beta, q)}[\secondarg:\firstarg]$ 
     & $(1-\beta)(1-q) + \beta q$ &  $(1-\beta) q + \beta (1-q) $ &
     %& $1 - \beta - q + 2 \beta q$ & $\beta + q - 2\beta q $ &
     %$ \beta + q - 
      $\tpi^{-1}$ \\    \midrule
     $\DKL[\firstarg:\secondarg]$ ($q=0$) & $0$ & $1$ & $\tpi^{-1}$ \\
    \text{\small $\frac{1}{\alpha(1-\alpha)}\DJSa [\firstarg : \secondarg ]$  ($\substack{ \, q = 0 \\ \, \beta = \alpha}$)}
    & $1-\beta$ & $\beta$ & $\tpi^{-1}$ \\ \midrule
     %& $\beta + q - 2 \beta q - 1$ & $ 2 \beta q - q - \beta$  & $\tpi$ \\
     $D_B^{(2-q)}[\secondarg:\firstarg]$  & $q$ & $0$ & $\tpi^{-q}$ \\
     \text{\small $\frac{1}{\beta(1-\beta)}$ \text{Breg. Info $D_{B}^{(2-q)}$}} & $\beta q$ & $(1-\beta) q$ & $\tpi^{-q}$
     \\ \midrule
          $D_C^{(q,\lambda)}[\firstarg:\secondarg]$  &  $ q $ & $1-\lambda$ & $\tpi^{\lambda - 1 - q}$ \\
     \text{\small $\frac{1}{\beta(1-\beta)}$ \text{Breg. Info $D_C^{(q,\lambda)}$}} & $(1-\beta)(1-\lambda) + \beta q $ & $(1-\beta)q + \beta (1-\lambda) $ & $\tpi^{\lambda - 1 - q}$ 
     \\ \bottomrule
     % [1.5ex]
     %  \midrule 
     %   $\fdiv[\firstarg:\secondarg]$ & $-f^{\prime\prime\prime}(1)-1$ & $f^{\prime\prime\prime}(1)+2$ & $\tpi^{-1}$  \\ 
    %\bottomrule
    \end{tabular}
      }
       \vspace*{-.1cm}
    \caption{ Dual pair of affine connections induced by divergence functions considered in \mysec{examples}, where 
    $\alpha(\beta,\rho,\tau)$ refers to \myeq{rt_affine_connection} for primal and dual  connections $\Gamma(\tpi)$ and $\Gamma^*(\tpi)$.  The factor $\rho^{\prime}(\tpi)\tau^{\prime}(\tpi)$ also specifies the Riemannian metric in \myeq{rhotau_param_metric1}.  
    %See \myapp{info_geo} for detailed background.  
    In each pair of rows, we list the rho-tau Bregman divergence $D_f[\rho({\firstarg}):\rho({\secondarg})]$ and induced Bregman Information, where values of $q$ indicate the order of $\rho(\tpi) = \log_q(\tpi)$.
    %values of $\alpha$ refer to expressions in \myeq{e}-(\ref{eq:alpha_connection}).  
    For the Jensen-Shannon divergence, note that we multiply by a factor $\frac{1}{\beta(1-\beta)}$ to ensure the standard form for $f$-divergences, with $f^{\prime\prime}(1)=1$.
    %\footnote{The \gls{JSD}, rescaled by $\frac{1}{\beta(1-\beta)}$ to ensure $f^{\prime\prime}(1)=1$, is also an $f$-divergence for $ f_{\beta}(u) = \frac{1}{\beta(1-\beta)} \left((1-\beta) u \log u + \big((1-\beta) u + \beta\big)\log \big((1-\beta) u + \beta \big)\right)$.}
    %For $f$-diverenges such as the \textsc{js} divergence, we multiply 
    %We ensure $f$-divergences $\fdiv$, including $\frac{1}{\beta(1-\beta)}D_{\text{JS}}^{(\beta)}$, are in standard form with $f^{\prime\prime}(1)=1$.  
    %In our notation, we confirm that the $\alpha$-connections induced by an $f$-divergence are the standard$-f^{\prime\prime\prime}(1)-1$ & $f^{\prime\prime\prime}(1)+2$
    }
    \label{tab:connections}
\vheader
\end{table}

Since our exposition in \cref{sec:rhotau} and \cref{sec:examples} considers a nonparametric manifold of arbitrary unnormalized densities, we also recall the nonparametric analogues of \cref{eq:rhotau_param_metric1} and \cref{eq:rt_affine_connection} from \citet{zhang2013nonparametric}.\footnote{ While rigorous constructions of such statistical manifolds are considered in \citep{pistone1995infinite, grasselli2010dual, loaiza2013q, loaiza2013riemannian, ay2017information}, we assume the manifold and tangent spaces are well-defined (as is done in \citep{zhang2013nonparametric}).}
For tangent vectors $u(\vx), v(\vx), w(\vx)$ (such that $\int u(\vx) d\vx = 0$) at a point $\tpi(\vx)$, 
\begin{align}
g_{u,v}({\tpi})&= \langle u, v \rangle  
%= \int \frac{1}{\tpi(\vz)} u(\vz) v(\vz) d\vz
=\int 
%{\rho^{\prime}\of{\tpi(\vx)} \tau^{\prime}\of{\tpi(\vx)}}
\rho^{\prime}_{\tpi}{(\vx)} \tau^{\prime}_{\tpi}{(\vx)} \, 
u(\vz) v(\vz) d\vz \label{eq:rhotau_metric} \\
% \end{align}
% \begin{align}
 \GammaRT{\beta}_{wu,v}(\tpi) &= \langle \nabla^{(\beta)}_w u, v \rangle    \label{eq:rhotau_connection}  \\
   &=   \int \rho^{\prime}_{\tpi}{(\vx)} \tau^{\prime}_{\tpi}{(\vx)} \Big(
 \big( d_w u(\vx) \big) v(\vx) 
 -\alpha\big(\vx; \rho,\tau,\beta \big)   
 u(\vx) w(\vx) v(\vx) \Big) d\vx . \nonumber \\
    & \qquad \text{where}\,\, \alpha\big(\vx; \rho,\tau,\beta \big) = -(1-\beta) \frac{\tau^{\prime\prime}_{\tpi}(\vx)}{\tau^{\prime}_{\tpi}(\vx)} - \beta \frac{\rho^{\prime\prime}_{\tpi}(\vx)}{\rho^{\prime}_{\tpi}(\vx)} . \nonumber
\nonumber
\end{align}
where $d_w u$ is the directional derivative of $u$ in the direction of $w$ and $\rho^{\prime}_{\tpi}{(\vx)} = \rho^{\prime}\of{\tpi(\vx)}$.
The parametric expression above can be recovered using, for example, $u(\vx) = \frac{\partial}{\partial \theta^i}\expfam(\vx)$.

\vspace*{5pt}
\robpara{Riemannian Metrics}
To recover the Fisher-Rao metric, we may consider the $\rho(\tpi) = \log \tpi$ and $\tau(\tpi) = \pi$ representations, which yields $\rho^\prime(\tpi) \tau^\prime(\tpi) = \tpi^{-1}$ as desired.
%and standard $\alpha$-connections, which are obtained for $\rho(\pi) = \log(\pi)$, $\tau(\pi) = \pi$, and $\alpha(x;\rho,\tau,\beta) = \alpha \cdot \pi_{\boldtheta}(\vx)^{-1}$.   
However, the Fisher-Rao metric may also be recovered using the representations $\rho(\tpi) = \log_q(\tpi)$, $\tau(\tpi) = \log_{1-q} \tpi$ used to derive the $\alpha$-divergence in \cref{example:alpha} (see e.g. \citet{nielsen2018elementary} Sec. 3.12).
Finally, the Jensen-Shannon divergence and Zhang's $(\beta, q)$ divergence also induce the Fisher-Rao metric due to the fact that the outer integration term $\rho^\prime(\tpi)\tau^\prime(\tpi) = \tpi^{-1}$, while the metric for the Beta divergence integrates $\rho^\prime(\tpi)\tau^\prime(\tpi) = \tpi^{-q}$ and thus may be referred to as a `deformed' metric \citep{naudts2018rho, zhang2021entropy}.

\robpara{Affine Connections}
Recall that the standard $\alpha$-connection \citep{amari1982differential, amari2000methods} is given by $\alpha(x;\rho,\tau,\beta) = \alpha \cdot \pi_{\boldtheta}(\vx)^{-1}$.   From \cref{tab:connections}, we see that the $\alpha$-connection with $\alpha = \beta$ is induced from \textit{either} the Amari $\alpha$-divergence or the Jensen-Shannon divergence, which are the rho-tau Bregman Information corresponding to the \kl divergences in either direction.   When treating the $\alpha$-divergence as a rho-tau Bregman divergence (as in \cref{example:alpha}) instead of a Bregman Information (as in \cref{example:geo}), we see that the order of the $\alpha$-connection is set by the representation parameter $\alpha = q$.   This mirrors the observations in \cref{sec:cichocki}.
Finally, note that the Beta divergence (as a rho-tau Bregman divergence) induces the alpha connections of order $q$ and $0$, instead of $1$ and $0$ for the \kl divergence or $q$ and $1-q$ for the Amari $\alpha$-divergence.

The divergence functionals derived from the rho-tau Bregman Information for either the Beta, $\alpha$, or Amari-Cichocki $(q,\lambda)$ divergences induce $\alpha$-connections which interpolate between the endpoint values based on the mixture parameter $\beta$.

\robpara{Limiting Behavior of Amari-$\alpha$ and Beta-Divergences}
Finally, in \cref{tab:limiting_kl}, we recall the limiting behavior of the $\alpha$ divergences as $q\rightarrow 0$ or $q \rightarrow 1$, and the Beta divergence as $q \rightarrow 1$ or $q \rightarrow 2$.   While these families of divergences agree in their limiting behavior as $q \rightarrow 1$, the Beta divergence recovers either the Euclidean Bregman divergence or Itakura-Saito divergence for $q = 0$ and $q \rightarrow 2$ respectively.

\newcommand{\tablesep}{\\[1.5ex]}%[2ex]
\small
\begin{table}[t]
%\vspace*{-.2cm}
    \centering
    \resizebox{\textwidth}{!}{
    \begin{tabular}{rrrrr}
    \toprule
    % $\rho,\tau$
        %  & \multirow{2}{*}{ $\DKL[\firstarg:\secondarg]$} & \multirow{2}{*}{$\DKL[\secondarg:\firstarg]$} & \multirow{2}{*}{Representation $\rho,\tau$} & \multirow{2}{*}{Convex Function $f(\rho)$, $f^*(\tau)$} \tablesep
   {{$\tau$-Deformed}} & Convex $f, f^*$ &  $q\rightarrow 0 $ & $q\rightarrow 1$ & $q = 2$   \\ \midrule 
 $\rho(\tpi) = \log_q \tpi$ & $f(\rho) = c \exp_q{\rho}$ &   $\DKL[\firstarg:\secondarg]$ & $\DKL[\secondarg:\firstarg]$ & $D_{\chi^2}[\secondarg:\firstarg]$   \tablesep
 $\tau(\tpi) = \log_{1-q}\tpi$ & $f^*(\tau) = c \exp_{1-q}{\rho}$ &   $\DKL[\secondarg:\firstarg]$ & $\DKL[\firstarg:\secondarg]$ &  $D_{\chi^2}[\firstarg:\secondarg]$   \tablesep \midrule
%  \midrule
    {{$\tau$-id}}  &  Convex $f, f^*$   & $q = 0$ & $q\rightarrow 1$  & $q \rightarrow 2$     \\ \midrule 
 $ \rho(\tpi) = \log_q \tpi$ & $f(\rho) =c (\exp_{q}{\rho})^{2-q}$ &  $\frac{1}{2}\|\firstarg - \secondarg\|_2^2$ & $\DKL[\secondarg:\firstarg]$ & $\DIS[\secondarg:\firstarg]$ \tablesep
  $\tau(\tpi) = \tpi$   &  $f^*(\tau) = c \log_{q-1}{\tau}$ &  $\frac{1}{2}\|\firstarg - \secondarg\|_2^2$ & $\DKL[\firstarg:\secondarg]$ &  $\DIS[\secondarg:\firstarg] $ \tablesep 
 \bottomrule
    \end{tabular}
     }
    \caption{ Limiting Behavior in $q$ for Amari $\alpha$ ($\tau$-deformed) and Beta-divergences ($\tau$-id) as rho-tau Bregman Divergences.  The $\rho, f$ rows indicate the behavior of $\rhobreg[\rho({\firstarg}):\rho({\secondarg})]$, and the $\tau, f^*$ rows indicate the behavior of $\taubreg[\tau({\firstarg}):\tau({\secondarg})]$, with $\rhobreg[\rho({\firstarg}):\rho({\secondarg})] = \taubreg[\tau({\secondarg}):\tau({\firstarg})]$.  Note that $D_{\chi^2}$ indicates Pearson's $\chi^2$ divergence and $\DIS$ indicates the Itakura-Saito divergence.}
    \label{tab:limiting_kl}
\end{table}
\normalsize

\section{Geodesics for the Rho-Tau Bregman Divergence}\label{app:qgeo}
In this section, we show that the quasi-arithmetic mixture path in the $\rho(\tpi)$ representation of densities is a geodesic with respect to the primal connection induced by the rho-tau divergence $\rhobreg[\rho(\tpi_1):\rho(\tpi_0)]$.
%\textcolor{red}{Recall from \citet{zhang2013nonparametric}} that the 
Recall from \citet{zhang2013nonparametric} Sec. 2.3 (Eq. 88) that the $\alpha$-connection (or covariant derivative) associated with the rho-tau Bregman divergence has the form
\begin{align}
    \nabla^{(\alpha)}_{\dot{\gamma}} \dot{\gamma} &= (d_{\dot{\gamma}} \dot{\gamma})_{\gamma_t} + \frac{d}{d\gamma} \big( \alpha \log \rho^{\prime}(\gamma_t) + (1-\alpha) \log \tau^{\prime}(\gamma_t) \big) \cdot \dot{\gamma}^2 \label{eq:affine_alpha} \\
    &= (d_{\dot{\gamma}} \dot{\gamma})_{\gamma_t} +  \left( \alpha \frac{ \rho^{\prime\prime}(\gamma_t)}{\rho^{\prime}(\gamma_t)} + (1-\alpha) \frac{ \tau^{\prime\prime}(\gamma_t)}{\tau^{\prime}(\gamma_t)} \right) \cdot \dot{\gamma}^2 \nonumber
\end{align}
where $\dot{\gamma} = \frac{d\gamma}{dt}$ and the parameter $\alpha$ plays the role of the convex combination or mixture parameter $\beta$.    In this section, we use the notation $\nabla^{(\alpha)}$ to represent the affine connection, instead of the Christoffel symbol notation from e.g. \myeq{rhotau_connection}, with $\Gamma^{(\alpha)}_{wu,v}(\tpi) = \langle \nabla^{(\alpha)}_w u, v \rangle$.

We are interested in the Bregman divergence and $\nabla^{(1)}$ connection for $\alpha =1$.   Using \myeq{abstract_affine}, we need to show that the geodesic equation $ \nabla^{(1)}_{\dot{\gamma}} \dot{\gamma} = 0$ holds (\cite{nielsen2018elementary} Sec. 3.12) for curves which are linear in the $\rho$-representation.

\rhotau*
\begin{proof}
We simplify each of the terms in the geodesic equation, where we rewrite the desired geodesic equation in \myeq{geodesic_eq} to match Eq. 88 of \citet{zhang2013nonparametric},
\begin{align}
    \nabla^{(1)}_{\dot{\gamma}} \dot{\gamma}  = 
    d_{\dot{\gamma}} \dot{\gamma}
    %\frac{d \dot{\gamma}}{d\gamma} \cdot \dot{\gamma} 
    + \big(\dot{\gamma}\big)^{2} \cdot  \left(\frac{d}{d\gamma} \log {\rho^{\prime}(\gamma)}
    %\log \rho^{\prime}(\gamma)
    \right) = 0
\end{align}
%in \myeq{geodesic_eq}.
%We first derive an expression for $\dot{\gamma}(t)$, using the chain rule and the simplified form for 
First, note the particularly simple expression for $\frac{d\rho(\gamma_t)}{dt} = \rho(\tpi_1) - \rho(\tpi_0)$ given the definition $\rho(\gamma(t)) = (1-t) \rho(\tpi_0) + t \, \rho(\tpi_1)$.  Noting the chain rule $\frac{d\rho_t}{dt} = \frac{d\rho(\gamma_t)}{d\gamma} \frac{d\gamma_t}{dt}$, we can rearrange to obtain an expression for $\dot{\gamma}(t)$
% $\dot{\gamma}(t) = \frac{d\gamma(t)}{dt} = \frac{d\gamma(t)}{}$, this implies that 
\begin{align}
    \dot{\gamma}(t) = \frac{d\gamma(t)}{dt} = \left(\frac{d\rho(\gamma_t)}{d\gamma}\right)^{-1} \frac{d\rho(\gamma_t)}{dt} =  \left(\frac{d\rho(\gamma_t)}{d\gamma}\right)^{-1} \left( \rho(\tpi_1) - \rho(\tpi_0) \right)  \label{eq:dot_gamma}
    %\frac{1}{\rho(\tpi_1) - \rho(\tpi_0)} \frac{d\rho(\gamma(t))}{dt}
\end{align}
Taking the directional derivative $d_{\dot{\gamma}}\dot{\gamma} = \frac{d\dot{\gamma}}{d\gamma}\cdot \dot{\gamma}$,
\begin{align}
    d_{\dot{\gamma_t}}\dot{\gamma_t}  &= \dot{\gamma_t} \cdot \frac{d}{d\gamma} \left[ \left(\frac{d\rho(\gamma_t)}{d\gamma}\right)^{-1} \left( \rho(\tpi_1) - \rho(\tpi_0) \right) \right] \\
    &=  - \dot{\gamma_t}\left( \rho(\tpi_1) - \rho(\tpi_0) \right) \frac{d\rho(\gamma_t)}{d\gamma}^{-2} \frac{d^2\rho(\gamma_t)}{d\gamma^2} \nonumber
\end{align}
Rewriting the final term in \myeq{geodesic_eq}, we have
\begin{align}
    \frac{d}{d\gamma} \big(  \log \rho^{\prime}(\gamma_t) \big) &= \frac{d}{d\gamma} \big(  \log \frac{d \rho(\gamma_t)}{d\gamma}  \big) =  \left(\frac{d\rho(\gamma_t)}{d\gamma}\right)^{-1} \frac{d^2  \rho(\gamma_t) }{d\gamma^2}.  
\end{align}
Putting it all together, we have
\begin{align}
    \nabla^{(1)}_{\dot{\gamma_t}} \dot{\gamma_t}  &= -\frac{d \dot{\gamma_t}}{d\gamma} \cdot \dot{\gamma_t} + \big(\dot{\gamma_t} \big)^{2} \cdot \frac{d}{d\gamma} \big(  \log \rho^{\prime}(\gamma_t) \big) \\
    &= -\dot{\gamma_t} \left( \rho(\tpi_1) - \rho(\tpi_0) \right) \frac{d\rho(\gamma_t)}{d\gamma}^{-2} \frac{d^2\rho(\gamma_t)}{d\gamma^2} + (\dot{\gamma_t})^2 \left(\frac{d\rho(\gamma_t)}{d\gamma}\right)^{-1} \frac{d^2  \rho(\gamma_t) }{d\gamma^2} \nonumber \\
\intertext{Noting that $\frac{d\rho(\gamma_t)}{d\gamma}^{-1} = \dot{\gamma}(\rho(\tpi_1) - \rho(\tpi_0))$ from \myeq{dot_gamma}, we have }
\nabla^{(1)}_{\dot{\gamma_t}} \dot{\gamma_t} &= -\dot{\gamma_t} \cancel{\frac{\rho(\tpi_1) - \rho(\tpi_0)}{\rho(\tpi_1) - \rho(\tpi_0)}} \cdot \dot{\gamma_t}  \frac{d\rho(\gamma_t)}{d\gamma}^{-1} \frac{d^2\rho(\gamma_t)}{d\gamma^2} + (\dot{\gamma_t})^2 \left(\frac{d\rho(\gamma_t)}{d\gamma}\right)^{-1} \frac{d^2  \rho(\gamma_t) }{d\gamma^2} \nonumber \\
&= -(\dot{\gamma_t})^2  \left(\frac{d\rho(\gamma_t)}{d\gamma}\right)^{-1} \frac{d^2\rho(\gamma_t)}{d\gamma^2} + (\dot{\gamma_t})^2 \left(\frac{d\rho(\gamma_t)}{d\gamma}\right)^{-1} \frac{d^2  \rho(\gamma_t) }{d\gamma^2} \nonumber \\
&=0 \nonumber
\end{align}
which proves the proposition.\qed
\end{proof}

\end{document}